\documentclass[twoside]{article}

\usepackage[accepted]{aistats2025}
\usepackage[round]{natbib}

\usepackage{amsthm}
\newtheorem{theorem}{Theorem}
\newtheorem{proposition}{Proposition}
\newtheorem{corollary}{Corollary}
\newtheorem{lemma}{Lemma}
\newtheorem{definition}{Definition}
\newtheorem{remark}{Remark}
\usepackage{subcaption}
\usepackage{amssymb}
\usepackage{xcolor}     
\usepackage{algorithm}
\usepackage{algpseudocode}
\usepackage{graphicx}
\usepackage{shortcuts}
\usepackage{hyperref}
\usepackage[capitalize,noabbrev]{cleveref}
\usepackage{bbding}
\usepackage{multirow}
\usepackage{booktabs}
\usepackage{pifont}
\usepackage{array}

\begin{document}

\runningtitle{Every Call is Precious: Global Optimization of Black-Box Functions with Unknown Lipschitz Constants}

%\runningauthor{Surname 1, Surname 2, Surname 3, , Surname n}

\twocolumn[

\aistatstitle{Every Call is Precious: Global Optimization of Black-Box Functions with Unknown Lipschitz Constants}

\aistatsauthor{Fares Fourati \And Salma Kharrat \And  Vaneet Aggarwal \And Mohamed-Slim Alouini}

\aistatsaddress{KAUST  \And KAUST  \And  Purdue University  \And KAUST } 

]

\begin{abstract}
Optimizing expensive, non-convex, black-box Lipschitz continuous functions presents significant challenges, particularly when the Lipschitz constant of the underlying function is unknown. Such problems often demand numerous function evaluations to approximate the global optimum, which can be prohibitive in terms of time, energy, or resources. In this work, we introduce \textit{Every Call is Precious (ECP)}, a novel global optimization algorithm that minimizes unpromising evaluations by strategically focusing on potentially optimal regions. Unlike previous approaches, ECP eliminates the need to estimate the Lipschitz constant, thereby avoiding additional function evaluations. ECP guarantees no-regret performance for infinite evaluation budgets and achieves minimax-optimal regret bounds within finite budgets. Extensive ablation studies validate the algorithm's robustness, while empirical evaluations show that ECP outperforms 10 benchmark algorithms—including Lipschitz, Bayesian, bandits, and evolutionary methods—across 30 multi-dimensional non-convex synthetic and real-world optimization problems, which positions ECP as a competitive approach for global optimization.
\end{abstract}

\begin{figure*}[t]
    \centering
    \begin{subfigure}[b]{0.13\textwidth}
        \includegraphics[width=\textwidth]{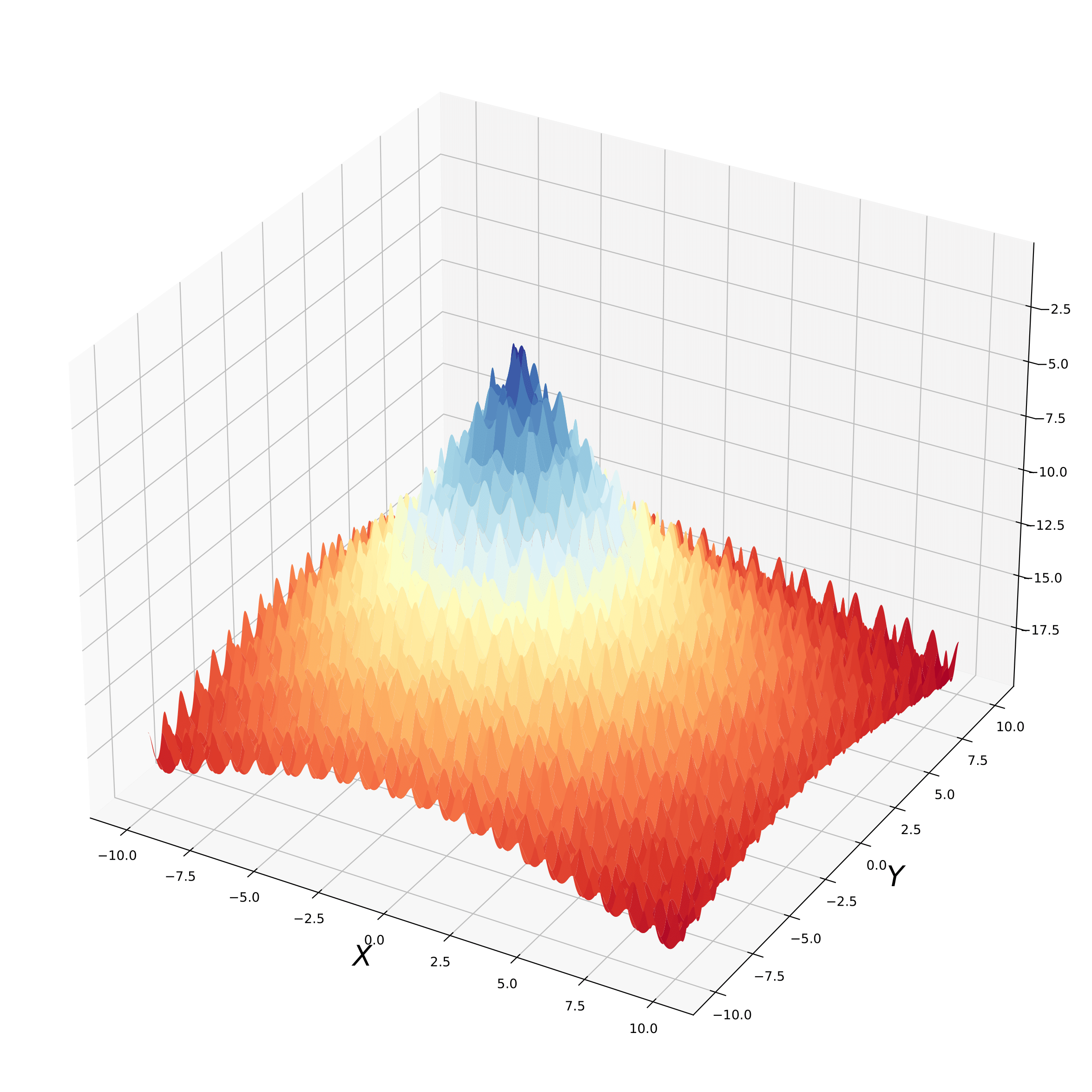}
        \caption{\small  Ackley}
        \label{fig:ackley_plot}
    \end{subfigure}
    \begin{subfigure}[b]{0.13\textwidth}
        \includegraphics[width=\textwidth]{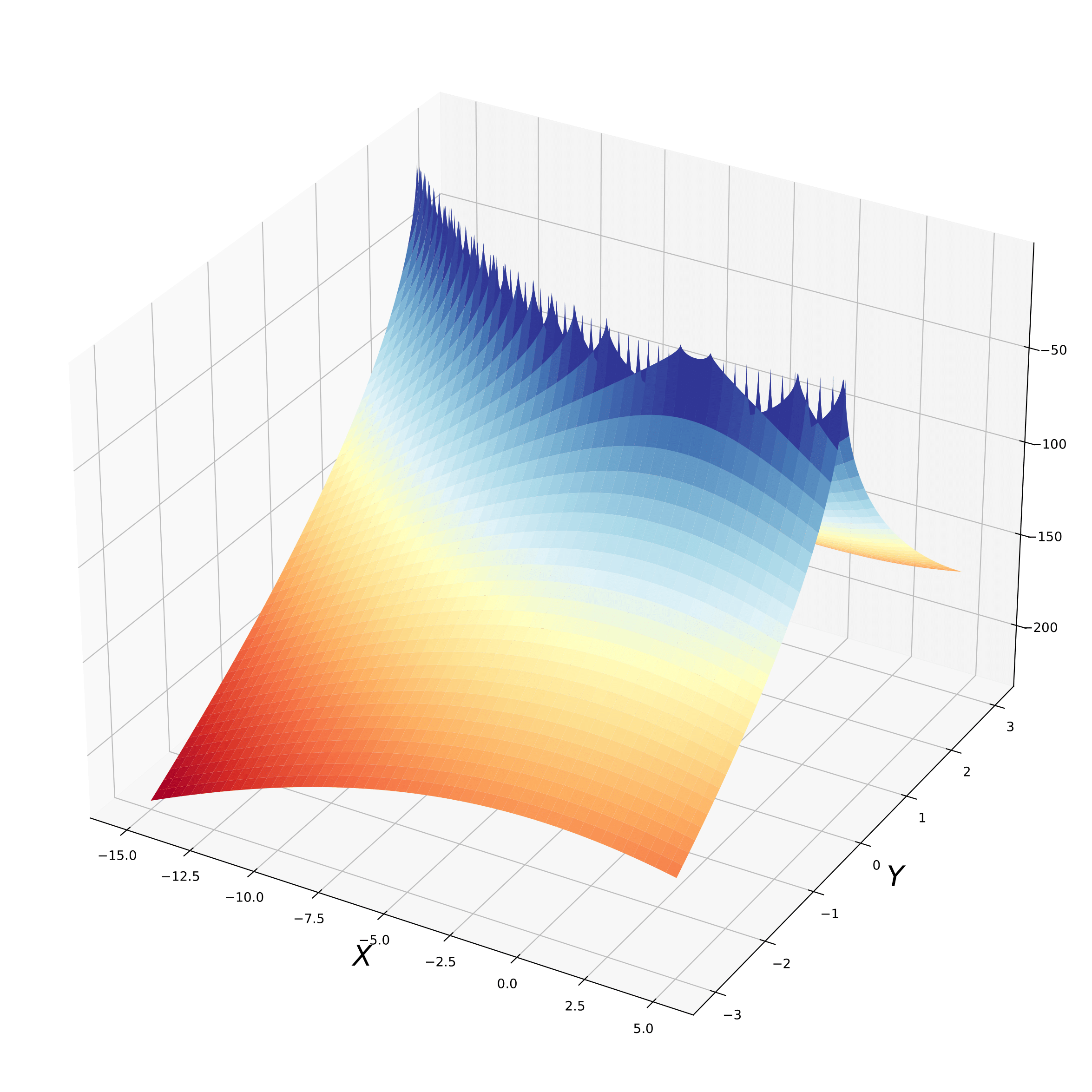}
        \caption{\small  Bukin}
        \label{fig:bukin}
    \end{subfigure}
    \begin{subfigure}[b]{0.13\textwidth}
        \includegraphics[width=\textwidth]{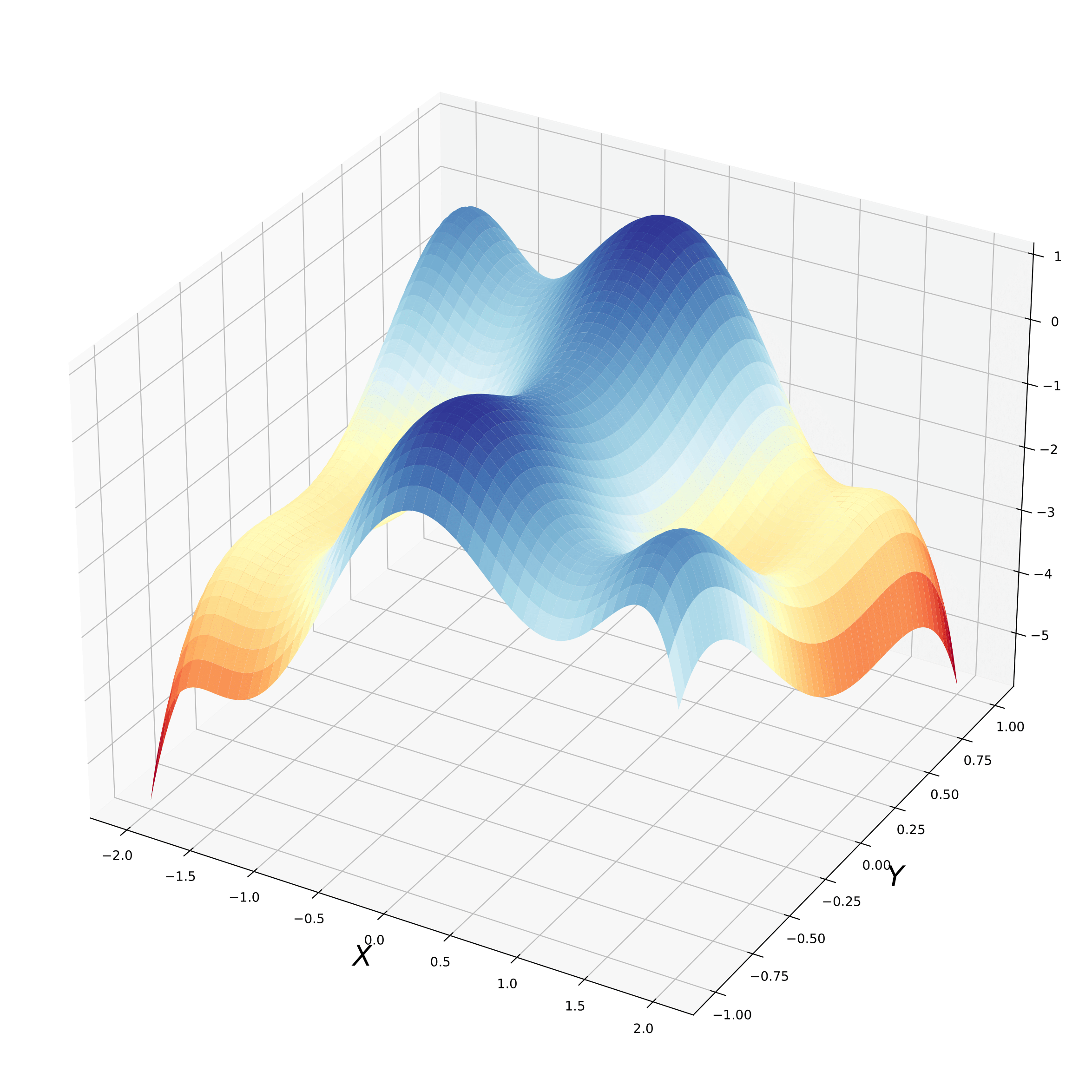}
        \caption{\small  Camel}
        \label{fig:camel}
    \end{subfigure}
    \begin{subfigure}[b]{0.13\textwidth}
        \includegraphics[width=\textwidth]{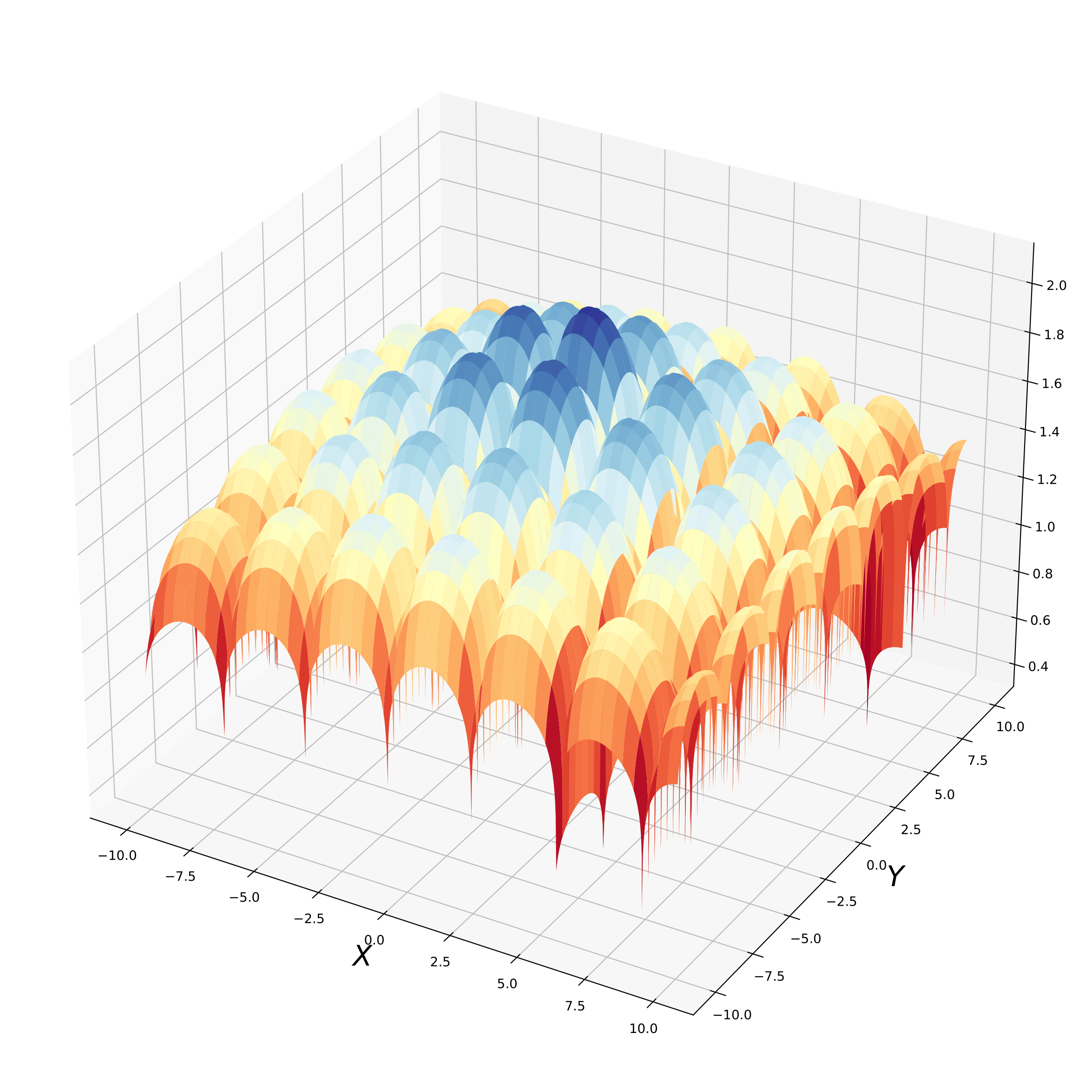}
        \caption{\small  Crossintray}
        \label{fig:crossintray}
    \end{subfigure}
    \begin{subfigure}[b]{0.13\textwidth}
        \includegraphics[width=\textwidth]{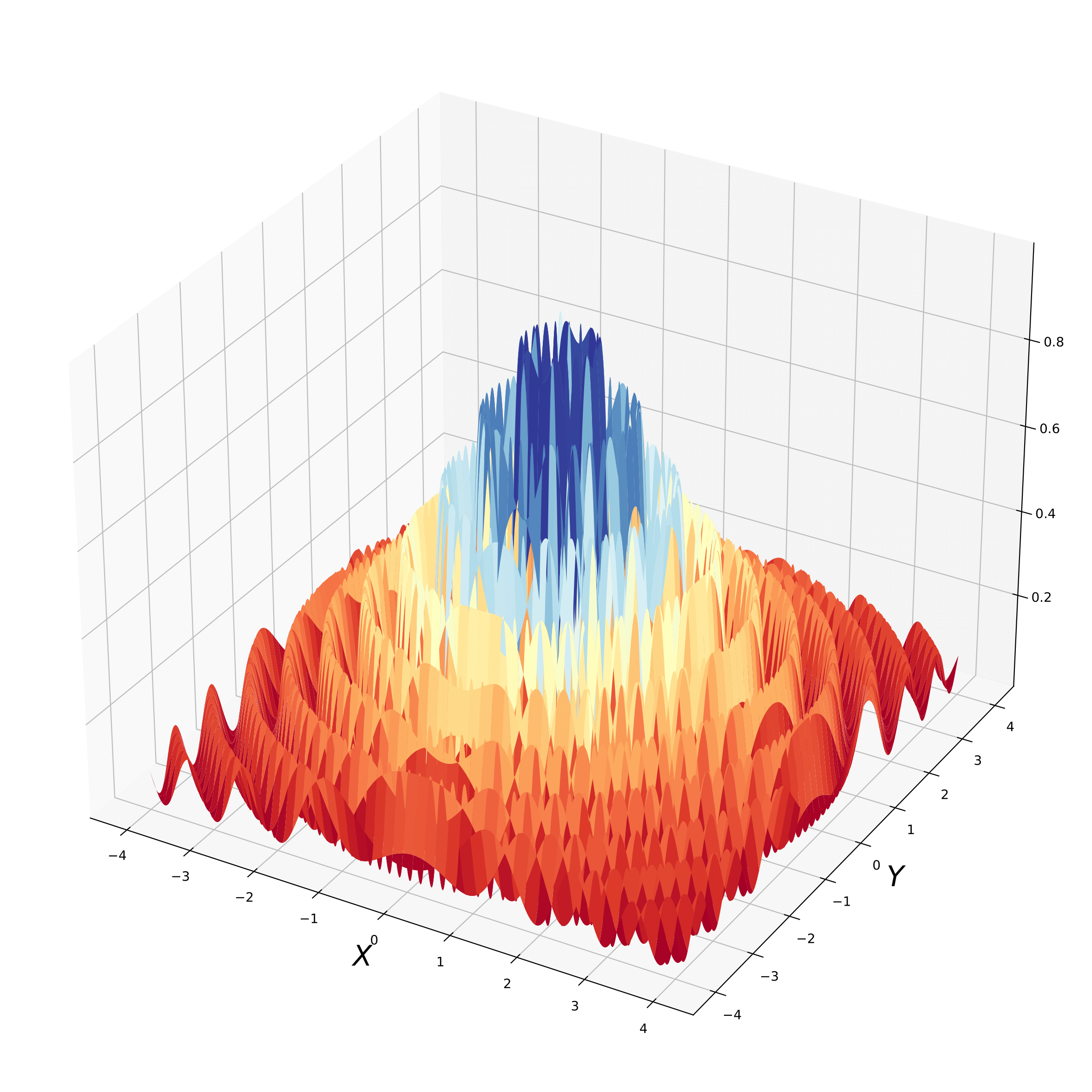}
        \caption{\small  Dropwave}
        \label{fig:dropwave}
    \end{subfigure}
    \begin{subfigure}[b]{0.13\textwidth}
        \includegraphics[width=\textwidth]{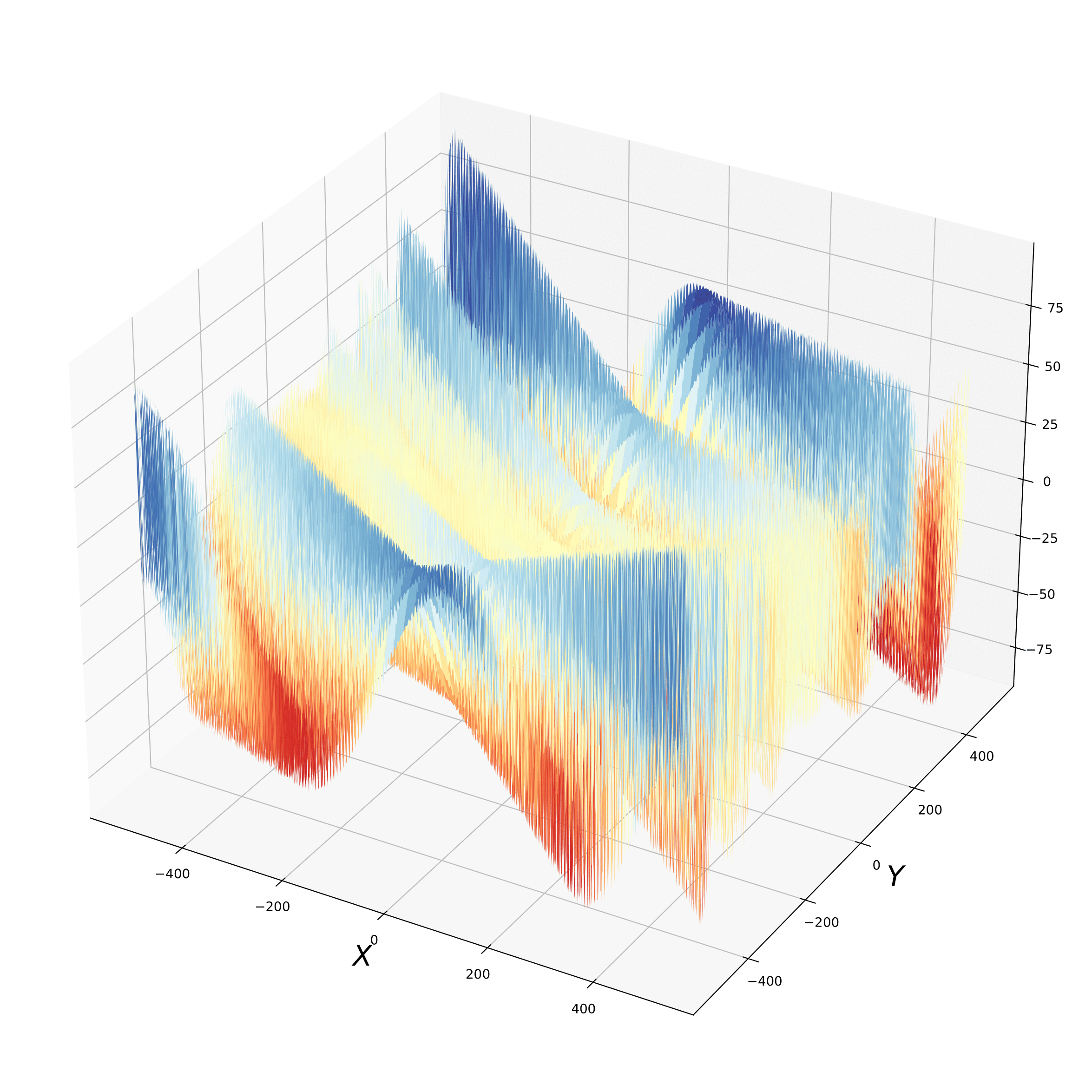}
        \caption{\small  Eggholder}
        \label{fig:eggholder}
    \end{subfigure}
    \begin{subfigure}[b]{0.13\textwidth}
        \includegraphics[width=\textwidth]{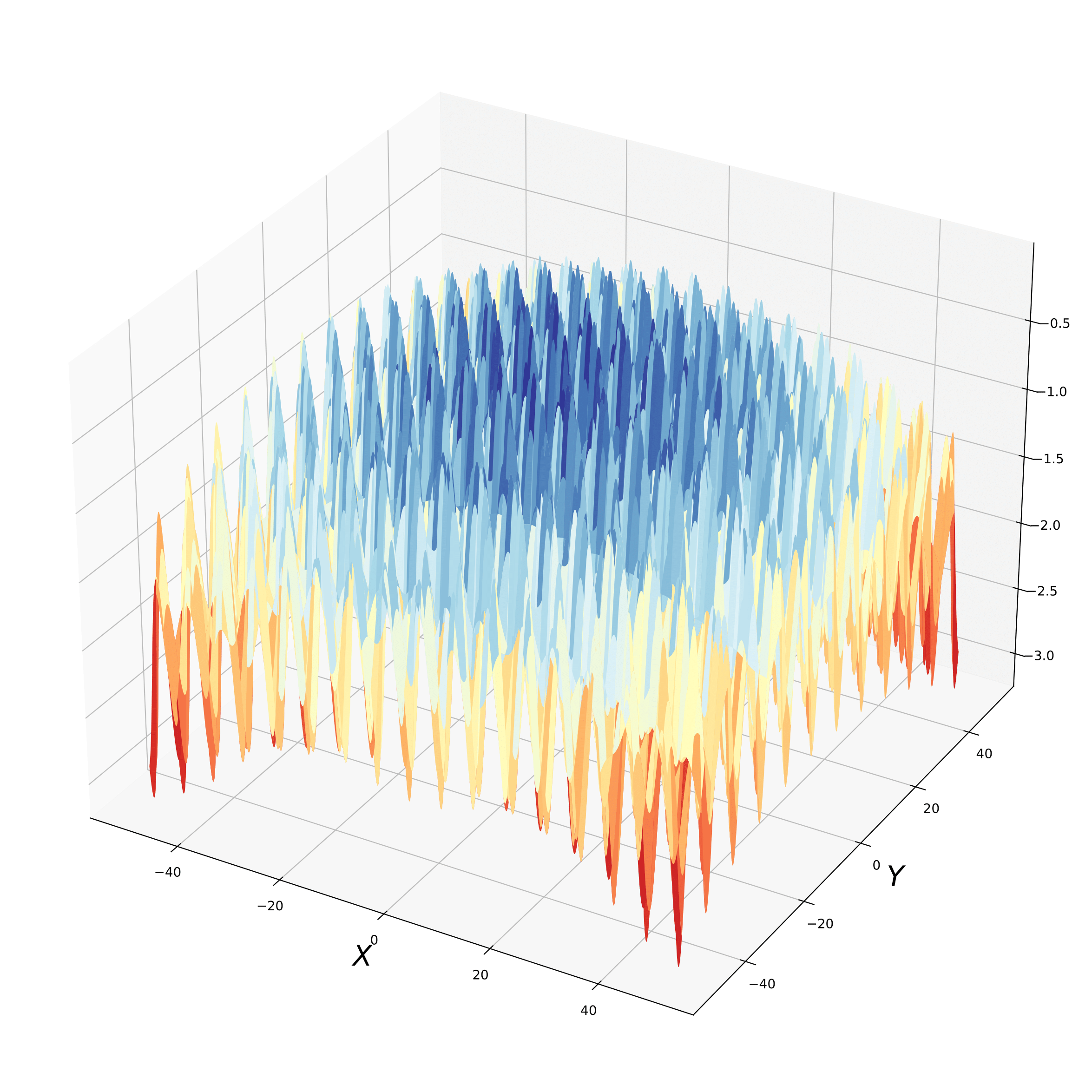}
        \caption{\small  Griewank}
        \label{fig:griewank}
    \end{subfigure}
    \begin{subfigure}[b]{0.13\textwidth}
        \includegraphics[width=\textwidth]{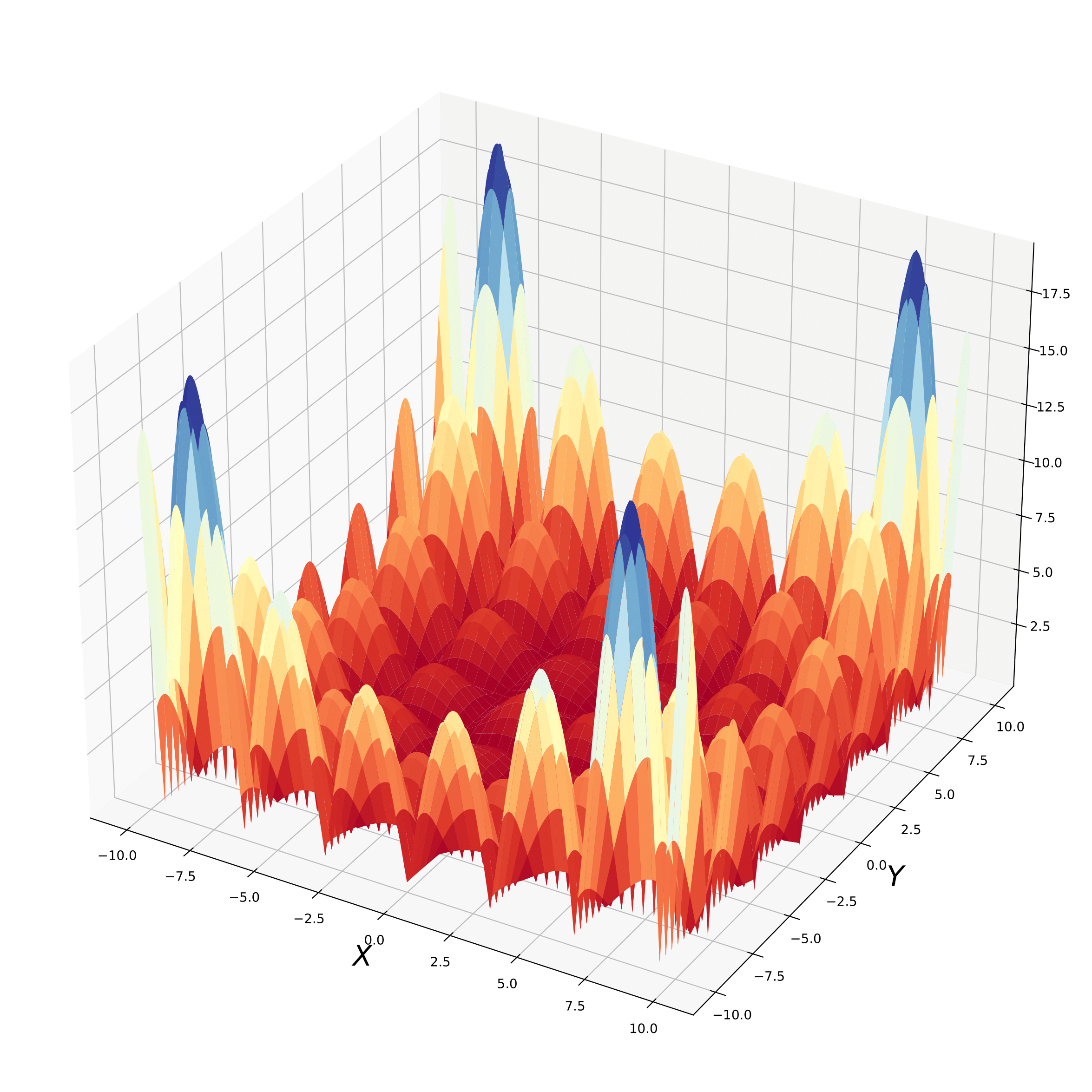}
        \caption{\small  Holder}
        \label{fig:holder}
    \end{subfigure}
    \begin{subfigure}[b]{0.13\textwidth}
        \includegraphics[width=\textwidth]{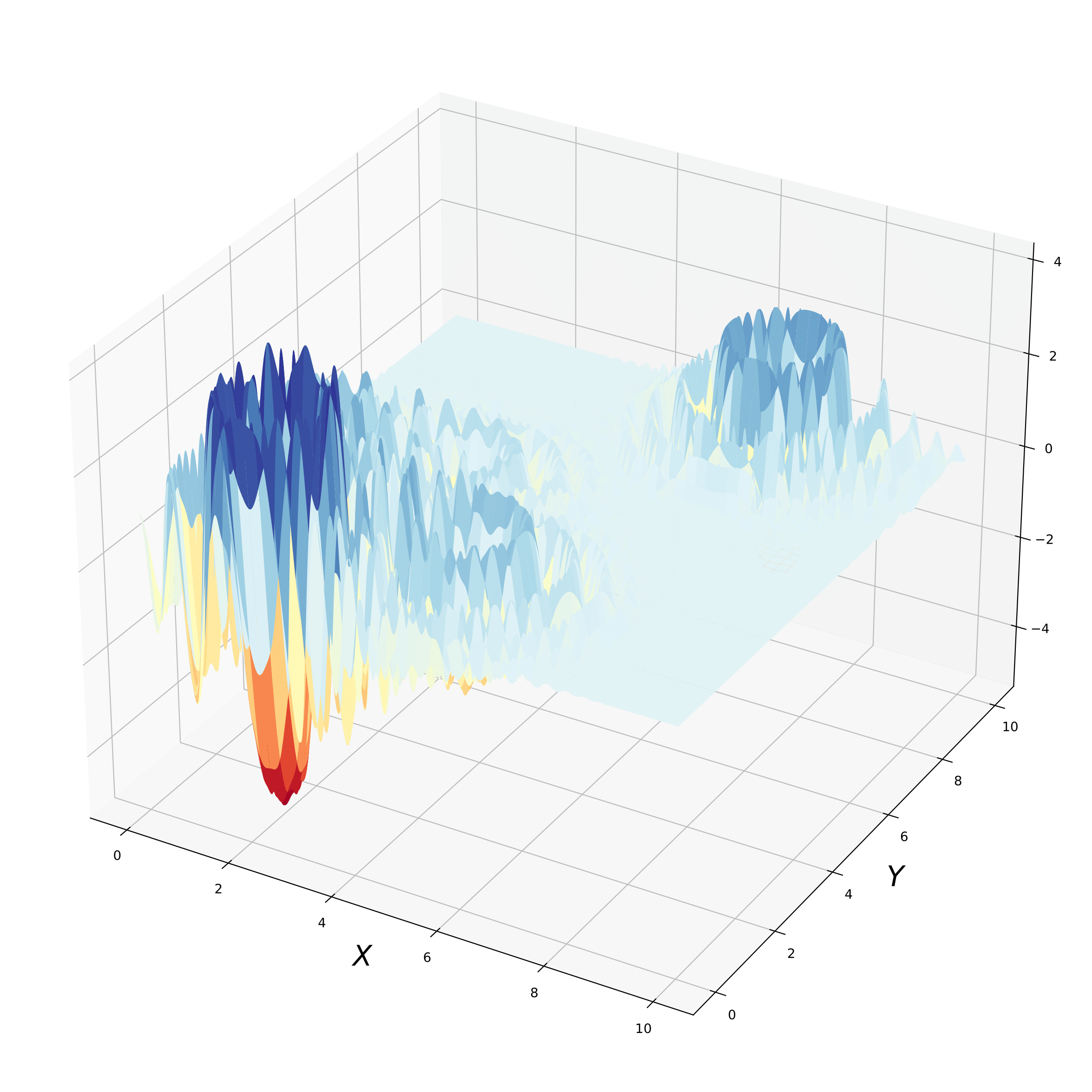}
        \caption{\small  Langermann}
        \label{fig:langermann}
        \end{subfigure}
        % \hfill
    \begin{subfigure}[b]{0.13\textwidth}
        \includegraphics[width=\textwidth]{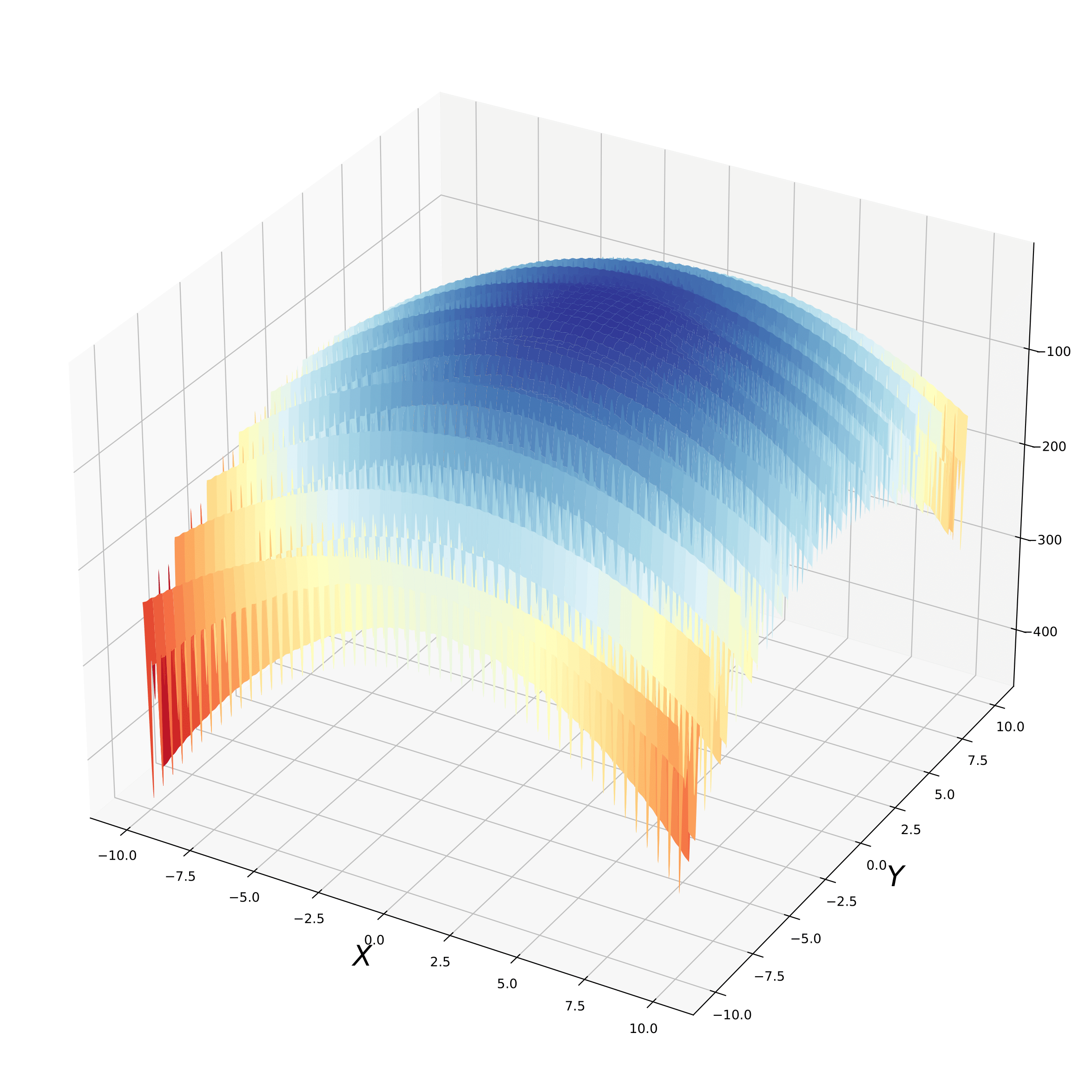}
        \caption{\small  Levy}
        \label{fig:levy}
    \end{subfigure}
     \begin{subfigure}[b]{0.13\textwidth}
        \includegraphics[width=\textwidth]{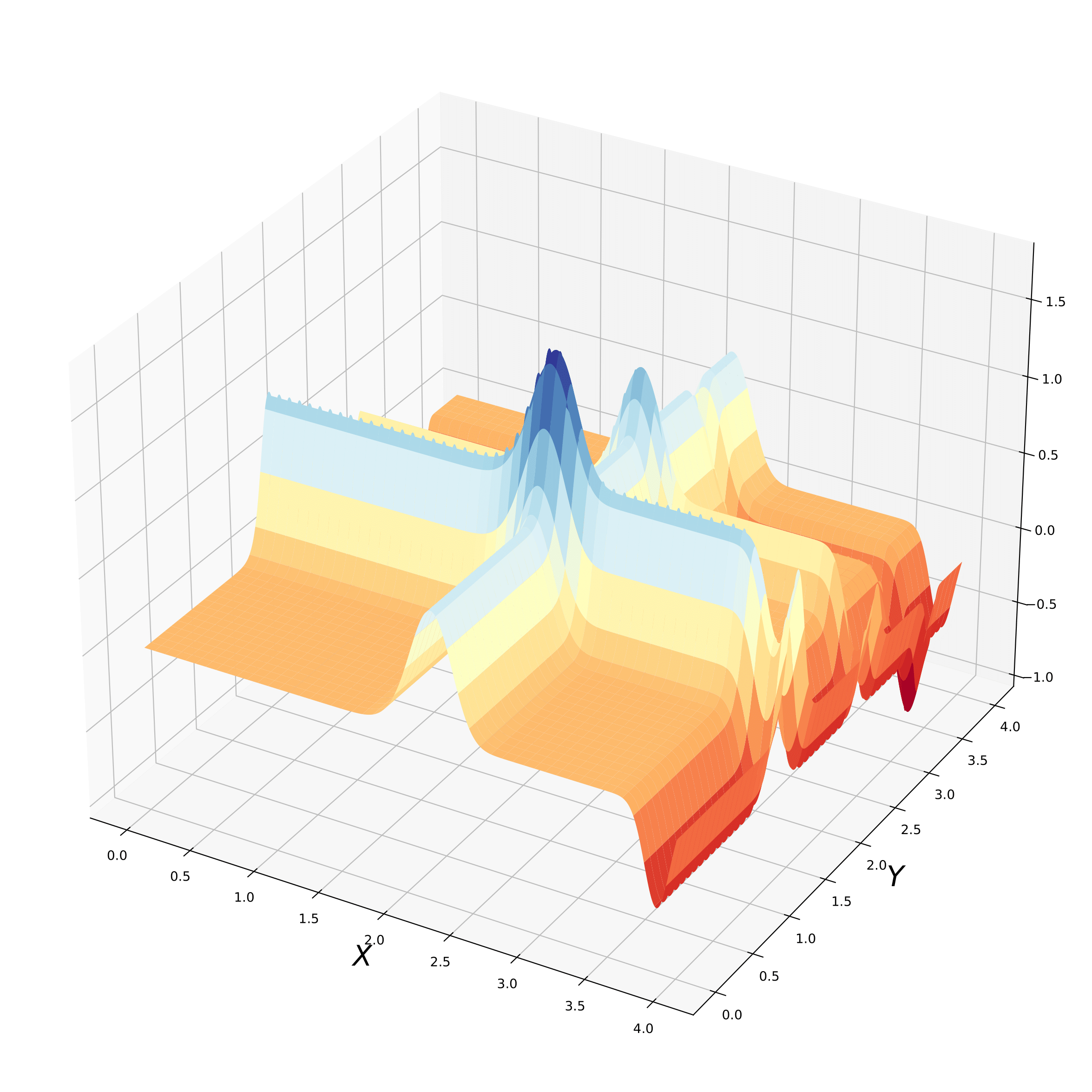}
        \caption{\small  Michalewicz}
        \label{fig:michalewicz}
    \end{subfigure}
    \begin{subfigure}[b]{0.13\textwidth}
        \includegraphics[width=\textwidth]{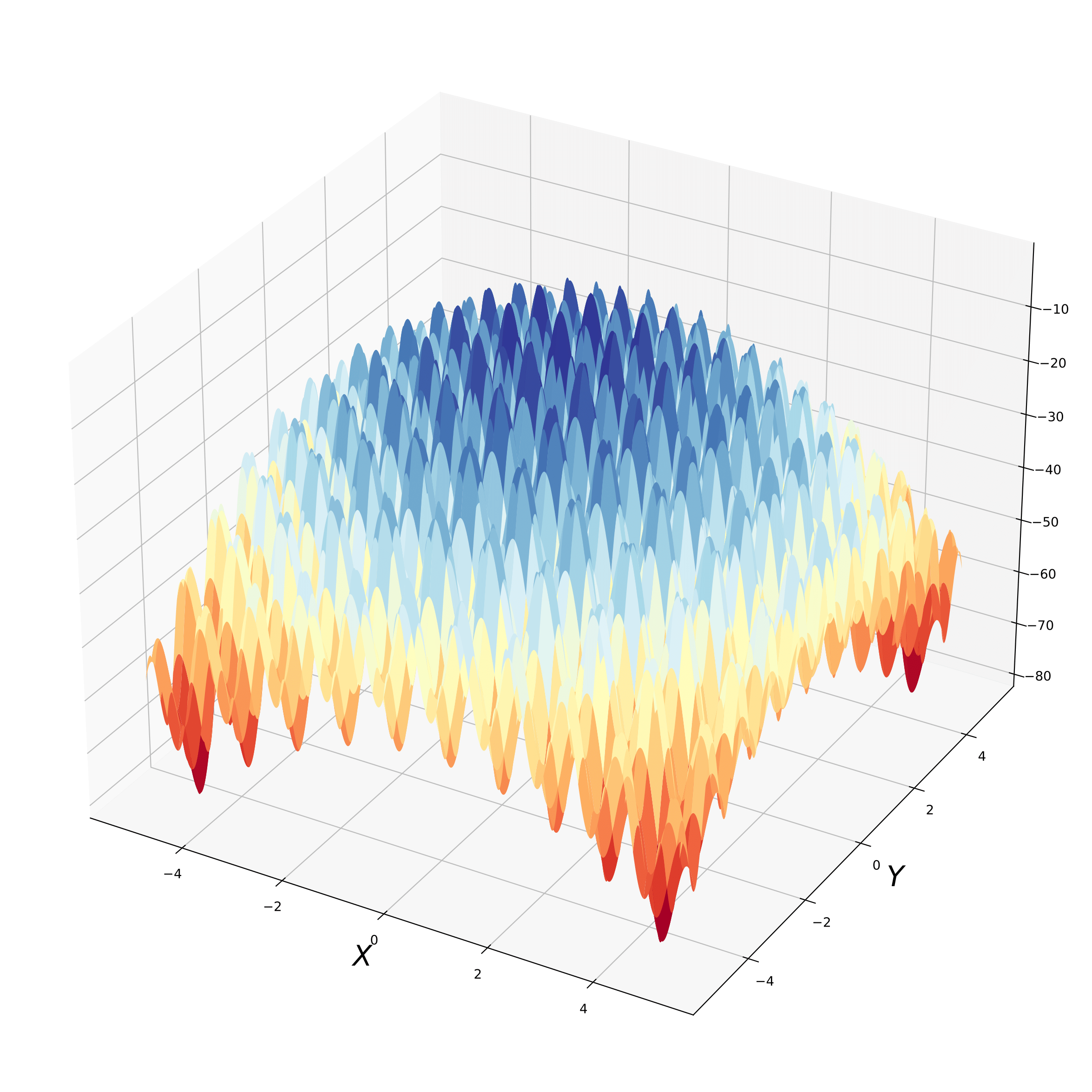}
        \caption{\small  Rastrigin}
        \label{fig:rastrigin}
    \end{subfigure}
    \begin{subfigure}[b]{0.13\textwidth}
        \includegraphics[width=\textwidth]{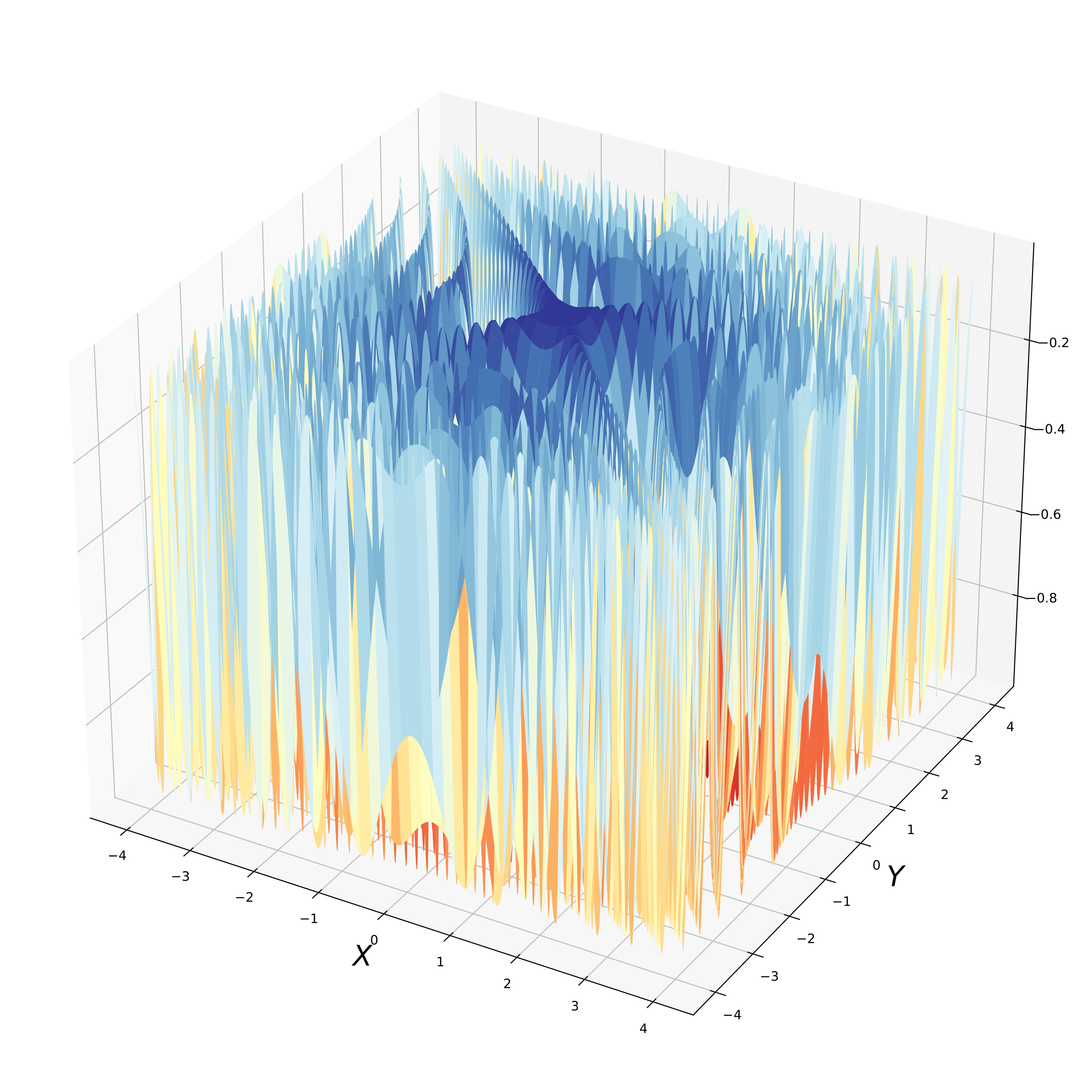}
        \caption{\small  Schaffer}
        \label{fig:schaffer}
    \end{subfigure}
        \begin{subfigure}[b]{0.13\textwidth}
        \includegraphics[width=\textwidth]{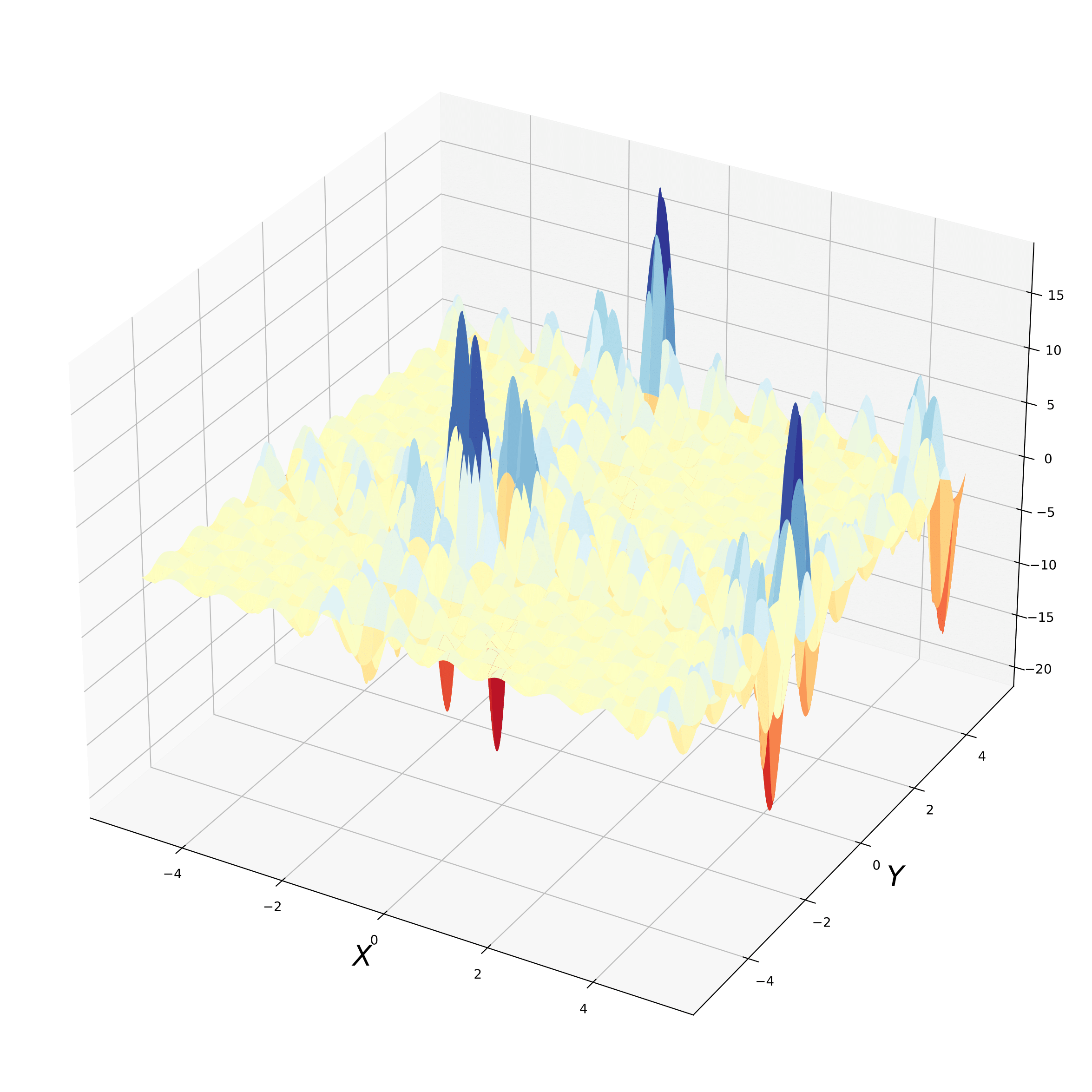}
        \caption{\small  Schubert}
        \label{fig:schubert}
    \end{subfigure}   
    \caption{\small Selected figures of various considered non-convex objective functions with two dimensions. 
    }
    \label{fig:all_figures}
\end{figure*}

\section{Introduction}
\label{into}

Global optimization is an evolving field of optimization \citep{torn1989global, horst2013handbook, floudas2014recent, zabinsky2013stochastic, stork2022new} that seeks to identify the best possible solution across the entire problem space, i.e., the entire set of feasible solutions, ensuring that the global optimum is found or at least well-approximated, even with a limited number of function evaluations. The objective function may be non-convex and exhibit multiple local optima. As a result, local optimization methods may only find a solution that is optimal in a limited region of the search space, potentially yielding a globally sub-optimal result. Moreover, the objective function can be non-differentiable or a black-box function \citep{jones1998efficient}, meaning it is only accessible through direct evaluations. Furthermore, evaluating the objective function can be expensive, requiring substantial amounts of money, time, or energy, which makes the process even more challenging, as it can limit the number of function evaluations.

Despite the challenges, global optimization problems are prevalent in engineering and real-world systems. These problems are applicable in various fields, including mechanical, civil, and chemical engineering, as well as in structural optimization, molecular biology, circuit chip design, and image processing \citep{zabinsky2013stochastic}. More recently, with the emergence of large language models (LLMs), a new line of work has focused on instruction learning in black-box LLMs, such as ChatGPT \citep{openai2023chatgpt}, through calls, without access to their underlying models \citep{pmlr-v235-chen24e, pmlr-v235-lin24r, kharrat2024acing}. Therefore, the development of efficient global optimization algorithms represents an intriguing research direction with the potential for significant impact across several disciplines.

A subfield of global optimization is Lipschitz optimization \citep{shubert1972sequential, piyavskii1972algorithm}, which assumes knowledge of the Lipschitz constant or an upper bound for it. While \cite{torn1989global} observed that practical objective functions often exhibit Lipschitz continuity, the exact value of this constant is seldom known. In this work, we focus on black-box Lipschitz continuous functions with \textit{unknown} constants. In such cases, a common approach is to estimate the Lipschitz constant or an upper bound and use this estimate as a proxy for the true constant \citep{malherbe2017global, serre2024lipo+}. Unlike these previous works, our approach bypass Lipschitz constant estimation—typically reliant on random uniform evaluations—and concentrate solely on evaluating points that are potential maximizers. This significantly improves the efficiency of the global optimization process, particularly for expensive objective functions with tight evaluation budgets.

We introduce \textit{Every Call is Precious} (ECP) a simple yet efficient approach that avoids unpromising uniform random evaluations. The core idea is to start with small acceptance regions, possibly empty, controlled by a variable, $\varepsilon_t$. This region leverages the Lipschitz continuity of the function—without requiring the Lipschitz constant—by using $\varepsilon_t$ and the points observed in previous iterations. The proposed region is theoretically guaranteed, for smaller values of $t$, to be a subset of potential maximizers, ensuring that every evaluated point is a viable candidate for the maximum. ECP gradually expands the acceptance region through a growing sequence of $\varepsilon_t$, progressively including more potential maximizers until all are covered when $\varepsilon_t \geq k$. This method is especially practical for small evaluation budgets, as it ensures that all evaluated points are potential maximizers, avoiding unpromising evaluations.

In the following, we summarize the contributions:

(1) We introduce ECP, a global optimization algorithm that eliminates the need to estimate a Lipschitz constant. Instead, we propose a simple yet effective adaptive search strategy based on a growing sequence of $\varepsilon_t$. The implementation is publicly available at \href{https://github.com/fouratifares/ECP}{\texttt{https://github.com/fouratifares/ECP}}. 

(2) We provide theoretical analysis of ECP, addressing its unique interplay between acceptance region expansion (\cref{lem:montonic_with_epsilon}) and contraction over time (\cref{prop:nonincreasing}). Key results include finite computational complexity (\cref{thm:complexity}), no-regret guarantees in the infinite-budget setting (\cref{thm:proba_convergence}), and minimax optimality with finite budgets (\cref{thm:ecpupperbound}).

(3) Benchmarks against 10 global optimization methods on 30 non-convex, multi-dimensional problems demonstrate that ECP consistently outperforms state-of-the-art Lipschitz, Bayesian, bandits, and evolutionary methods. Furthermore, extensive experiments validate ECP’s robustness to hyperparameter choices.

\section{Related Works}
\label{related}

Several methods have been proposed for global optimization, with the simplest being non-adaptive exhaustive searches, such as grid search, which uniformly divides the space into representative points \citep{zabinsky2013stochastic}, or its stochastic alternative, Pure Random Search (PRS), which employs random uniform sampling \citep{brooks1958discussion, zabinsky2013stochastic}. However, these methods are often inefficient, as they fail to exploit previously gathered information or the underlying structure of the objective function \citep{zabinsky2013stochastic}.

To enhance efficiency, adaptive methods have been developed that leverage collected data and local smoothness. Some of these methods need the knowledge of the local smoothness, including HOO \citep{bubeck2011x}, Zooming \citep{kleinberg2008multi}, and DOO \citep{munos2011optimistic}, while others do not, such as SOO \citep{munos2011optimistic, preux2014bandits, kawaguchi2016global} and SequOOL \citep{pmlr-v98-bartlett19a}. In this work, however, we focus on Lipschitz functions. 

To address Lipschitz functions with unknown Lipschitz constants, the DIRECT algorithm \citep{jones1993lipschitzian, jones2021direct} employs a deterministic splitting approach of the whole space, sequentially dividing and evaluating the function over subdivisions that have recorded the highest upper bounds. 

More recently, \cite{malherbe2017global} introduced AdaLIPO, an adaptive stochastic no-regret strategy that estimates the Lipschitz constant through uniform random sampling, which is then used to identify potentially optimal maximizers based on previously explored points. Later, AdaLIPO+ \citep{serre2024lipo+} was introduced as an empirical enhancement over it, reducing the exploration probability over time. Both approaches optimize the search space using an acceptance condition, yet they necessitate additional uniform random evaluations, making them less efficient in small-budget scenarios.

Under alternative assumptions, various global optimization methods have been proposed. For instance, Bayesian optimization (BO) \citep{fernando2014bayes, 7352306, frazier2018tutorial, balandat2020botorch} constructs a probabilistic model of the objective function and uses it to evaluate the most promising points, making it particularly effective for global optimization. 

While several BO algorithms are theoretically guaranteed to converge to the global optimum of the unknown function, they often rely on the assumption that the kernel's hyperparameters are known in advance. To address this limitation, hyperparameter-free approaches such as Adaptive GP-UCB (A-GP-UCB) \citep{JMLR:v20:18-213} have been proposed. More recently, \citet{JMLR:v23:21-0888} introduced SMAC3 as a robust baseline for global optimization. In our empirical evaluation, we show that ECP outperforms these recent baselines from BO.

Other approaches, such as CMA-ES \citep{hansen1996adapting, hansen2006cma, hansen2019pycma}, and simulated annealing \citep{metropolis1953equation, kirkpatrick1983optimization}, later extended to Dual Annealing \citep{xiang1997generalized, tsallis1988possible, tsallis1996generalized}, are also notable, although they do not guarantee no-regret \citep{malherbe2017global} or theoretical finite-budget guarantees for Lipschitz functions.

Other related approaches include contextual bandits \citep{auer2002using, langford2007epoch, filippi2010parametric, valko2013finite, lattimore2020bandit}, such as the NeuralUCB algorithm \citep{zhou2020neural}, which leverages neural networks to estimate upper-confidence bounds. While NeuralUCB is not primarily designed for global maximization, it can be adapted by randomly sampling points, estimating their bounds, evaluating the point with the highest estimate, and retraining the network. However, it may be inefficient for small budgets, as neural networks require a large number of samples to train effectively. Finally, other works on bandits address black-box discrete function maximization \citep{fourati2023randomized, fourati2024combinatorial, pmlr-v235-fourati24b}, which is not the focus of this work.

\section{Problem Statement}
\label{sec:problem_statement}

In this work, we consider a black-box, non-convex, deterministic, real-valued function $f$, which may be expensive to evaluate—requiring significant time, energy, or financial resources. The function is defined over a convex, compact set $\mathcal{X} \subset \mathbb{R}^d$ with a non-empty interior and has a maximum over its input space\footnote{For all $x \in \R^d$, 
we denote its $\ell_2$-norm as $\norm{x}_2= (\sum_{i=1}^d x_i^2)^{\frac{1}{2}}$.
We define $B(x,r)=\{x' \in \R^d: \norm{x-x'}_2 \leq r \}$
the  ball centered in $x$ of radius $r\geq 0$. For any bounded set $\X \subset \R^d$, we define its radius as
$\textrm{rad}(\X)=\max \{r>0: \exists x\in \X \textrm{~where~} B(x,r)\subseteq \X \}$ and its diameter as $\diam{\X}=\max_{(x,x')\in \X^2}\norm{x-x'}_2$.}. 

The objective of this work is global maximization, seeking a global maximizer, defined as follows:
$$
x^{\star} \in \underset{x \in \X}{\arg \max}~ f(x) 
$$
with a minimal number of function calls.  Starting from an initial point $x_1$ and given its function evaluation $f(x_1)$, adaptive global optimization algorithms leverage past observations to identify potential global optimizers. Specifically, depending on the previous evaluations
$(x_1, f(x_1)), \cdots, (x_t, f(x_t)) $, it chooses at each iteration $t \geq 1$ a point $x_{t+1}\in\X$ to evaluate and receives its function evaluation
 $f(x_{t+1})$. After $n$ iterations, the algorithm returns $x_{\hatin}$, one of the evaluated points,
where $\hatin \in \arg\max_{i=1, \cdots, n}f(x_i)$, representing the point with the highest evaluation.

To assess the performance of an algorithm \(A \in \mathcal{G}\), where \(\mathcal{G}\) is the set of global optimization algorithms, over the function \(f\), we consider its regret after \(n\) iterations, i.e., after evaluating \(x_1, \dots, x_n\) by $A$, as follows:
\begin{equation}
\label{regret}
\mathcal{R}_{A,f}(n) = \max_{x \in \X} f(x) - \max_{i=1, \cdots, n}f(x_i),    
\end{equation}
measuring the difference between the true maximum and the best evaluation over the $n$ iterations.

We consider $f$ to be Lipschitz with an unknown finite Lipschitz constant $k$, i.e., there exists an unknown $k \geq 0$, such that for any two points, $x$ and $x'$ in $\X$, the absolute difference between $f(x)$ and $f(x')$ is no more than $k$ times the distance between $x$ and $x'$, i.e., 
$$
\forall (x,x')\in\X^2 \quad |f(x) -f(x')| \leq k \cdot \norm{x-x'}_2~
\!.
$$
Morever, we denote the set of Lipschitz-continuous functions defined on $\X$, with a Lipschitz constant $k$, as $\text{Lip}(k) := \{f : \X \to \R \text{~s.t.~} | f(x) -f(x') | \leq k \cdot \norm{x -x'}_2, ~\forall  (x,x')\in \X^2\}$
and their union $\bigcup_{k\geq 0} \Lip(k)$  denotes the set of all Lipschitz-continuous functions.

We define the notion of no-regret, equivalent to optimization consistency \citep{malherbe2017global}, where the best evaluation converges to the true maximum in probability, which provides a formal guarantee that the algorithm's regret diminishes with an increasing budget $n$ (number of evaluations).
\begin{definition}
\label{def:consistency}
{\sc (No-Regret)} An algorithm \( A \in \mathcal{G} \) is no-regret over a set \( \mathcal{F} \) of real-valued functions, having a maximum over their domain \( \mathcal{X} \), if and only if:
\begin{equation*}
\forall f \in \mathcal{F}, \quad \mathcal{R}_{A,f}(n) \xrightarrow{p} 0,
\end{equation*}
where \( \mathcal{R}_{A,f} \), defined in Eq.~\eqref{regret}, represents the regret of algorithm \( A \) after \( n \) evaluations of function \( f \).
\end{definition}

In fact, finite-time lower bounds on the minimax regret can be established for the class of functions with a fixed Lipschitz constant $k$, as we recall below.

\begin{proposition}(\cite{bull2011convergence})
\label{prop:minimax}
{\sc (Lower Bound)}
For any Lipschitz function, $f \in \text{Lip}(k)$, with any constant $k \geq 0$ 
and any 
$n \in \mathbb{N}^{\star}$, we have
$$
\begin{aligned}
\inf_{A \in \mathcal{G}}
\sup_{f \in \normalfont{\text{Lip}}(k)}
\esp{ \mathcal{R}_{A,f}(n) }
\geq c \cdot k \cdot n^{-\frac{1}{d}}
\end{aligned}
$$
where $c = \rad{\X}/(8\sqrt{d})$. The expectation is taken over the $n$ evaluations of $f$ by the algorithm $A$.
\end{proposition}

The minimax regret lower bound for optimizing Lipschitz functions shows that the best achievable regret decays as $n^{-\frac{1}{d}}$, where $n$ is the budget of function evaluations and $d$ is the input space dimensionality, which underscores the difficulty of high-dimensional optimization. The regret bounds also scale with the Lipschitz constant $k$, indicating functions with smaller $k$ are easier to optimize. This result sets performance limits for algorithms, with the best expected regret bounded by $\Theta(k \cdot n^{-\frac{1}{d}})$, which can be recovered by any algorithm with a covering rate of $\mathcal{O}(n^{-1/d})$ \citep{malherbe2017global}. However, the objective functions used to establish the lower bound of $\Omega(k n^{-1/d})$ often feature spikes that are nearly constant across most of the domain, offering limited practical relevance. Therefore, we focus on developing an algorithm that not only meets theoretical guarantees but also demonstrates outstanding performance across a wide variety of non-convex multi-dimensional objective functions.

\section{Every Call is Precious (ECP)}
\label{sec:algorithm}

In this section, we present ECP, an efficient algorithm, for maximizing 
an unknown (possibly expensive) function $f$ without knowing its Lipschitz constant $k\geq 0$.

\begin{algorithm}[t]
\small
\caption{\small  Every Call is Precious (ECP)}
\textbf{Input:} $n\in\mathbb{N}^{\star}$, $\varepsilon_1 > 0$, $\tau_{n,d} > 1$, $C>1$, $\X\subset\mathbb{R}^d$, $f$
\begin{algorithmic}[1]
\State Let $x_1 \sim \mathcal{U}(\X)$, Evaluate $f(x_1)$
\State $t \leftarrow 1$, $\quad h_{1} \leftarrow 1$, $\quad h_{2} \leftarrow 0$%, $\quad v_{1} \leftarrow 0$ 
\While{$t < n$} 
    \State Let $x_{t+1} \sim \mathcal{U}(\X)$, $\quad h_{t+1} \leftarrow h_{t+1}+1$,
    \If{$(h_{t+1} - h_{t}) > C$} $\quad$ \Comment{({Growth Condition})}
        \State $\varepsilon_t \leftarrow \tau_{n, d} \cdot \varepsilon_t$, $\quad h_{t+1} \leftarrow 0$%, \quad v_t \leftarrow v_t + 1$
    \EndIf
    \If{$x_{t+1} \in \mathcal{A}_{\varepsilon_t,t}$} $\quad$ \Comment{({Acceptance Condition})}
        \State Evaluate $f(x_{t+1})$
        \State $t \leftarrow t+1, \quad h_{t} \leftarrow h_{t+1}$%, \quad v_{t+1} \leftarrow 0$
        \State $\varepsilon_{t+1} \leftarrow \tau_{n, d} \cdot \varepsilon_{t}, \quad h_{t+1} \leftarrow 0$ %\Comment{({Growth})}
    \EndIf
\EndWhile
\State Return $x_{\hat{i}}$ where $\hat{i} \in \arg\max_{i=1, \cdots, n} f(x_i)$
\end{algorithmic}
\label{ALG:ECP_ALGORITHM}
\end{algorithm}

Our proposed algorithm, ECP, presented in \cref{ALG:ECP_ALGORITHM}, takes as input the number of function evaluations $n \in \mathbb{N}^\star$ (budget), the search space $\mathcal{X}$, the black-box function $f$, a value $\varepsilon_1 > 0$, a coefficient $\tau_{n,d} > 1$, and a constant $C > 1$. The algorithm begins by sampling and evaluating a point $x_1$ uniformly at random from the entire space $\mathcal{X}$ (line 1). It then proceeds through $n-1$ rounds (each round concludes after one function evaluation), where in each round $t \geq 1$ (up to $t = n-1$), a random variable $x_{t+1}$ is repeatedly sampled uniformly from the input space $\mathcal{X}$ until a sample meets the acceptance condition. Once the condition is satisfied, the sample is evaluated, and the algorithm moves to the next round $t+1$.

ECP accepts (evaluates) the sampled point $x_{t+1} = x$ if and only if the following inequality is verified:
$$
\min_{i=1, \cdots, t} ( f(x_i) + \varepsilon_t \cdot \norm{x-x_i}_2 )
\geq \max_{j=1, \cdots, t}f(x_j),
$$
where $\varepsilon_t$ is a growing sequence staring from $\varepsilon_1$ and continuously multiplied by a coefficient $\tau_{n,d} >1$. An illustration of the acceptance region for a non-convex, single-dimensional objective function can be found in \cref{fig:illustration}, where it can be seen that $\varepsilon_t$ controls the acceptance region size. The coefficient $\tau_{n,d} > 1$ is some non-decreasing function of $n$ and $d$, such as $\tau_{n,d} = \max \{1+\frac{1}{nd}, \tau \} \geq \tau > 1$.

The algorithm tracks the number of sampled but rejected points during each iteration $t \geq 1$ before increasing $\varepsilon_t$, using the variable $h_{t+1}$. This variable is initialized to zero (lines 2) and is reset to zero whenever $\varepsilon_t$ is increased (line 6 and 11). Additionally, $h_{t}$ is initialized with the number of rejections from the previous iteration (lines 2 and 10), before acceptance at iteration $t-1$. When the difference between the current and the previous number of rejections exceeds a given threshold $C > 1$ (line 5), $\varepsilon_t$ is increased by multiplying it with the factor $\tau_{n,d} > 1$. This growth condition is further analyzed in \cref{prop:rejection_growth}. 

Thus, $\varepsilon_t$ grows when a rapidly increasing number of samples is generated without an accepted point (lines 5-6) and also when a sample is evaluated (line 11). While the former type of growth is stochastic and depends on the threshold $C > 1$, the latter is deterministic and occurs regardless of the input parameters. Consequently, it follows that $\varepsilon_t \geq \varepsilon_1 \tau_{n,d}^{t-1}$.

Formally, for any value of $\varepsilon_t$ and at any given time step $t$, we define $\mathcal{A}_{\varepsilon_t,t}$ as the acceptance region, presenting the set of potentially accepted points at round $t$ with current $\varepsilon_t$. 
The choice of this condition will be motivated and analyzed in further details in \cref{section:acceptance_condition_analysis}.

\begin{figure}[t] 
\centering
\includegraphics[width=0.34\textwidth]{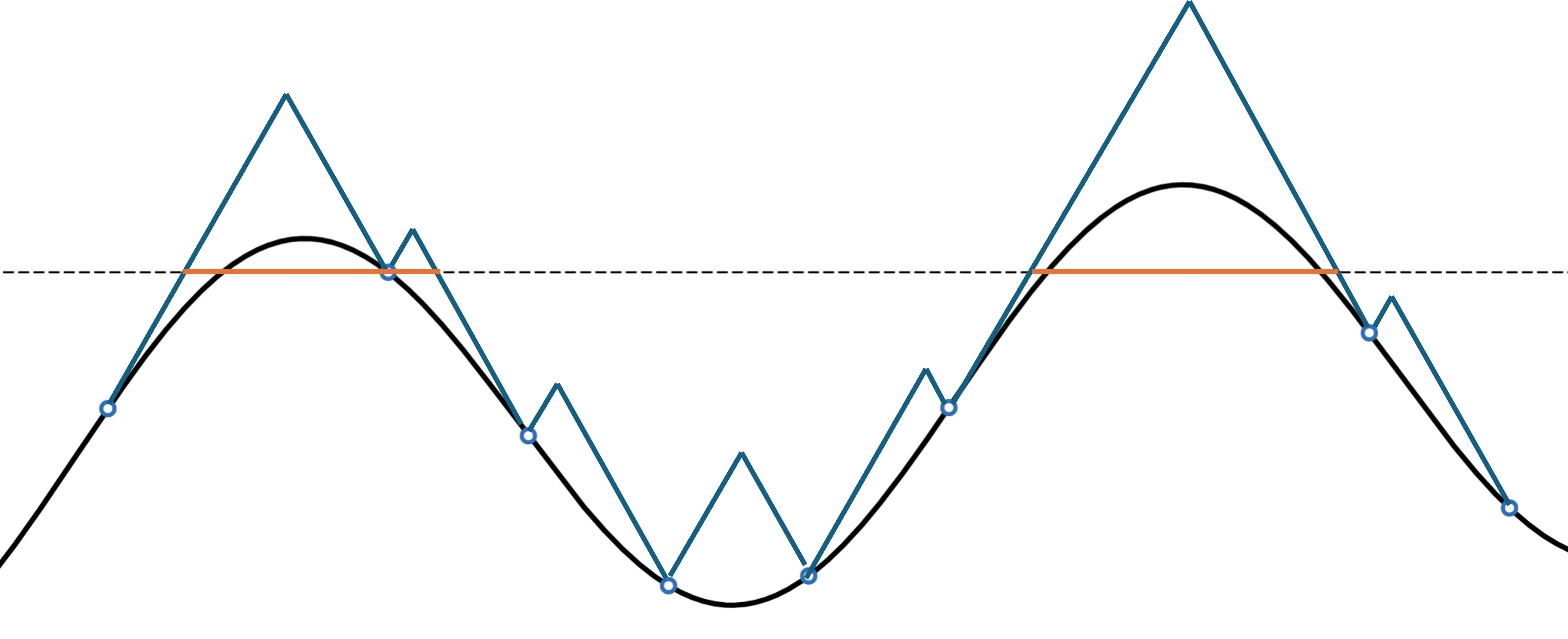} 
    \caption{\small An illustration of the acceptance region (orange), which is determined based on the 8 evaluated points on a non-convex, single-dimensional objective function (black). The $\varepsilon_8$ controls the slopes of the blue functions, directly impacting the acceptance region.}
    \label{fig:illustration} 
\end{figure}

\begin{definition}
\label{def:potentially_accepted}
{\sc (Acceptance Region)}
The set of points potentially accepted by ECP at any time step $t \geq 1$, for any value $\varepsilon_t > 0$, is defined as:
$
 \mathcal{A}_{\varepsilon_t,t} \triangleq \{x \in \X: 
 \min_{i=1, \cdots, t}  (f(x_i) + \varepsilon_t \cdot \norm{x-x_i}_2)
\geq \max_{j=1, \cdots, t}f(x_j) \}. 
$
\end{definition}

This acceptance, condition as shown later in in \cref{prop:potential} in \cref{section:acceptance_condition_analysis}, ensures that the points potentially accepted are optimal for smaller values of $\varepsilon_t$ and, as $\varepsilon_t$ increases, encompass all potentially optimal points. Therefore, unlike PRS, which accepts any sampled point, or other methods that allow for some evaluations of uniform random samples over the entire space $\X$ (to estimate the Lipschitz constant $k$), ECP evaluates a point only if it falls within a smaller subspace of $\X$ containing potential maximizers.

Notice that the acceptance region depends on the previously explored points until iteration time $t$ and the current value of $\varepsilon_t$. As shown in the later analysis, the acceptance region exhibits intersting properties: it is non-decreasing with increasing values of $\varepsilon_t$ at a given iteration $t$ (for any time step $t$ and any $u \leq v$, we have $\mathcal{A}_{\varepsilon_t = u,t} \subseteq \mathcal{A}_{\varepsilon_t = v,t}$, as shown in \cref{lem:montonic_with_epsilon}), and it is non-increasing with increasing time steps $t$ for a given value of $\varepsilon_t$ (for any $\varepsilon_t=\varepsilon$ and any $t_1 \leq t_2$, we have $\mathcal{A}_{\varepsilon_t = \varepsilon,t_2} \subseteq \mathcal{A}_{\varepsilon_t = \varepsilon, t_1}$, as shown in \cref{prop:nonincreasing}). Therefore, the smaller the $\varepsilon_t$ and the larger the time $t$, the more points are rejected, as verified by the increasing upper bound on the rejection probability in \cref{prop:rejection_proba}. 
To balance this potential growth in rejections, $\varepsilon_t$ is designed to continuously grow by a multiplicative constant $\tau_{n,d} > 1$, both to include more potential points and to mitigate the risk of exponential growth in rejections as the time step $t$ increases, which ensures both a fast algorithm and guaranteed convergence, with polynomial computational complexity, as shown in \cref{thm:complexity}. Regret analysis of ECP are presented in \cref{sec:regret_analysis}.

\begin{remark}{\sc (Extension to other smoothness assumptions)}  
The proposed optimization framework can be generalized to encompass a broad class of globally and locally smooth functions by making slight modifications to the decision rule. As an example, consider the set of functions analyzed by \cite{munos2014bandits}, characterized by a unique maximizer \( x^\star \) and satisfying \( f(x^\star) - f(x) \leq \ell(x^\star, x) \) for all \( x \in X \), where \( \ell : X \times X \to \mathbb{R}^+ \) is a semi-metric defining local smoothness around the maxima. By adapting \cref{prop:potential}, the decision rule for selecting \( X_{t+1} \) can be reformulated as testing whether  
$
\max_{i=1,\ldots,t} f(X_i) \leq \min_{i=1,\ldots,t} f(X_i) + \ell(X_{t+1}, X_i).
$ 
\end{remark}

\section{Theoretical Analysis}
\label{sec:theory_analysis}

In the following, we provide theoretical analysis of ECP. First, we motivate the considered acceptance region, then we analyze the rejection growth and the computational complexity, and finally, we show that ECP is no-regret with optimal minimax regret bound.

\subsection{Acceptance Region Analysis}
\label{section:acceptance_condition_analysis}

In this section, we motivate the proposed acceptance region and the design of the algorithm with respect to the growing $\varepsilon_t$. The acceptance region is inspired by the previously studied active subset of consistent functions in active learning \citep{dasgupta2011two, hanneke2011rates, malherbe2017global}. Hence, we start by the definition of consistent functions.

\begin{definition}
\label{def:consistent_functions}
{\sc (Consistent functions)}
The active subset of Lipschitz functions, with a Lipschitz constant $k$, consistent with the black-box function $f$ over $t\geq1$ evaluated samples 
$(x_1,f(x_1)), \cdots,(x_t,f(x_t))$ is:
$
\mathcal{F}_{k,t} \triangleq \left\{ g \in \Lip(k) : \forall i \in\{ 1, \cdots, t\},
~ g(x_i) = f(x_i) \right\}.
$
\end{definition}

Using the above definition of a consistent function, we define the subset of points that can maximize at least some function $g$ within that subset of consistent functions and possibly maximize the target $f$.

\begin{definition}
\label{def:potential_maximizers}
{\sc (Potential Maximizers)}
For a Lipschitz function \( f \) with a Lipschitz constant \( k \geq 0 \), let \( \mathcal{F}_{k,t} \) be the set of consistent functions with respect to \( f \), as defined in Definition \ref{def:consistent_functions}. For any iteration \( t \geq 1 \), the set of potential maximizers is defined as follows:
$
\mathcal{P}_{k,t} \triangleq \left\{ x \in \mathcal{X} : \exists g \in \mathcal{F}_{k,t} \text{ where } x \in \underset{x \in \mathcal{X}}{\arg \max}~g(x) \right\}.
$
\end{definition}

We can then show the relationship between the potential maximizers and our proposed acceptance region. But first, we demonstrate an important characteristic of our acceptance region, which is being a no-decreasing region, function of the value of $\varepsilon_t$, as shown in \cref{prooflem:montonic_with_epsilon}.

\begin{lemma}
\label{lem:montonic_with_epsilon}
{\sc(Expanding with Respect to $\varepsilon_t$)}  
Let $u, v > 0$ be two values of $\varepsilon_t$ such that $u \leq v$. Then, the set of potentially accepted points at time $t$ corresponding to $\varepsilon_t = u$ is a subset of the set of actions at time $t$ corresponding to $\varepsilon_t = v$, i.e.,
$
\mathcal{A}_{\varepsilon_t=u,t} \subseteq \mathcal{A}_{\varepsilon_t=v,t}.
$
\end{lemma}
Therefore, for a fixed iteration $t$, by design of our algorithm, which increases $\varepsilon_t$, the acceptance region is non-decreasing. Using the above result in \cref{lem:montonic_with_epsilon} and \cref{prop:potential_k} from \cref{app:prelimnaries}, we derive in \cref{proof:prop:potential} the relationship between our considered acceptance region and the set of potential maximizers, summarized in the following proposition.

\begin{proposition}
{\sc (Potential Optimality)} 
\label{prop:potential}
For any iteration $t$, if $\mathcal{P}_{k,t}$ denotes the set of potential maximizers of $f \in \Lip(k)$, as in \cref{def:consistent_functions}, and $\mathcal{A}_{\varepsilon_t,t}$ denotes our acceptance region, defined in \cref{def:potentially_accepted},
then:
$
\begin{aligned}
\forall \varepsilon_t \leq k, \quad \mathcal{A}_{\varepsilon_t,t} \subseteq  \mathcal{P}_{k,t}, \text{ and } \forall \varepsilon_t > k, \quad \mathcal{P}_{k,t} \subseteq \mathcal{A}_{\varepsilon_t,t}.
\end{aligned}
$
\end{proposition}

Therefore, from the above proposition, we conclude that starting with an arbitrarily small $\varepsilon_1$, smaller than or equal to the unknown Lipschitz constant $k$, ECP evaluates the function only over sampled points that are potential maximizers of the unknown function $f$. Furthermore, when $\varepsilon_t$ reaches or exceeds $k$, i.e., $\varepsilon_t \geq k$, which is unavoidable with a growing number of evaluations, the acceptance space does not exclude any potential maximizer, as all potential maximizers remain within the acceptance condition, which is crucial to guarantee the no-regret property of ECP.

Beyond the aforementioned inclusions, it can be seen that both our acceptance region $\mathcal{A}_{\varepsilon_t,t}$ and the true set of potential maximizers $\mathcal{P}_{k,t}$ are functions of the time step $t$. In fact, the set of consistent functions $\mathcal{F}_{k,t}$ is non-increasing as the number of evaluations increases. Consequently, the set of potential maximizers of at least one of these functions, $\mathcal{P}_{k,t}$, also becomes non-increasing with an increasing number of evaluations. We show in \cref{proof:prop:nonincreasing} that, in addition to the inclusions in \cref{prop:potential}, our acceptance region follows the same trend with increasing iteration steps $t$. That is, our acceptance region is a non-increasing region with respect to $t$, as provided in \cref{prop:nonincreasing}.

\begin{proposition} {\sc (Shrinking with respect to $t$)} \label{prop:nonincreasing}
For any $\varepsilon_t = \varepsilon$, for any $t_1, t_2 \geq 1$ such that $t_1 \leq t_2$, we have $\mathcal{P}_{k, t_2} \subseteq \mathcal{P}_{k, t_1}$ and
$\mathcal{A}_{\varepsilon_t=\varepsilon, t_2} \subseteq \mathcal{A}_{\varepsilon_t=\varepsilon, t_1}$.
\end{proposition}
Thus, given some fixed value of $\varepsilon_t = \varepsilon$, as \( t \) increases, accepting points becomes increasingly difficult.  

Now, rethinking \cref{prop:potential}, we note that for \( \varepsilon_t \leq k \), only potential maximizers are evaluated. A curious reader might ask: if it is guaranteed that \( \mathcal{A}_{\varepsilon_t, t} \subseteq \mathcal{P}_{k, t} \) for \( \varepsilon_t \leq k \)—meaning all accepted points are potential maximizers—\textit{why expand the acceptance region?}

Both \cref{prop:potential} and \cref{prop:nonincreasing} illustrate the relationship between our acceptance region and the set of potential maximizers, and together they motivate the growth of $\varepsilon_t$ over the iterations for two key reasons. First, as indicated in \cref{prop:potential}, increasing \( \varepsilon_t \) ensures that we do not overlook any potential maximizers. Specifically, it is only when \( \varepsilon_t \) reaches or exceeds \( k \) that all potential maximizers fall within the acceptance region, i.e., \( \mathcal{P}_{k,t} \subseteq \mathcal{A}_{\varepsilon_t,t} \). Second, as stated in \cref{prop:nonincreasing}, for a fixed \( \varepsilon_t \), the acceptance region is a non-increasing set with respect to \( t \), which can lead to a growing probability of rejection—potentially in an exponential manner. To counteract this exponential growth, our algorithm increases the value of \( \varepsilon_t \), thereby preventing the rejection rate from becoming unsustainable. In the following section, we further analyze the rejection probability of sampled points from the region \( \mathcal{X} \). From this analysis, we derive guarantees on the likelihood of accepting points after a constant amount of growth, which ensures the probabilistic termination of the ECP algorithm in polynomial time.

\subsection{Rejection Growth Analysis and Computational Complexity}
\label{sec:complexity}

The result in \cref{prop:nonincreasing} demonstrates that the acceptance region is non-increasing over time, with respect to iteration $t$, when a constant $\varepsilon_t$ is used, leading to a non-decreasing probabilistic rejection of sampled points. Furthermore, the result in \cref{lem:montonic_with_epsilon} shows that the acceptance region increases with rising $\varepsilon_t$ values in a given iteration $t$, resulting in a decreasing probabilistic rejection of sampled points. When $\varepsilon_t$ increases within the same iteration $t$, it is scaled by a multiplicative factor $\tau_{n,d}>1$ whenever growth is detected, i.e., $\varepsilon_t$ becomes $\varepsilon_t \tau_{n,d}^{v_t}$, where $v_t$ represents the number of growth detection within iteration $t$. 

Consider $\Delta=\max _{x \in \mathcal{X}} f(x)-\min _{x \in \mathcal{X}} f(x)$, $\lambda$ the standard Lebesgue measure that generalizes the notion of volume of any open set, and $\Gamma(x)=\int_0^{\infty} t^{x-1} e^{-t} \mathrm{~d} t$. In what follows, we characterize this rejection growth by providing an upper bound on the probability of rejection in \cref{prop:rejection_proba}, proved in \cref{proof:rejection_proba}, function of the algorithm constants $\varepsilon_1>0$ and $\tau_{n,d} > 1$, with $\varepsilon_t \geq \varepsilon_1 \cdot \tau_{n,d}^{(t-1)}$, and $v_t$ which depends on $C\geq1$.

\begin{proposition}{\sc (ECP Rejection Probability)} \label{prop:rejection_proba}  For any Lipschitz function $f$, let $\left(x_i\right)_{1 \leq i \leq t}$ be the previously evaluated points of ECP until time $t$, and let $v_t$ the number of increases of $\varepsilon_t$ at iteration $t$ (i.e., the number of times we validate growth condition in iteration $t$). For any $x \in \mathcal{X}$, let $R(x, t, v_t)$ be the event of rejecting $x$ at time $t+1$ after $v_t$ growths, we have:
$$
\begin{aligned}
 \mathbb{P}\left(R(x, t+1, v_t)\right) &\leq \frac{t (\sqrt{\pi}\Delta)^{d} }{\varepsilon_1^d \tau_{nd}^{(t-1)d}\tau_{n,d}^{v_t d} \Gamma(d / 2+1) \lambda(\mathcal{X})}.  
\end{aligned}
$$
\end{proposition}

\begin{figure*}[t] 
\centering
\includegraphics[width=0.83\textwidth]{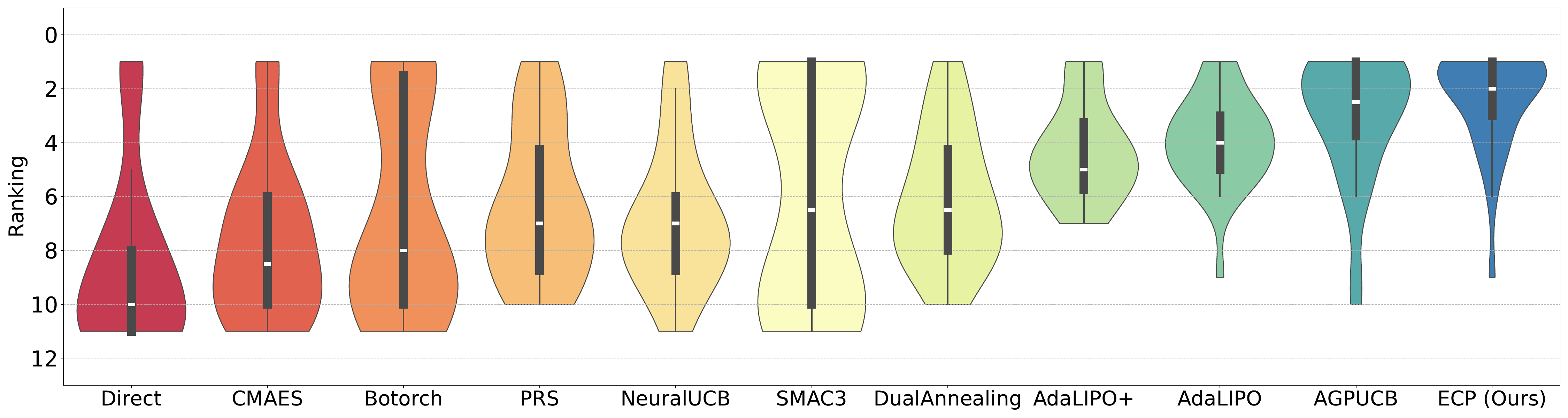} 
    \caption{\small Violin plots showing the ranking distributions of diverse optimization algorithms, ordered by increasing median ranking, across 30 non-convex, multi-dimensional synthetic and real-world problems. The results are based on a budget of $n=50$ for the different algorithms, with maxima averaged over 100 repetitions.}
    \label{fig:violin} 
\end{figure*}

For completeness, although the original works do not provide this information, we note that by applying the same bounding techniques used in \cref{prop:rejection_proba} and leveraging the law of total probability, one can derive the rejection bounds for both AdaLIPO and AdaLIPO+ \citep{malherbe2017global, serre2024lipo+}. We present these results in \cref{prop:rejection_proba_adalipo} in \cref{rejection_adalipo_app}. Unlike our approach, which does not require space-filling to estimate the true Lipschitz constant, they allocate a portion of uniform random sampling for this purpose, with probability $p$. Consequently, their rejection upper bound is reduced by this factor, let $R(x, t)$ be the event of rejecting $x$ at time $t+1$, then its probability is upperbounded as $\mathbb{P}\left(R(x, t+1)\right) \leq \frac{(1-p)t (\sqrt{\pi}\Delta)^d}{k_t^d \Gamma(d/2 + 1) \lambda(\mathcal{X})}$. In such case, it can be seen that the upperbound of ECP is the closest to LIPO, which assumes the knowledge of $k$.

\begin{remark}
\label{remark:knowing_k}
If the exact value of $k$ is known, then for any choice of $C\geq 1$, setting $\varepsilon_1 = k$ and $\tau_{n,d} = 1$ (in which case the choice of $C$ becomes irrelevant since $\varepsilon_t$ does not increase when multiplied by one), the ECP algorithm recovers the exact method used in LIPO \citep{malherbe2017global}, which assumes knowledge of $k$ and is known to achieve optimal rates for known Lipschitz constants. Furthermore, our algorithm achieves the same rejection rate of LIPO, as shown in Theorem 1 in \citep{serre2024lipo+}, for $\varepsilon_1 = k$ and $\tau_{n,d}=1$.
\end{remark}

Increasing the values of $\varepsilon_1$ and $\tau_{n,d}$ causes the rejection probability to approach zero, reducing the algorithm to a pure random search and undermining the efficiency of function evaluations. Therefore, smaller values for both $\varepsilon_1$ and $\tau_{n,d}$ are required, which leads to a higher rejection rate and creates a trade-off between time complexity and output quality. We mitigate this by introducing a constant $C > 1$, allowing the use of smaller $\varepsilon_1$ and $\tau_{n,d}$ while preserving the algorithm's speed by adaptively increasing $\varepsilon_t$ (when needed) in response to high rejection rates. $\varepsilon_t$ grows deterministically, ensuring that $\varepsilon_t \geq \varepsilon_1 \tau_{n,d}^{t-1}$, and stochastically when growth conditions are met. By multiplying $\varepsilon_t$ by $\tau_{n,d} > 1$ during rejection growth, even with small values of $\tau_{n,d}$ and $\varepsilon_1$, we guarantee the eventual acceptance of a point, proven in \cref{proof:cor:non_zero_acceptance}.

\begin{corollary}{\sc (Likely Acceptance)} \label{cor
}
\label{cor:non_zero_acceptance}
There exists a maximum finite number of increases, $v$, independent of the iteration $t$, function of $n$ and $d$, such that the probability of acceptance is at least $1/2$, 
$$
 v = \left\lceil\frac{1}{d} \log_{\tau_{n,d}}\left(\frac{2 n (\sqrt{\pi}\Delta)^{d} }{\varepsilon_1^d \Gamma(d / 2+1) \lambda(\mathcal{X})}\right)\right\rceil.
 $$
\end{corollary}

In the following we characterize this growth and show in \cref{proof:prop:rejection_growth} that it maintains a linear growth in $t$, with coefficient $C > 1$, which ensures a fast algorithm.

\begin{proposition}
{\sc (Rejection Growth)}
\label{prop:rejection_growth}
For any iteration step $t \geq 1$, the growth condition used in ECP, with a constant $C > 1$, ensures
$
h_{t+1} \leq (t+1) \cdot C.
$
Otherwise, $\varepsilon_t$ grows and $h_{t+1}$ is reset to zero.
\end{proposition}

While algorithms like AdaLIPO may fall into an infinite loop of rejections, which happens with growing $t$, as the acceptance of a point gets harder, the result above is crucial in showing that the growth condition in ECP, i.e., $h_{t+1}-h_{t}>C$, prevents super-linear rejection growth by continuously increasing $\varepsilon_t$ for the same iteration $t$. This increase makes point acceptance more likely in that iteration, as demonstrated in \cref{lem:montonic_with_epsilon}. The growth of $\varepsilon_t$ continues until a point is accepted, thereby avoiding the infinite sampling loop.

By controlling the growth in the number of rejections, we can establish an upper bound on the worst-case computational complexity of ECP for a fixed evaluation budget of $n$. This result is formalized in the following theorem, shown in \cref{proof:thm:complexity}.
\begin{theorem} 
\label{thm:complexity}
{\sc (ECP Computational Complexity)}
Consider the ECP algorithm tuned with any $\varepsilon_1 > 0$, any $\tau_{n,d} > 1$, and any constant $C > 1$. Then, for any function $f \in \mathcal{F}$ and any budget $n \in \mathbb{N}^{\star}$, with a probability at least $1-\frac{1}{2^C-1}$, the computational complexity of ECP is at most $\mathcal{O}\left(\frac{n^2 C}{d} \log_{\tau_{n,d}}\left(\frac{n \Delta^{d} }{\varepsilon_1^d \Gamma(d / 2+1) \lambda(\mathcal{X})}\right) \right)$. 
\end{theorem}

\begin{remark}
\label{remark_params}
Notice that a larger constant $C$ implies less constraint on growth, as shown in \cref{prop:rejection_growth}, which leads to more patience before increasing $\varepsilon_t$. Furthermore, smaller values of $\varepsilon_1 > 0$ and $\tau_{n,d} > 1$ lead to a slower growth of $\varepsilon_t$, resulting in higher rejection rates, as shown in \cref{prop:rejection_proba}. Therefore, either increasing $C$, decreasing $\tau_{n,d}$, or decreasing $\varepsilon_1$ result in higher rejection rates and consequently potentailly more waiting time to accept a sampled point, thereby potentially increasing the computational complexity of the algorithm, as validated in \cref{thm:complexity}.
\end{remark}

\subsection{Regret Analysis}
\label{sec:regret_analysis}

In the following we provide the regret guarantees of ECP, both for infinite and finite budgets. But first, we define the $i^\star$ as the hitting time, after which $\varepsilon_t$ reaches or overcomes the Lipschitz constant $k$.

\begin{definition}{\sc (Hitting Time)}
\label{i_star_definition}
For the sequence $\left(\varepsilon_i\right)_{i \in \mathbb{N}}$ and the unknown Lipschitz constant $k > 0$, we can define 
$
i^\star \triangleq \min \left\{i \in \mathbb{N}^\star: \varepsilon_i \geq k\right\}.
$
\end{definition}

In the following lemma, we upper-bound the time $t$ after which $\varepsilon_t$ is guaranteed to reach or exceed $k$. This shows that for sufficiently large $t$, the event of reaching $k$ is inevitable, with proof provided in \cref{proof:hitting_time_upperbound}.

\begin{lemma}{\sc (Hitting Time Upper-bound)}
\label{hitting_time_upperbound}
For any function $f \in \Lip(k)$, for any $\tau_{n,d}>1$ and any $\varepsilon_1 > 0$, the hitting time $i^\star$ is upperbounded as follows:
\begin{equation*}
\forall \varepsilon_1 > 0,  \quad i^\star \leq \max\left(\left\lceil\log_{\tau_{n,d}}\left(\frac{k}{\varepsilon_1}\right)\right\rceil,1\right).
\end{equation*}
\end{lemma}

Unlike other optimization methods that do not guarantee no-regret, such as AdaLIPO+ (due to decaying exploration) or evolutionary algorithms like CMA-ES, similar to AdaLIPO, ECP is a no-regret algorithm over Lipschitz functions, as shown in \cref{app:prop:proba_convergence}.

\begin{theorem}
\label{thm:proba_convergence}
{\sc (No-regret)} 
For any unkown Lipschitz constant $k$,
the ECP algorithm tuned with any hyper-parameter $\varepsilon_1 > 0$ and using any geometric growth hyper-parameter $\tau_{n,d} > 1$, and $C > 1$ converges in probability to the exact maximum, {\it i.e.}
   \begin{equation*}
       \forall f \in \Lip(k),~\mathcal{R}_{\text{ECP},f}(n)
    \xrightarrow{p} 0.   
   \end{equation*}
\end{theorem}

If the function is Lipschitz, even with an unknown Lipschitz constant, similar to previous global optimization methods \citep{malherbe2016ranking, malherbe2017global}, we demonstrate in \cref{app:faster} that ECP consistently outperforms or matches PRS.

\begin{proposition}
\label{prop:fasterprs}
{\sc (ECP Faster than PRS)} 
Consider the ECP algorithm tuned with any initial value $\varepsilon_1>0$, any constant $ \tau_{n,d}>1 $, and any constant $C > 1$.
Then, for any $f \in \normalfont{\text{Lip}}(k)$ and
$n \geq i^\star$,
\[
\P\left(\max_{i=i^\star, \cdots,  n } f(x_i) \geq y \right) \geq 
\P\left(\max_{i=i^\star, \cdots,  n } f(x'_i) \geq y \right)
\]
for all $y \geq \max_{i=1, \cdots, i^{\star} -1} f(x_i)$, where $x_{1}, \cdots,  x_n$ are $n$ evaluated points by ECP
and $x_{i^\star}', \cdots,  x'_n$ are $n$ 
independent uniformly distributed points over $\X$.
\end{proposition}

Finally, the following theorem, shown in \cref{proof:thm:ecpupperbound}, upperbounds the finite-time ECP regret.

\begin{theorem}
\label{thm:ecpupperbound}
{\sc (Regret Upper Bound)} Consider ECP
tuned with any $\varepsilon_1>0$, any $ \tau_{n,d}>1 $, and any $C > 1$.
Then, for any non-constant $f \in \Lip(k)$, with some unknown Lipschitz constant $k$,
any $n\in\mathbb{N}^{\star}$ and
$\delta \in (0,1)$, we have 
with probability at least $1-\delta$,
\begin{align*}
\mathcal{R}_{\text{ECP}, f}(n)
&\leq 
\diam{\X} \cdot i^{\star \frac{1}{d}} \cdot k \cdot
\left(  \frac{\ln(\frac{1}{\delta})}{ n } \right)^{\frac{1}{d}}.
\end{align*}
\end{theorem}

With large values of $\tau_{n,d}$ and $\varepsilon_1$, $i^\star$ (bounded in \cref{hitting_time_upperbound}) tends to one, and the ECP regret bound exactly recovers the finite-time upper bound of LIPO \cite{malherbe2017global}, which, unlike our method, requires knowledge of $k$, see \cref{table_rejection_proba}. From the above theorem, it can be seen that ECP achieves the minimax optimal rate of $\mathcal{O}(k n^{-\frac{1}{d}})$, matching the lower bound $\Omega(k n^{-1/d})$ provided by \cref{prop:minimax}.

%%%%%%%%%%%%%%%%%%%%%%%%%%%%%%%%%%%%%%%%%%%%%%%%%%%%%%%%%%%%%%%%
%%%%%%%%%%%%%%%%%%          50            %%%%%%%%%%%%%%%%%%%%%%
%%%%%%%%%%%%%%%%%%%%%%%%%%%%%%%%%%%%%%%%%%%%%%%%%%%%%%%%%%%%%%%%

\begin{table*}[ht]
\centering
\resizebox{\textwidth}{!}{
\begin{tabular}{cccccccccccc}
\hline
\textbf{Problems}& \textbf{PRS}& \textbf{DIRECT}  & \textbf{CMA-ES}&\textbf{DualAnnealing}&\textbf{NeuralUCB} & \textbf{AdaLIPO}  & \textbf{AdaLIPO+} & \textbf{Botorch} & \textbf{SMAC3} & \textbf{A-GP-UCB} & \textbf{ECP (Ours)} \\ \hline
\hline
autoMPG	2D & -30.31 (5.69) &	-27.85 (0.00) &	-25.77 (8.14) &	-28.36 (3.81) &	-30.75 (6.44) &	-25.58 (1.67) &	-26.34 (2.02) &	\textbf{-23.10 (0.00)} & -23.32 (0.00) & \textbf{-23.10 (0.00)} & -25.17 (1.50) \\
breastCancer 2D&	-0.08 (0.01) &  -0.08 (0.00) &	\textbf{-0.07 (0.01)} &	-0.08 (0.01) &	-0.08 (0.00) &	\textbf{-0.07 (0.00)} &	\textbf{-0.07 (0.00)} &	\textbf{-0.07 (0.00)} &	\textbf{-0.07 (0.00)}& \textbf{-0.07 (0.00)}&	\textbf{-0.07 (0.00)} \\
concrete 2D&	-27.85 (2.84) &	-28.31 (0.00) &	-26.05 (6.30) &	-27.07 (1.72) &	-27.75 (2.84) &	-25.67 (0.91) &	-26.07 (1.13) &	-24.50 (1.25) & \textbf{-24.42 (0.00)}&	-24.45 (0.14)&	-25.33 (0.62) \\
housing 2D& -13.60 (0.62) &	-13.79 (0.00) &	-13.14 (0.91) &	-13.40 (0.49) &	-13.68 (0.68) &	-12.99 (0.16) &	-13.09 (0.22) &		\textbf{-12.73 (0.00)}	&-12.75 (0.00)	& \textbf{-12.73 (0.00)}&	-12.98 (0.13) \\
yacht 2D&	-67.54 (6.85) &	-64.46 (0.00) &	-61.48 (7.21) &	-65.64 (5.69) &	-67.75 (7.73) &	-59.78 (1.04) &	-60.75 (1.58) &	\textbf{-58.07 (0.00)}	&-58.40 (0.00) &\textbf{-58.07 (0.00)} &	-60.08 (1.50) \\
\hline
\hline
Ackley 2D& -4.92 (1.48)  & -4.92 (0.00) & -8.69 (5.47) & -4.72 (1.72) & -5.26 (1.62) & -2.08 (1.19) &  -2.37 (1.42) & -6.39 (2.11)	&-4.89 (0.00)	&-1.39 (2.35) & \textbf{-1.38 (0.80)} \\ 
Bukin   2D& -21.09 (10.09)  & -47.28 (0.00) & -12.51 (13.60) & -19.43 (10.22)& -18.92 (9.11) & -14.54 (6.95) &  -15.22 (6.97) & -32.41 (16.33)&	\textbf{-0.91 (0.00)}	&-8.40 (8.81)&    -11.33 (5.50)  \\ 
Camel  2D& 0.89 (0.13)  &  0.99 (0.00) & 0.90 (0.26) & 0.90 (0.14) & 0.97 (0.06) & 1.01 (0.02) &  0.99 (0.05)&  0.72 (0.32)&	\textbf{1.02 (0.00)}	&1.01 (0.10)& \textbf{1.02 (0.01)}  \\ 
Cross-in-tray 2D & 1.99 (0.07) & 1.96 (0.00) & 1.76 (0.18) & 2.00 (0.07) & 1.95 (0.10) & 2.01 (0.07) &  2.02 (0.07) & 	1.94 (0.10)&	1.88 (0.00)&	2.00 (0.09) & \textbf{2.03 (0.06)} \\ 
Damavandi 2D&  -3.57 (1.56)  & -9.26 (0.00) & -4.37 (5.71) & -3.18 (1.26) & -4.17 (2.17) & -2.55 (0.57) & -2.59 (0.56) & -5.53 (3.27)	 &	\textbf{-2.02 (0.00)} & -2.83 (3.42) &  -2.24 (0.29) \\ 
Drop-wave 2D& 0.73 (0.13)  & 0.14 (0.00) & 0.58 (0.24) & 0.75 (0.13) &  0.70 (0.20) &0.74 (0.12) &   \textbf{0.76 (0.12)}& 0.66 (0.16) & 0.20 (0.00) & 0.70 (0.19)&   \textbf{0.76 (0.12)}\\ 
Easom  2D& 0.06 (0.18)  & 0.00 (0.00) &  0.04 (0.18) &  0.05 (0.12) & 0.07 (0.21) & 0.05 (0.15) & 0.06 (0.16) & 0.02 (0.09)	 &	0.00 (0.00)	& \textbf{0.13 (0.33)} &  0.06 (0.15) \\ 
Egg-holder 2D& 61.11 (11.57)  & 54.74 (0.00) &  21.38 (22.00) & 64.77 (11.85) &  64.26 (11.90) & 59.32 (10.44) &  62.99 (12.71) & 60.24 (10.02)	& \textbf{84.71 (0.00)} & 62.49 (13.57) &  69.91 (11.70) \\ 
Griewank  2D&  -0.26 (0.13)   & -1.06 (0.00) & -0.43 (0.28) & -0.27 (0.14) & -0.26 (0.13) & -0.27 (0.12) &  -0.27 (0.12) & -0.40 (0.21)	 & -0.43 (0.00) & \textbf{-0.13 (0.11)} &  -0.25 (0.13) \\ 
Himmelblau  2D& -2.96 (3.12)  & -3.94 (0.00) & -2.32 (5.95) & -2.96 (2.76) & -3.17 (2.85) & -1.25 (1.40) & -1.56 (1.58) & -7.10 (7.16)	 & \textbf{-0.08 (0.00)} & -0.96 (4.03) & -0.74 (0.82) \\ 
Holder  2D& 14.44 (3.42)  & 8.83 (0.00) &  6.61 (4.98) & 14.69 (3.45) & 14.33 (3.48) & 15.53 (3.19) & 15.22 (3.05) & 15.97 (1.34)	 & 0.78 (0.00)	& 16.08 (4.45) &  \textbf{17.03 (2.17)} \\ 
Langermann  2D& 2.92 (0.76)  & \textbf{3.98 (0.00)} & 1.61 (1.13) & 2.73 (0.70) & 2.63 (1.00)& 2.71 (0.82) & 2.69 (0.83) & 2.30 (0.94)	 & 2.63 (0.00) & 2.80 (1.01) & 2.32 (1.10) \\
Levy  2D& -3.87 (3.56)  & -6.56 (0.00) & -35.73 (45.40) &-2.60 (2.33) & -4.48 (4.18) & -1.63 (1.04) &  -2.27 (1.83) & -7.5 (7.34)	 & -4.18 (0.00) & -1.67 (4.82) &  \textbf{-0.80 (0.49)} \\ 
Michalewicz 2D& 1.11 (0.28)  &  1.36 (0.00) &  1.20 (0.37) & 1.08 (0.27) & 1.06 (0.27) & 1.28 (0.30) & 1.21 (0.28) & 0.98 (0.28)	 & 1.00 (0.00)	& 1.35 (0.41) & \textbf{1.38 (0.29)} \\  
Rastrigin  2D& -6.86 (3.52)   & -9.63 (0.00) & -12.69 (9.86) & -5.98 (3.58) & -6.73 (3.62) &  -6.75 (3.39) & -6.66 (3.45) & -9.9 (5.65)	 &	-38.22 (0.00) &	\textbf{-5.52 (3.51)} &  \textbf{-5.52 (2.93)} \\ 
Schaffer   2D& \textbf{-0.01 (0.01)}  & \textbf{-0.01 (0.00)} & \textbf{-0.01 (0.01)} & \textbf{-0.01 (0.01)} & \textbf{-0.01 (0.01)} & \textbf{-0.01 (0.01)} & \textbf{-0.01 (0.01)} & \textbf{-0.01 (0.02)} &	-0.94 (0.00)	& \textbf{-0.01 (0.01)} &   \textbf{-0.01 (0.01)} \\ 
Schubert 2D& 8.28 (4.51)  & 1.18 (0.00) & 4.46 (4.10) & 8.36 (4.87) &  7.60 (4.15) &  7.92 (4.44) &   7.90 (4.82) & 5.26 (3.55)	 &	3.79 (0.00) &	\textbf{9.80 (6.40)} & 7.80 (4.46) \\ 
% \\
\hline
\hline
% \\
Colville 3D  & -0.26 (0.26)  & \textbf{-0.06 (0.00)} &  -1.63 (3.22) & -0.33 (0.42) &-0.33 (0.32) & -0.20 (0.16) & -0.26 (0.32) & -0.76 (1.02)	 &	-0.69 (0.00) &	-0.15 (0.16)& -0.17 (0.14)  \\
Hartmann 3D   & 3.42 (0.31)  & 3.51 (0.00) & 3.57 (0.36) & 3.35 (0.32) & 3.64 (0.17) & 3.73 (0.11) & 3.68 (0.17) & \textbf{3.86 (0.00)} & 2.89 (0.00) &	3.73 (0.27) & 3.79 (0.04) \\ 
Hartmann 6D   & 1.77 (0.56)  & 0.42 (0.00) & 1.99 (0.58) &  1.69 (0.53) & 1.77 (0.52) &  1.94 (0.52) & 1.93 (0.49) & 	\textbf{3.21 (0.26)}	& 2.54 (0.00) &	2.75 (0.60) & 2.01 (0.43) \\ 
Rosenbrock 3D & -0.48 (0.26)  & -0.54 (0.00) & -0.40 (0.34) & -0.42 (0.26) & -0.32 (0.16) & -0.36 (0.18) & -0.37 (0.18) & -0.62 (0.32) &	\textbf{-0.10 (0.00)}&	-0.24 (0.25) &  -0.16 (0.08)\\ 
Perm 10D  & -0.13 (0.14)  & nan (nan)  & -1.28 (3.10)  & -0.15 (0.14)  & -0.13 (0.12)  & -0.11 (0.12)  & -0.11 (0.10)  & 	-0.26 (0.23)&	\textbf{-0.01 (0.00)} &	-0.13 (0.19) & -0.08 (0.07) \\ 
Perm 20D  & -2.52 (2.21)  & nan (nan) & -44.41 (124.04)  & -2.99 (3.11)  & -3.02 (4.65)  & -2.24 (2.24)  & -2.46 (2.62)  & -7.21 (7.06)	& -3.81 (0.00)	&-2.91 (4.25) & \textbf{-1.59 (1.53)} \\
Powell 100D&	3.20 (0.29)	&nan (nan)	&2.76 (0.62)&	3.24 (0.30)&	3.27 (0.33)	&3.36 (0.34)&	3.33 (0.33)&	3.30 (0.00)&	3.42 (0.00)	&3.18 (0.29)	& \textbf{3.64 (0.34)}\\
Powell 1000D&	\textbf{0.23 (0.01)}	&nan (nan)	&\textbf{0.23 (0.01)}	&\textbf{0.23 (0.01)}&	\textbf{0.23 (0.01)}&	\textbf{0.23 (0.01)}	&\textbf{0.23 (0.01)}	&\textbf{0.23 (0.01)}	&0.22 (0.00)	&\textbf{0.23 (0.01)}	&\textbf{0.23 (0.01)}\\
% \\
\hline
\hline
% \textbf{Average Rank} &   &  &  &  &   &   &  &  &  \\ 
% \hline
% #top-1	2	3	3	2	2	3	4	8	9	10	13
\textbf{\# Top-1} &  2 & 3 & 3 & 2  & 2  & 3 & 4 & 8 & 9 & 10 &\textbf{13} \\  \hline
\end{tabular}
}
\caption{\small Comparison of the maximum returned value using a budget of $n=50$ for the different algorithms, averaged over 100 repetitions with standard deviations reported. The ECP hyperparameters are fixed across all problems: $\varepsilon_1=10^{-2}$, $\tau_{n,d} = \max\left(1+\frac{1}{n d}, \tau\right)$, with $\tau=1.001$, and $C = 10^3$. For AdaLIPO, we set $p=0.1$, as specified in their paper. nan indicates an error related to insufficient budget to return a value.}
\vspace{-0.2cm}
\label{tab:example_50}
\end{table*}

\section{Numerical Analysis}
\label{sec:experiments}

Since we do not assume knowledge of the Lipschitz constant, we consider benchmarks that do not require that. We compare our approach against 10 benchmarks, including AdaLIPO \citep{malherbe2017global}, AdaLIPO+ \citep{serre2024lipo+}, PRS \citep{zabinsky2013stochastic}, DIRECT \citep{jones1993lipschitzian}, DualAnnealing \citep{xiang1997generalized}, CMA-ES \citep{hansen1996adapting, hansen2006cma, hansen2019pycma}, Botorch \citep{balandat2020botorch}, SMAC3 \citep{JMLR:v23:21-0888}, A-GP-UCB \citep{JMLR:v20:18-213}, and NeuralUCB \citep{zhou2020neural}. We fix the same budget $n$ for all the methods and report the maximum achieved value over the rounds, averaged over 100 repetitions, with reported standard deviations.

We evaluate the proposed method on various global optimization problems using both synthetic and real-world datasets. The implementation of the considered objectives is publicly available at \href{https://github.com/fouratifares/ECP}{\texttt{https://github.com/fouratifares/ECP}}. 

The synthetic functions were designed to challenge global optimization methods due to their highly non-convex curvatures \citep{molga2005test, simulationlib}. Some of the 2D functions are shown in Fig. \ref{fig:all_figures} for reference. For the real-world datasets, we follow the same set of global optimization problems considered in \citep{malherbe2016ranking, malherbe2017global} drawn from \citep{frank2010uci}, including Auto-MPG, Breast Cancer Wisconsin, Concrete Slump Test, Housing, and Yacht Hydrodynamics, optimizing the hyper-parameters of a Gaussian kernel ridge regression by minimizing the empirical mean squared error of the predictions. 

Across all optimization problems, we fix $\varepsilon_1 = 10^{-2}$ and use $\tau_{n,d} = \max\left(1 + \frac{1}{nd}, \tau\right)$, with $\tau = 1.001$ and $C = 10^3$ (discussion and ablation of these parameters can be found in \cref{app:hyperparams_ecp}). Results for $n=50$ across different algorithms and problems are presented in \cref{tab:example_50}. 
The first block in this table consists of real-world problems, while the second and third blocks correspond to $d$-dimensional synthetic non-convex objectives, for $d=2$ and $d>2$, respectively.

ECP demonstrates outstanding performance across benchmarks, achieving the highest values (Top-1) in 13 problems, while other approaches achieved at most 10 out of the 30 problems. We further illustrate the ranking distribution of the 11 algorithms in the violin plot in \cref{fig:violin} across all the objective functions. Clearly, ECP demonstrates the highest median ranking across the objectives and shows the highest tendency towards top rankings. Notably, it recovers or outperforms all other algorithms on the highest-dimensional problems (Perm 20D, Powell 100D, and Powell 1000D). Additionally, as shown in \cref{tab:example_25} for half the budget and \cref{tab:example_100} for double the budget in \cref{app:extended_results}, ECP significantly outperforms all the considered Lipschitz, bandit, and evolutionary approaches.

While the main focus of this work is on expensive functions and small budgets, it is evident from \cref{tab:example_100} and \cref{tab:optimization_results} in \cref{app:extended_results} that ECP continues to perform well even for larger budgets. However, while larger budgets do not disadvantage ECP, its performance advantage diminishes as other global optimization methods, such as CMA-ES, are sometimes able to find better approximations in these larger-budget settings. Additionally, \cref{tab:optimization_results} highlights a key difference between ECP and similar methods for Lipschitz functions, such as AdaLIPO and AdaLIPO+. These methods can encounter infinite loops as their complexity increases with larger budgets. In contrast, ECP avoids such problems by adaptively increasing the acceptance region through the gradual expansion of $\varepsilon_t$, which validates our earlier observations following \cref{prop:rejection_growth}.

\section{Discussion}

Several global optimization approaches rely on surrogate models. While surrogate-based methods excel at modeling complex function landscapes, they introduce significant computational overhead and risk overfitting in low-budget scenarios. With a limited budget, learning an accurate surrogate model becomes challenging, leading to poor sampling and suboptimal performance. In contrast, ECP is specifically designed for low-budget settings, which are common when dealing with expensive functions. Although large budgets do not degrade ECP’s performance, when computational resources permit a high number of function evaluations, methods that allocate some of these evaluations to learning the function’s smoothness or constructing a surrogate model may perform better.

ECP provides theoretical guarantees under the sole assumption of global Lipschitz continuity, without requiring prior knowledge of the Lipschitz constant. While this assumption is minimal compared to others, it may not hold in all real-world scenarios. Nevertheless, empirically, ECP outperforms several established algorithms across a range of non-convex optimization problems, including cases where the Lipschitz continuity assumption is not strictly satisfied.

\section{Conclusion}

We introduce ECP, a global optimization algorithm for black-box functions with unknown Lipschitz constants. Our theoretical analysis shows that ECP is no-regret as evaluations increase and meets minimax optimal regret bounds within a finite budget. Empirical results across diverse non-convex, multi-dimensional optimization tasks demonstrate that ECP outperforms state-of-the-art methods.

\bibliographystyle{plainnat}
\bibliography{biblio}

\newpage
\onecolumn

\appendix

\newpage
\section{Extended Empirical Results for Various Budgets}
\label{app:extended_results}
%%%%%%%%%%%%%%%%%%%%%%%%%%%%%%%%%%%%%%%%%%%%%%%%%%%%%%%%%%%%%%%%
%%%%%%%%%%%%%%%%%%          25            %%%%%%%%%%%%%%%%%%%%%%
%%%%%%%%%%%%%%%%%%%%%%%%%%%%%%%%%%%%%%%%%%%%%%%%%%%%%%%%%%%%%%%%

\begin{table*}[ht]
% \tiny
\centering
% \resizebox{0.7\textheight}{!}{
\resizebox{0.99\textwidth}{!}{
\begin{tabular}{ccccccccc}
\hline
\textbf{Objectives}& \textbf{PRS}& \textbf{DIRECT}  & \textbf{CMA-ES}&\textbf{DualAnnealing}&\textbf{NeuralUCB} & \textbf{AdaLIPO}  & \textbf{AdaLIPO+} & \textbf{ECP (Ours)} \\ \hline
\\
autoMPG & -33.17 (7.34) & -45.86 (0.00) & -31.41 (16.40) & -33.87 (7.47) & -33.15 (6.51) & -27.20 (2.74) & -29.24 (4.31) & \textbf{-26.69 (2.58)} \\
breastCancer & -0.09 (0.01) & -0.12 (0.00) & -0.08 (0.01) & -0.09 (0.02) & -0.09 (0.03) & \textbf{-0.07 (0.00)} & -0.08 (0.00) &  \textbf{-0.07 (0.00)} \\
concrete & -29.75 (3.74) & -41.50 (0.00) & -28.73 (15.73) & -30.03 (4.60) & -29.48 (3.43) & -26.61 (1.83) & -27.42 (1.99) &  \textbf{-26.33 (1.45)} \\
housing & -14.03 (0.82) & -16.81 (0.00) & -13.53 (1.73) & -14.17 (1.27) & -14.00 (0.77) & -13.23 (0.34) & -13.36 (0.44) &  \textbf{-13.20 (0.31)} \\
yacht & -70.97 (8.81) & -84.59 (0.00) & -67.18 (10.02) & -72.33 (9.86) & -71.22 (8.12) & \textbf{-62.24 (3.06)} & -63.93 (4.24) &  -62.35 (2.75) \\
\\
\hline
\hline
\\
ackley	& -6.25 (1.93) &	-4.92 (0.00) &	-10.97 (4.33) &	-6.29 (2.36) &	-6.27 (2.03) &	-3.26 (1.40) &	-3.82 (1.86) &	\textbf{-2.69 (1.15)}  \\
bukin &	-28.82 (13.32) &	-83.45 (0.00) &	-24.93 (23.96) &	-31.37 (14.29) &	-31.50 (15.25) &	-24.07 (11.03) &	-25.59 (11.56) &	\textbf{-23.01 (10.19)} \\
camel &	0.73 (0.27) &	-2.25 (0.00) &	0.71 (0.34) &	0.73 (0.27) &	0.88 (0.17) &	0.95 (0.10) &	0.92 (0.14) &	\textbf{0.99 (0.07)} \\
crossintray &	1.94 (0.09) &	1.96 (0.00) &	1.74 (0.15) &	1.93 (0.09) &	1.91 (0.10) &	\textbf{1.97 (0.09)} &	1.94 (0.09) &		\textbf{1.97 (0.10)} \\
damavandi &	-5.35 (3.00) &	-57.26 (0.00) &	-11.83 (13.30) &	-5.36 (3.63) &	-5.72 (3.83) &	-3.21 (1.39) &	-3.96 (1.99) &	\textbf{-2.57 (0.54)}  \\
dropwave	& 0.65 (0.15) &	0.14 (0.00) &	0.48 (0.24) &	0.62 (0.16) &	0.54 (0.20) &	0.64 (0.17) &	\textbf{0.67 (0.16)} &		\textbf{0.67 (0.16)} \\
easom &	\textbf{0.05 (0.16)} &	0.00 (0.00) &	0.00 (0.00) &	0.02 (0.06) &	0.03 (0.13) &	0.01 (0.08) &	0.02 (0.09) &		0.02 (0.09)  \\
eggholder &	55.37 (12.86) &	1.17 (0.00) &	19.78 (25.68) &	55.43 (12.85) &	55.04 (14.20) &	55.28 (14.46) &	53.89 (11.79) &		\textbf{58.52 (13.37)}  \\
griewank &	-0.37 (0.16) &	-1.56 (0.00) &	-0.53 (0.31) &	-0.39 (0.19) &	-0.38 (0.18) &	-0.36 (0.16) &	-0.37 (0.16) &		\textbf{-0.35 (0.19)} \\
himmelblau &	-5.20 (5.39) &	-5.83 (0.00) &	-9.34 (17.93) &	-6.84 (6.91) &	-5.76 (4.83) &	-2.83 (2.86) &	-3.82 (4.04) &	\textbf{-2.73 (2.30)}  \\
holder  &	12.20 (3.85) &	2.57 (0.00) &	5.97 (4.56) &	11.99 (3.98) &	12.00 (4.19) &	13.26 (3.73) &	12.63 (4.27) &	\textbf{15.18 (3.10)} \\
langermann &	\textbf{2.32 (0.97)} &	0.04 (0.00) &	1.04 (1.01) &	2.22 (0.87) &	2.15 (0.97) &	2.10 (0.98) &	2.26 (0.94) &		1.71 (1.08) \\	
levy &	-7.29 (6.85) &	-22.20 (0.00) &	-52.70 (45.67) &	-6.18 (5.91) &	-8.38 (7.82) &	-3.85 (4.31) &	-5.56 (5.78) &	\textbf{-1.78 (1.44)}  \\
michalewicz	& 0.97 (0.26) &	\textbf{1.36 (0.00)} &	0.97 (0.40) &	0.92 (0.27) &	0.89 (0.25) &	1.07 (0.30) &	0.97 (0.33) &		1.01 (0.33)\\
rastrigin &	-9.79 (5.41) &	-60.41 (0.00) &	-16.16 (8.96) &	-9.73 (5.19) &	-10.63 (5.08) &	-9.28 (4.72) &	-10.06 (4.89) &	\textbf{-7.02 (3.97)} \\			
schaffer &	-0.02 (0.03) &	\textbf{-0.01 (0.00)} &	\textbf{-0.01 (0.01)} &	\textbf{-0.01 (0.01)} &	\textbf{-0.01 (0.01)} &	\textbf{-0.01 (0.01)} &	\textbf{-0.01 (0.01)} &	\textbf{-0.01 (0.01)} \\
schubert& 	6.26 (4.50) &	0.55 (0.00) &	4.02 (4.13) &	5.55 (4.28) &	5.48 (4.12) &	5.65 (3.92) &	5.42 (4.14) &		\textbf{7.27 (4.64)} \\
 \\	
\hline
\hline
\\
colville 4D  &	\textbf{-0.52 (0.51)} &	nan (nan)  &	-9.65 (12.30) &	-0.65 (0.66) &	-0.64 (0.81) &	-0.59 (0.56) &	-0.70 (0.82) &		-0.68 (0.80) \\					
hartmann 3D  &	3.08 (0.47) &	0.99 (0.00) &	3.27 (0.60) &	3.10 (0.45) &	3.43 (0.35) &	3.49 (0.30) &	3.44 (0.33) &	\textbf{3.63 (0.21)} \\
hartmann 6D  &	1.48 (0.61) &	nan (nan)  &	\textbf{1.68 (0.56)} &	1.29 (0.59) &	1.44 (0.63) &	1.61 (0.54) &	1.52 (0.56) &		1.51 (0.49) \\
rosenbrock 3D &	-0.64 (0.32) &	nan (nan)  &	-0.83 (0.62) &	-0.63 (0.36) &	-0.47 (0.29) &	-0.50 (0.27) &	-0.61 (0.29) &	\textbf{-0.29 (0.23)} \\
Perm 10D  &	-0.23 (0.22) &	nan (nan)  &	-1.25 (2.90) &	-0.29 (0.25) &	-0.22 (0.23) &	-0.18 (0.16) &	-0.24 (0.22) &	\textbf{-0.15 (0.16)} \\
perm 20D &-5.38 (5.90) &nan (nan) &-73.29 (153.01) &-5.52 (7.49) &-4.90 (6.43) &-5.52 (7.97) &-4.83 (7.06) & \textbf{-3.76 (4.14)}\\
\\
\hline
\textbf{\# Top-1} & 3  & 2 & 1 & 1 & 2  &  4 & 2  & 22 \\  \hline
\end{tabular}
}
\caption{\small Comparison of the maximum returned value using a budget of $n=25$ for the different algorithms, averaged over 100 repetitions with standard deviations reported. The ECP hyperparameters are fixed across all problems: $\varepsilon_1=10^{-2}$, $\tau_{n,d} = \max\left(1+\frac{1}{n d}, \tau\right)$, with $\tau=1.001$, and $C = 10^3$. For AdaLIPO, we set $p=0.1$, as specified in their paper. nan indicates an error related insufficient budget to return a value.}
\label{tab:example_25}
\end{table*}

%%%%%%%%%%%%%%%%%%%%%%%%%%%%%%%%%%%%%%%%%%%%%%%%%%%%%%%%%%%%%%%%
%%%%%%%%%%%%%%%%%%          100            %%%%%%%%%%%%%%%%%%%%%%
%%%%%%%%%%%%%%%%%%%%%%%%%%%%%%%%%%%%%%%%%%%%%%%%%%%%%%%%%%%%%%%%

\begin{table*}[ht]
% \tiny
\centering
% \resizebox{0.7\textheight}{!}{
\resizebox{0.99\textwidth}{!}{
\begin{tabular}{ccccccccc}
\hline
\textbf{Objectives}& \textbf{PRS}& \textbf{DIRECT}  & \textbf{CMA-ES}&\textbf{DualAnnealing}&\textbf{NeuralUCB} & \textbf{AdaLIPO}  & \textbf{AdaLIPO+} &  \textbf{ECP (Ours)} \\ \hline
\\									
autoMPG &	-27.21 (2.86) &	-24.46 (0.00) &	\textbf{-23.15 (0.19)} &	-25.76 (1.88) &	-27.67 (3.01) &	-24.42 (0.73) &	-24.61 (0.89) &		-24.11 (0.62) \\	
breastCancer &	-0.08 (0.00) &	\textbf{-0.07 (0.00)} &	\textbf{-0.07 (0.00)} &	-0.08 (0.00) &	\textbf{-0.07 (0.00)} &	\textbf{-0.07 (0.00)} &	\textbf{-0.07 (0.00)} &		\textbf{-0.07 (0.00)} \\
concrete &	-26.24 (1.21) &	-25.56 (0.00) &	\textbf{-24.46 (0.26)} &	-25.90 (1.05) &	-26.79 (1.57) &	-25.22 (0.59) &	-25.37 (0.59) &		-24.92 (0.37) \\			
housing &	-13.22 (0.29) & -13.05 (0.00) 	&	\textbf{-12.76 (0.08)} &	-13.08 (0.25) &	-13.35 (0.36) &	-12.85 (0.07) &	-12.89 (0.08) &		-12.84 (0.07) \\
yacht &	-63.71 (3.77) &	-59.95 (0.00) &	\textbf{-58.11 (0.12)} &	-61.81 (2.48) &	-64.87 (4.30) &	-58.63 (0.39) &	-59.17 (0.65) &		-58.67 (0.39) \\
\\
\hline
\hline
\\						
ackley &	-4.23 (1.20) &	-4.90 (0.00) &	-6.21 (6.43) &	-3.27 (1.17) &	-4.09 (1.12)	& -1.05 (0.71) &	-1.17 (0.90) &		\textbf{-0.71 (0.43)} \\																
bukin &	-16.09 (7.33) &	-27.38 (0.00) &	\textbf{-4.79 (3.89)} &	-13.55 (7.94) &	-15.41 (7.58) &	-8.77 (4.21) &	-9.68 (4.94) &	-8.74 (3.99) \\
camel &	0.96 (0.08) &	0.99 (0.00) &	0.84 (0.51) &	0.99 (0.04) &	1.00 (0.03) &	1.02 (0.01) &	1.02 (0.01) &		\textbf{1.03 (0.00)} \\
crossintray &	2.03 (0.05) &	1.96 (0.00) &	1.77 (0.17) &	2.04 (0.05) &	2.00 (0.09)	& 2.06 (0.05) &	2.04 (0.05) &	\textbf{2.08 (0.05)} \\
damavandi &	-2.90 (0.95) &	-9.26 (0.00) &	\textbf{-2.05 (0.39)} &	-2.36 (0.39) &	-2.85 (0.88)	& -2.16 (0.16) &	-2.25 (0.31) &		-2.09 (0.09) \\
dropwave &	0.80 (0.11) &	0.14 (0.00) &	0.65 (0.27) &	0.83 (0.10) &	0.78 (0.16) &	0.81 (0.11) &	0.82 (0.10) &	\textbf{0.83 (0.10)} \\
easom &	\textbf{0.10 (0.21)} &	0.00 (0.00) &	0.09 (0.26) &	0.09 (0.16) &	0.07 (0.16) &	0.07 (0.18) &	\textbf{0.10 (0.22)} &		\textbf{0.10 (0.21)} \\
eggholder &	67.89 (11.40) &	54.74 (0.00) &	28.18 (23.50) &	71.59 (10.36) &	69.46 (11.04)	 & 70.58 (10.02) &	67.77 (10.59) &		\textbf{74.63 (11.37)} \\
griewank &	-0.20 (0.09) &	-1.06 (0.00) &	-0.43 (0.27) &	-0.18 (0.09) &	-0.20 (0.09)	& -0.19 (0.09) &	-0.19 (0.10) &		\textbf{-0.17 (0.08)} \\
himmelblau &	-1.44 (1.53) &	-3.94 (0.00) &	-0.25 (0.51) &	-1.30 (1.32) &	-1.55 (1.25) &	-0.55 (0.60) &	-0.50 (0.50) &	\textbf{-0.20 (0.22)} \\
holder&	16.21 (2.80) &	8.83 (0.00) &	7.41 (5.49) &	18.24 (0.97) &	15.48 (3.12) &	17.32 (2.22) &	17.67 (1.83) &	\textbf{18.74 (0.52)} \\
langermann	& 3.29 (0.60) &	\textbf{3.98 (0.00)} &	1.84 (1.17) &	3.26 (0.62) &	2.83 (0.88) &	3.00 (0.75) &	3.23 (0.65) &		3.00 (0.93) \\
levy	& -2.44 (2.06) &	-0.80 (0.00) &	-25.48 (41.98) &	-1.27 (1.16) &	-2.27 (1.78) &	-0.98 (0.71) &	-1.49 (1.07) &	\textbf{-0.47 (0.37)} \\
michalewicz &	1.26 (0.27) &	1.36 (0.00) &	1.33 (0.45) &	1.29 (0.27) &	1.24 (0.29) &	1.55 (0.24) &	1.47 (0.27) &	\textbf{1.73 (0.10)} \\
rastrigin	& -5.28 (2.57) &	-9.63 (0.00) &	-11.35 (10.04) &	\textbf{-3.12 (1.67)} &	-5.82 (2.65) &	-4.96 (2.56) &	-4.82 (2.58) &	-4.17 (2.10) \\	
schaffer &	\textbf{-0.00 (0.00)} &	-0.01 (0.00) &	-0.01 (0.01) &	\textbf{-0.00 (0.00)} &	\textbf{-0.00 (0.00)} &	\textbf{-0.00 (0.00)} &	\textbf{-0.00 (0.00)} &		\textbf{-0.00 (0.00)} \\	
schubert &	11.27 (4.58) &	1.18 (0.00) &	6.59 (5.75) &	\textbf{11.96 (4.46)} &	10.43 (4.43) &	9.63 (4.42) &	10.39 (4.39) &		10.46 (4.84) \\

\\
\hline
\hline
\\
colville 4D & -0.15 (0.11) & \textbf{-0.06 (0.00)} & -0.12 (0.20) & -0.13 (0.12) & -0.14 (0.11) & -0.13 (0.10) & -0.12 (0.11) &  -0.07 (0.06) \\
hartmann 3D & 3.59 (0.19) & 3.51 (0.00) & 3.70 (0.40) & 3.61 (0.18) & 3.72 (0.08) & 3.80 (0.06) & 3.79 (0.06) & \textbf{3.84 (0.02)} \\
hartmann 6D & 2.13 (0.45) & 0.42 (0.00) & \textbf{2.58 (0.39)} & 1.97 (0.49) & 2.11 (0.48) & 2.31 (0.38) & 2.33 (0.35) &  2.51 (0.32) \\
rosenbrock 3D & -0.32 (0.15) & -0.54 (0.00) & -0.16 (0.10) & -0.29 (0.15) & -0.20 (0.08) & -0.27 (0.12) & -0.28 (0.13) &  \textbf{-0.11 (0.04)} \\
perm 10D & -0.06 (0.07) & -20.96 (0.00) & -0.38 (0.66) & -0.08 (0.08) & -0.05 (0.05) & -0.05 (0.06) & -0.06 (0.06) &  \textbf{-0.04 (0.04)} \\
perm 20D & -1.27 (1.08) & nan (nan) & -5.33 (14.95) & -1.49 (1.44) & -1.36 (1.44) & -1.08 (1.02) & -1.30 (1.32) & \textbf{-0.88 (0.86)} \\	
\\
\hline
% \textbf{Total Score} &   &  &  &  &   &   &  &  &  \\ 
% \hline
\textbf{\# Top-1} & 2  & 3 & 8 & 3 &  2 & 2  & 3 & 17 \\  \hline
\end{tabular}
}
\caption{\small Comparison of the maximum returned value using a budget of $n=100$ for the different algorithms, averaged over 100 repetitions with standard deviations reported. The ECP hyperparameters are fixed across all problems: $\varepsilon_1=10^{-2}$, $\tau_{n,d} = \max\left(1+\frac{1}{n d}, \tau\right)$, with $\tau=1.001$, and $C = 10^3$. For AdaLIPO, we set $p=0.1$, as specified in their paper. nan indicates an error related insufficient budget to return a value.}
\label{tab:example_100}
\end{table*}

\begin{table*}[ht]
\centering
\resizebox{0.99\textwidth}{!}{
\begin{tabular}{ccccccccc}
\hline
\textbf{Objective} & \textbf{PRS} & \textbf{Direct} & \textbf{CMA-ES} & \textbf{DualAnnealing} & \textbf{NeuralUCB} & \textbf{AdaLIPO} & \textbf{AdaLIPO+} & \textbf{ECP (Ours)} \\ \hline
\\
ackley & -3.33 (0.91) & -1.64 (0.00) & -10.07 (6.63) & -1.98 (0.94) & -3.04 (0.79) & Infinite Loop & Infinite Loop & \textbf{-0.25 (0.11)} \\ 
bukin & -10.02 (4.14) & -27.38 (0.00) & \textbf{-0.18 (0.17)} & -7.35 (3.61) & -7.58 (4.55) & Infinite Loop & Infinite Loop & -2.86 (1.51) \\ 
camel & 0.99 (0.06) & 1.01 (0.00) & \textbf{1.03 (0.00)} & 1.02 (0.01) & 1.02 (0.01) & Infinite Loop & Infinite Loop &  \textbf{1.03 (0.00)} \\ 
crossintray & 2.07 (0.03) & \textbf{2.12 (0.00)} & 1.72 (0.15) & 2.11 (0.02) & 2.08 (0.02) & Infinite Loop & Infinite Loop & \textbf{2.12 (0.00)} \\ 
damavandi & -2.21 (0.18) & -2.01 (0.00) & \textbf{-2.00 (0.00)} & -2.05 (0.05) & -2.24 (0.19) & Infinite Loop & Infinite Loop &  -2.01 (0.01) \\ 
rosenbrock & -0.22 (0.10) & -0.54 (0.00) & \textbf{-0.05 (0.00)} & -0.17 (0.05) & -0.12 (0.04) & Infinite Loop & Infinite Loop & -0.08 (0.01) \\ 
\\
\hline
\textbf{\# Top-1} & 0 & 1 & 4 & 0 & 0 & 0 & 0 & 3 \\ \hline
\end{tabular}
}
\caption{\small Comparison of the maximum returned value using a budget of 300 for the different algorithms, averaged over multiple repetitions with standard deviations reported. The ECP hyperparameters are fixed across all problems: $\varepsilon_1 = 10^{-2}$, $\tau_{n,d} = \max\left(1 + \frac{1}{n d}, \tau\right)$, with $\tau = 1.001$, and $C = 10^3$. For AdaLIPO, we set $p = 0.1$, as specified in their paper. Infinite Loop indicates that the method failed to return a valid value within the budget.}
\label{tab:optimization_results}
\end{table*}

\newpage
\section{Extended Related Works}
\label{app:related_works}

Several methods have been proposed for global optimization \citep{torn1989global, horst2013handbook, floudas2014recent, zabinsky2013stochastic, stork2022new}. The most straightforward ones are non-adaptive exhaustive searches, such as brute-force methods, also known as grid search, which involves dividing the space into representative points and evaluating each one \citep{zabinsky2013stochastic}. A stochastic version of grid search is Pure Random Search (PRS) \citep{brooks1958discussion, zabinsky2013stochastic}, which uses random uniform sampling. While both of these approaches may work in certain situations, their non-adaptive nature makes them generally inefficient, particularly in low-budget settings, as they can lead to unnecessary function evaluations by failing to leverage previously discovered information \citep{zabinsky2013stochastic}, as well as, failing to leverage any potential structure of the function.

To improve upon exhaustive search, several methods have been proposed to leverage previously discovered information, as well as potential structure of the objective function, such as Lipschitz continuity or smoothness \citep{shubert1972sequential, kleinberg2008multi, bubeck2011x, munos2011optimistic, preux2014bandits, kawaguchi2016global, pmlr-v98-bartlett19a}. Some of these algorithms need the knowledge of the local smoothness such as HOO \citep{bubeck2011x}, Zooming \cite{kleinberg2008multi}, or DOO \citep{munos2011optimistic}. Among the works relying on an unknown local smoothness, SOO \citep{munos2011optimistic, preux2014bandits, kawaguchi2016global} and SequOOL \citep{pmlr-v98-bartlett19a}. However, in this work we focus on Lipschitz continious functions, with unkonw Lipschitz constants. 

The introduction of the Lipschitz constant, first proposed in the pioneering works of \cite{shubert1972sequential} and \cite{piyavskii1972algorithm}, sparked significant research and has been instrumental in the creation of numerous effective global optimization algorithms, including DIRECT \citep{jones1993lipschitzian}, MCS \citep{huyer1999global}, and, more recently, AdaLIPO/AdaLIPO+ \citep{malherbe2017global, serre2024lipo+}.

The DIRECT algorithm \citep{jones1993lipschitzian} is a Lipschitz optimization algorithm where the Lipschitz constant is unknown. It uses a deterministic splitting technique of the search space in order to sequentially divide and evaluate the function over a subdivision of the space that have recorded the highest upper bound among all subdivisions of similar size for at least a possible value of $k$. \cite{preux2014bandits} generalized DIRECT in a broader setting by extending the DOO algorithm to any unknown and arbitrary local semi-metric under the name SOO. The no-regret property of DIRECT was shown in \cite{finkel2004convergence} and \cite{munos2014bandits} derived convergence rates for SOO using weaker local smoothness assumptions. However, regret upper bounds are not known for SOO or DIRECT. 

\cite{malherbe2017global} proposed LIPO algorithm, which relies on the Lipschitz constant to derive acceptance regions, and for an unknown Lipschitz constant, proposed adaptive stochastic strategy which directly relies on the estimation of the Lipschitz constant and presents guarantees for globally Lipschitz functions. It is proven to be a no-regret (consistent) algorithm for Lipschitz functions with unknown Lipschitz constants. AdaLIPO works by sampling and evaluating points uniformly at random with some probability $p$ to estimate the Lipschitz constant $k$. The estimated constant is then used to identify potentially optimal maximizers based on previously explored points, thereby refining the search space. More recently, AdaLIPO+ was introduced as an empirical improvement over AdaLIPO \citep{serre2024lipo+}, which follows the same acceptance process as AdaLIPO, except that the exploration probability $p$ used to estimate $k$ decreases over time. Like our method, both AdaLIPO and AdaLIPO+ optimize the search space with an acceptance condition to evaluate potential maximizers based on the assumed Lipschitzness of the objective function and leveraging previously evaluated points. However, unlike our method, they require additional uniformly random samples from the entire search space to estimate the Lipschitz constant, which makes them less efficient in more extreme budget-constrained scenarios.

Beyond Lipschitz optimization, various global maximization methods have been proposed, each operating under different assumptions and function structures, all aiming to identify a global maximum. One prominent approach in black-box optimization is Bayesian optimization (BO), which builds a probabilistic model of the objective function and uses it to select the most promising points to evaluate \citep{fernando2014bayes, 7352306, frazier2018tutorial, balandat2020botorch}. While several BO algorithms are theoretically guaranteed to converge to the global optimum of the unknown function, they often rely on the assumption that the kernel's hyperparameters are known in advance. To address this limitation, hyperparameter-free approaches such as Adaptive GP-UCB \citep{JMLR:v20:18-213} have been proposed. More recently, \citet{JMLR:v23:21-0888} introduced SMAC3 as a robust and efficient baseline for global optimization. In our empirical evaluation, we demonstrate that ECP outperforms these recent BO baselines.

Evolutionary algorithms, such as CMA-ES \citep{hansen1996adapting, hansen2006cma, hansen2019pycma}, are also known for their practical efficiency, though they do not guarantee a no-regret performance under an infinite evaluation budget \citep{malherbe2017global}. Simulated annealing \citep{metropolis1953equation, kirkpatrick1983optimization}, motivated by
the physical annealing process when slowly cooling metals, is popular approach for global optimization, later extended to DualAnnealing \citep{xiang1997generalized}, which combines the generalization of classical simulated annealing and fast simulated annealing \citep{tsallis1988possible, tsallis1996generalized}. However, these methods lack the theoretical regret guarantees for Lipschitz optimization. 

Another related class of algorithms is contextual bandits \citep{auer2002using, langford2007epoch, filippi2010parametric, valko2013finite}, a special case of reinforcement learning (with a single state) \citep{sutton2018reinforcement, haarnoja2018soft, lattimore2020bandit, pmlr-v235-fourati24a}, which involve selecting from a finite set of arms with continuous representations. A notable recent development in this area is NeuralUCB \citep{zhou2020neural}, which uses neural networks to estimate upper-confidence bounds for each point. While these methods are primarily designed for selection, they can be adapted for global maximization by randomly sampling points (arms) in each round and selecting the one that maximizes the estimated upper-confidence bound, then retraining the neural network, based on all the previous observation. However, such approaches may be inefficient for small budgets, as neural networks require a large number of samples for effective training.

Finally, while this work focuses on continuous black-box function maximization, other research addresses discrete black-box function maximization, particularly in combinatorial settings. Submodular maximization, a key area within discrete optimization, where recent studies, such as those in \citep{fourati2023randomized, fourati2024combinatorial, pmlr-v235-fourati24b}, explore effective approximation algorithms for these problems. 

\begin{table*}[t]
\small
\centering
\newcolumntype{C}[1]{>{\centering\arraybackslash}p{#1}}
\newcolumntype{L}[1]{>{\raggedright\arraybackslash}p{#1}}
\resizebox{0.99\textwidth}{!}{
\begin{tabular}{L{4cm}|C{1.3cm}|C{1.3cm}|C{1.3cm}|C{1.3cm}|C{1.3cm}|C{1.3cm}|C{7cm}}
\toprule
Optimization Method & For Lipschitz $f$ & For Unkown $k$ & Stochastic Adaptive & Stopping Guarantees  & No-regret & No Space Filling & Finite Budget Regret Bound \\
\midrule
PRS            & \ding{53} & \ding{51}       & \ding{53}  & \ding{51}       &   \ding{51}    & \ding{53}   & $\mathcal{O}\left(c \cdot\left( \frac{\ln(1/\delta)}{ n } \right)^{\frac{1}{d}}\right)$     \\
CMA-ES  \citep{hansen2006cma}       & \ding{53}   & \ding{51}    & \ding{51}  & \ding{51}   & \ding{53}  &   \ding{53}           &  $-$     \\
DualAnnealing  & \ding{53}  & \ding{51}     & \ding{51}  & \ding{51}     & \ding{53} &  $-$           &   $-$     \\
NeuralUCB \citep{zhou2020neural}     & \ding{53}  & \ding{51}     & \ding{51}  & \ding{51}   & \ding{53}  &  \ding{53}          &  $-$     \\
BayesOpt \citep{fernando2014bayes}      & \ding{53}  & \ding{51}     & \ding{51}  & \ding{51}    & \ding{53} &   \ding{53}           &  $-$     \\
DIRECT \citep{jones1993lipschitzian} & \ding{51}    & \ding{51}  & \ding{53} &\ding{51}     &  $-$           &  $-$  &  $-$    \\
SOO \citep{preux2014bandits} & \ding{53}    & \ding{51}  & \ding{53} &\ding{51}     &  $-$           &  $-$  &  $-$    \\
AdaLIPO \citep{malherbe2017global} & \ding{51}  &\ding{51}  & \ding{51}  & \ding{53} & \ding{51}  & \ding{53}    &  $\mathcal{O}\left(c \cdot\left( \frac{5}{p} + \frac{2 \ln (\delta / 3)}{p \ln \left(1-\Gamma(f, k_{i^{\star}-1})\right)} \right)^{\frac{1}{d}} \cdot \left( \frac{\ln(3/\delta)}{n} \right)^{\frac{1}{d}}\right)$\\
AdaLIPO+/AdaLIPO+$|$ns \citep{serre2024lipo+} & \ding{51}  & \ding{51}  & \ding{51}  & \ding{51}/\ding{53} & \ding{53} &  \ding{53}    &  $-$    \\
LIPO \citep{malherbe2017global} & \ding{51}  & \ding{53}  & \ding{51}  & \ding{53} & \ding{51}  &  \ding{51}    &  $\mathcal{O}\left(c \cdot\left( \frac{\ln(1/\delta)}{ n } \right)^{\frac{1}{d}}\right)$\\
LIPO+/LIPO+$|$ns \citep{serre2024lipo+} & \ding{51}  & \ding{53}  & \ding{51}  & \ding{51}/\ding{53} & \ding{53} &  \ding{51}    &  $-$    \\
ECP (ours)     & \ding{51}  & \ding{51}  & \ding{51}  & \ding{51} & \ding{51}  &  \ding{51}   &  $\mathcal{O}\left(c \cdot \max\left(\left\lceil\log_{\tau_{n,d}}\left(\frac{k}{\varepsilon_1}\right)\right\rceil^{\frac{1}{d}}, 1\right) \cdot \left( \frac{\ln(1/\delta)}{n} \right)^{\frac{1}{d}}\right)$  \\
\midrule 
Lower-Bound            & - &  -      & -  &  -     &  -     & -   &   $\Omega(k \cdot \left(\frac{1}{n}\right)^\frac{1}{d})$   \\
\bottomrule
\end{tabular}
}
\caption{\small Theoretical Comparison of Global Optimization Methods. Here, \(\delta \in (0,1)\) and \(c = k \cdot \diam{\mathcal{X}}\), where \(k\) is the unknown Lipschitz constant. The first column indicates whether the optimization method was proposed for Lipschitz functions. The second column specifies whether the method requires prior knowledge. The third column indicates whether or not the method is stochastic adaptive, stochastically leveraging the collected data and Lipschitzness of the function. The fourth column addresses whether the method is guaranteed to terminate within a specified budget of evaluations. The fifth column indicates whether the method provides no-regret guarantees. The sixth column notes whether the method eliminates blind space filling or pure random search. The final column presents the known finite budget upper bounds on the regret for Lipschitz functions.
}
\label{table_rejection_proba}
\end{table*}

\section{Missing Proofs, Lemmas, and Propositions}
\label{app:missing_proofs}

\subsection{Preliminaries}
\label{app:prelimnaries}

\begin{lemma}(\cite{malherbe2017global}) 
\label{prop:potential_k}
If $\mathcal{P}_{k,t}$ denotes the set of potential maximizers of the function $f$, as defined in \cref{def:consistent_functions},
then we have $\mathcal{A}_{k,t} = \mathcal{P}_{k,t}$. 
\end{lemma}

\begin{lemma}
\label{prop:cvg_prs}
{(\cite{malherbe2017global})}. Let $\X \subset \R^d$ be
a compact and convex set with non-empty interior
and let $f \in \Lip(k)$ be a $k$-Lipschitz functions defined on $\X$ for some $k\geq0$.
Then, 
for any $n \in \mathbb{N}^{\star}$ and
$\delta \in(0,1)$, we have with probability at least $1-\delta$,
\[
\max_{x \in \X}f(x) -\max_{i=1, \cdots, n} f(x_i) 
\leq k \cdot \diam{\X} \cdot \left(  \frac{\ln(1/\delta)}{n} \right)^{\frac{1}{d}}
\]
where $x_1, \cdots, x_n$ denotes a sequence of $n$ independent copies of
$x \sim \mathcal{U}(\X)$.
\end{lemma}

\subsection{Proof of \cref{lem:montonic_with_epsilon}}
\label{prooflem:montonic_with_epsilon}

Assume $\exists x \in \mathcal{A}_{u, t}$, where $\varepsilon_t = u$, we have
$$
\min_{i=1, \cdots, t}  f(x_i) + u \cdot \norm{x-x_i}_2
\geq \max_{i=1, \cdots, t}f(x_i).
$$
Hence, for $u \leq v$, then
$$
\begin{aligned}
 \min_{i=1, \cdots, t}  f(x_i) + v \cdot \norm{x-x_i}_2
&\geq \min_{i=1, \cdots, t}  f(x_i) + u \cdot \norm{x-x_i}_2
&\geq \max_{i=1, \cdots, t}f(x_i).   
\end{aligned}
$$ 
Hence, $x \in \mathcal{A}_{v, t}$.
Therefore, $\forall u \leq v, \quad \mathcal{A}_{u,t} \subseteq \mathcal{A}_{v,t}$.

\hfill \(\square\)

\subsection{Proof of \cref{prop:potential}}
\label{proof:prop:potential}
Using \cref{prop:potential_k} from \cref{app:prelimnaries}, we have $\mathcal{P}_{k,t} = \mathcal{A}_{k,t}$. Consider two cases:
\begin{itemize}
    \item \textbf{Case 1}: If $x \in \mathcal{A}_{\varepsilon_t, t}$, then by \cref{lem:montonic_with_epsilon}, for all $\varepsilon_t \leq k$, we have $\mathcal{A}_{\varepsilon_t,t} \subseteq \mathcal{A}_{k,t} = \mathcal{P}_{k,t}$.
    \item \textbf{Case 2}: If $x \in \mathcal{P}_{k,t} = \mathcal{A}_{k,t}$, then by \cref{lem:montonic_with_epsilon}, for all $\varepsilon_t \geq k$, we have $\mathcal{P}_{k,t} = \mathcal{A}_{k,t} \subseteq \mathcal{A}_{\varepsilon_t,t}$.
\end{itemize}
\hfill \(\square\)

\subsection{Proof of \cref{prop:nonincreasing}}
\label{proof:prop:nonincreasing}
At time $t_2+1$, for some value $\varepsilon$, a point $x \in \mathcal{X}$ is accepted if and only if it belongs to a ball $\mathcal{A}_{\varepsilon, t_2+1}$ within $\mathcal{X}$ :
$$
\min _{1 \leq i \leq t_2} f\left(x_i\right)+\varepsilon\left\|x-x_i\right\|_2 \geq \max _{1 \leq i \leq t_2} f\left(x_i\right).
$$
As $t_1 \leq t_2$, we have
$$
\begin{aligned}
\min _{1 \leq i \leq t_1} f\left(x_i\right) &+ \varepsilon\left\|x-x_i\right\|_2 \geq 
&\min _{1 \leq i \leq t_2} f\left(x_i\right)+\varepsilon\left\|x-x_i\right\|_2.
\end{aligned}
$$
And we have 
$$
\begin{aligned}
\max_{1 \leq i \leq t_2} f\left(x_i\right) \geq \max _{1 \leq i \leq t_1} f\left(x_i\right)
\end{aligned}
$$
Therefore, $\mathcal{A}_{\varepsilon, t_2} \subseteq \mathcal{A}_{\varepsilon, t_1}$. 
\hfill \(\square\)

\subsection{Proof of \cref{prop:rejection_proba}}
\label{proof:rejection_proba}

At time $t$, a candidate $x \in \mathcal{X}$ is rejected if and only if it belongs to a ball within $\mathcal{X}$ :
$$
\min _{1 \leq i<t} f\left(x_i\right)+\varepsilon_t\left\|x-x_i\right\|_2<\max _{1 \leq i<t} f\left(x_i\right)
$$
Let $j$ be in the arg min of the LHS of the above inequality. 
$$
\begin{aligned}
f\left(x_j\right)+\varepsilon_t\left\|x-x_j\right\|_2<\max _{1 \leq i<t} f\left(x_i\right) & \Longleftrightarrow \varepsilon_t\left\|x-x_j\right\|<\max _{1 \leq i<t} f\left(x_i\right)-f\left(x_j\right) \\
& \Longleftrightarrow x \in B\left(x_j, \frac{\max _{1 \leq i<t} f\left(x_i\right)-f\left(x_j\right)}{\varepsilon_t}\right) \bigcap \mathcal{X} \\
& \Longrightarrow x \in B\left(x_j, \frac{\max _{1 \leq i<t} f\left(x_i\right)-f\left(x_j\right)}{\varepsilon_t}\right) \\
& \Longrightarrow x \in B\left(x_j, \frac{\max _{x \in \mathcal{X}} f(x)-\min _{x \in \mathcal{X}} f(x)}{\varepsilon_t}\right) .
\end{aligned}
$$
Therefore, the volume of a ball of radius $\frac{\Delta}{\varepsilon_t}$ is an upper bound on the volume that can be removed from the region of potential maximizers, for any sequence of iterations $\left(x_i\right)_{1 \leq i<t}$. Thus, at time $t+1$, at most the volume of $t$ disjoint balls of radius $\frac{\Delta}{\varepsilon_t}$ could be removed. This leads to the following lower bound on the volume in which potential maximizers should be seek $\mathcal{V}_{t+1}$ verifies:
$$
\mathcal{V}_{t+1} \geq \lambda(\mathcal{X})-\frac{t \pi^{d / 2} \Delta^d}{\varepsilon_t^d \Gamma(d / 2+1)}
$$
As ECP samples candidates uniformly at random in $\mathcal{X}$, the probability of rejecting a candidate is bounded from above by the probability of sampling uniformly at a point in the union of the $t$ disjoint balls. 

When $\varepsilon_t$ increases within the same iteration $t$, it is scaled by a multiplicative factor $\tau_{n,d}>1$ whenever growth is detected, i.e., $\varepsilon_t$ becomes $\varepsilon_t \tau_{n,d}^{v_t}$, where $v_t$ represents the number of growth detection within iteration $t$.

Finally, at time $t$, by design of the algorithm $\varepsilon_t \geq \varepsilon_1 \tau_{n,d}^{t-1}$. Hence, after $v_t$ growth detection within iteration $t$, $\varepsilon_t \geq \varepsilon_1 \tau_{n,d}^{t-1} \tau_{n,d}^{v_t}$

\hfill \(\square\)

\subsection{AdaLIPO/AdaLIPO+ Rejection Probability}
\label{rejection_adalipo_app}

\begin{proposition}{\sc (AdaLIPO/AdaLIPO+ Rejection Probability)} \label{prop:rejection_proba_adalipo}  For any $k$-Lipschitz function $f$, let $\left(x_i\right)_{1 \leq i \leq t}$ be the previously generated and evaluated points of AdaLIPO or AdaLIPO+ until time $t$. For any $x \in \mathcal{X}$, let $R(x, t)$ be the event of rejecting $x$ at time $t+1$. We have the following upper bound:
$$
\mathbb{P}(R(x, t+1)) \leq \frac{(1-p_t) t (\sqrt{\pi}\Delta)^{d} }{k_t^d \Gamma(d / 2+1) \lambda(\mathcal{X})}
$$
where $p_t = \min(1, \ln(\frac{1}{t}))$ and $p_t = p \in (0,1)$, are the exploration probabilities for AdaLIPO+ and AdaLIPO, respectively, to estimate the unknown Lipschitz constant $k$. $\Delta=\max _{x \in \mathcal{X}} f(x)-\min _{x \in \mathcal{X}} f(x)$, $\lambda$ is the standard Lebesgue measure that generalizes the notion of volume of any open set, and $\Gamma$ is the Gamma function given by $\Gamma(x)=\int_0^{\infty} t^{x-1} e^{-t} \mathrm{~d} t$.
\end{proposition}

\begin{proof}
We begin by observing that the AdaLIPO+ and AdaLIPO algorithms explore the entire space $\mathcal{X}$ uniformly at random with a probability $p_t$. Specifically, for AdaLIPO+, the exploration probability is $p_t = \min(1, \ln\left(\frac{1}{t}\right))$, while for AdaLIPO, the exploration probability is a constant $p \in (0,1)$. Both methods aim to estimate the unknown Lipschitz constant $k$.

Next, note that during exploration, with probability $p_t$, a point $x$ is accepted with certainty (i.e., rejected with probability zero). To formalize this, we apply the law of total probability to compute the probability of rejecting a point $x$ at time $t+1$:
\[
\begin{aligned}
   \mathbb{P}(R(x, t+1)) &= (1 - p_t) \mathbb{P}(R(x, t+1) \mid \mathcal{E} = 0) + p_t \mathbb{P}(R(x, t+1) \mid \mathcal{E} = 1) \\
   &= (1 - p_t) \mathbb{P}(R(x, t+1) \mid \mathcal{E} = 0) + p_t \cdot 0 \\
   &= (1 - p_t) \mathbb{P}(R(x, t+1) \mid \mathcal{E} = 0),
\end{aligned}
\]
where $\mathcal{E}$ denotes the event that the algorithm is in an exploitation phase (i.e., $\mathcal{E} = 1$). The second equality follows from the fact that during exploration ($\mathcal{E} = 1$), a point is never rejected, hence $\mathbb{P}(R(x, t+1) \mid \mathcal{E} = 0) = 0$. It remains to upper-bound the probability of rejecting a point during the exploitation phase, $\mathbb{P}(R(x, t+1) \mid \mathcal{E} = 0)$, which can be done by analyzing the specific rejection criteria of the algorithm during exploitation.

At time $t$, a candidate $x \in \mathcal{X}$ is rejected if and only if it belongs to a ball within $\mathcal{X}$ :
$$
\min _{1 \leq i<t} f\left(x_i\right)+k_t\left\|x-x_i\right\|_2<\max _{1 \leq i<t} f\left(x_i\right)
$$
Let $j$ be in the arg min of the LHS of the above inequality. 
$$
\begin{aligned}
f\left(x_j\right)+\varepsilon_t\left\|x-x_j\right\|_2<\max _{1 \leq i<t} f\left(x_i\right) & \Longleftrightarrow k_t\left\|x-x_j\right\|<\max _{1 \leq i<t} f\left(x_i\right)-f\left(x_j\right) \\
& \Longleftrightarrow x \in B\left(x_j, \frac{\max _{1 \leq i<t} f\left(x_i\right)-f\left(x_j\right)}{k_t}\right) \bigcap \mathcal{X} \\
& \Longrightarrow x \in B\left(x_j, \frac{\max _{1 \leq i<t} f\left(x_i\right)-f\left(x_j\right)}{k_t}\right) \\
& \Longrightarrow x \in B\left(x_j, \frac{\max _{x \in \mathcal{X}} f(x)-\min _{x \in \mathcal{X}} f(x)}{k_t}\right) .
\end{aligned}
$$
Therefore, the volume of a ball of radius $\frac{\Delta}{k_t}$ is an upper bound on the volume that can be removed from the region of potential maximizers, for any sequence of iterations $\left(x_i\right)_{1 \leq i<t}$. Thus, at time $t+1$, at most the volume of $t$ disjoint balls of radius $\frac{\Delta}{k_t}$ could be removed. This leads to the following lower bound on the volume in which potential maximizers should be seek $\mathcal{V}_{t+1}$ verifies:
$$
\mathcal{V}_{t+1} \geq \lambda(\mathcal{X})-\frac{t \pi^{d / 2} \Delta^d}{k_t^d \Gamma(d / 2+1)}
$$
As AdaLIPO/AdaLIPO$+$ samples candidates uniformly at random in $\mathcal{X}$, the probability of rejecting a candidate is bounded from above by the probability of sampling uniformly at a point in the union of the $t$ disjoint balls. 
\end{proof}

\subsection{Proof of \cref{cor:non_zero_acceptance}}
\label{proof:cor:non_zero_acceptance}

We define $v_t$ as the number of increases of $\varepsilon_t$ at a given iteration $t$. 
From \cref{prop:rejection_proba} we have 
$$
\begin{aligned}
 \mathbb{P}(R(x, t+1), v_t) &\leq \frac{t (\sqrt{\pi}\Delta)^{d} }{\varepsilon_1^d \tau_{n,d}^{(t-1)d} \tau_{n,d}^{v_t d} \Gamma(d / 2+1) \lambda(\mathcal{X})}.  
\end{aligned}
$$
Now choose any
$$
\begin{aligned}
 v &\geq \frac{1}{d} \log_{\tau_{n,d}}\left(\frac{2 n (\sqrt{\pi}\Delta)^{d} }{\varepsilon_1^d \Gamma(d / 2+1) \lambda(\mathcal{X})}\right)\\
 &\geq \frac{1}{d} \log_{\tau_{n,d}}\left(\frac{2 t (\sqrt{\pi}\Delta)^{d} }{\varepsilon_1^d \tau_{n,d}^{(t-1)d} \Gamma(d / 2+1) \lambda(\mathcal{X})}\right)   
\end{aligned}
$$ 
Thus,
$$
\begin{aligned}
 v d \ln(\tau_{n,d}) \geq  \ln\left(\frac{2 t (\sqrt{\pi}\Delta)^{d} }{\varepsilon_1^d \tau_{n,d}^{(t-1)d} \Gamma(d / 2+1) \lambda(\mathcal{X})}\right)   
\end{aligned}
$$ 
Then,
$$
\begin{aligned}
 \tau_{n,d}^{v d} \geq  \frac{2 t (\sqrt{\pi}\Delta)^{d} }{\varepsilon_1^d \tau_{n,d}^{(t-1)d} \Gamma(d / 2+1) \lambda(\mathcal{X})}  
\end{aligned}
$$ 
Then, we have  $\mathbb{P}(R(x, t+1, v)) \leq 1/2$, hence the probability of acceptance $\mathbb{P}(A(x, t+1)) \geq 1/2$.
\hfill \(\square\)

\subsection{Proof of \cref{prop:rejection_growth}} 
\label{proof:prop:rejection_growth}
We prove this by induction. Notice that $h_1 = 1 \leq 1 \cdot C$. Furthermore, we have $h_{t+1} \leq h_t + C$. Therefore, $h_2 \leq h_1 + C \leq 2 \cdot C$. Now, assume the statement is true for some iteration time $t \geq 1$, i.e., $h_{t} \leq t \cdot C$. Then,
$$
h_{t+1} \leq h_t + C \leq t \cdot C + C \ \leq (t+1) \cdot C.
$$
Thus, the proposition holds for all $t \geq 1$ by induction.
\hfill \(\square\)

\subsection{Proof of \cref{thm:complexity}}
\label{proof:thm:complexity}
 Notice that for any budget $n \in \mathbb{N}^{\star}$, the computational complexity of running ECP is bounded by $\mathcal{O}(n \cdot v \cdot \max_{t =1 \cdots n} h_t )$. 
 
 By \cref{prop:rejection_growth}, we have $\forall t\geq 1$, the stochastic growth condition used in ECP, ensure $h_t \leq t \cdot C \leq n \cdot C$. 
 
 Moreover, by \cref{cor:non_zero_acceptance}, we know with at most $
 v = \left\lceil\frac{1}{d} \log_{\tau_{n,d}}\left(\frac{2 n (\sqrt{\pi}\Delta)^{d} }{\varepsilon_1^d \Gamma(d / 2+1) \lambda(\mathcal{X})}\right)\right\rceil
 $ increases, the acceptance probability is at least $1/2$. Thus, with $\delta_t=\mathbb{P}(R(x, t+1), v)\leq1/2$, the growth happens again with at most $\delta \leq (1/2)^{tC}$ after $v$ growths, in round $t$. Hence, by the union bound, the probability of the growth after $v$ growths happens in any of the $t$ rounds, is at most $\frac{1}{2^C} \cdot \frac{1-\frac{1}{2^C}^{t}}{1-\frac{1}{2^C}} \leq \frac{1}{2^C} \cdot \frac{1}{1-\frac{1}{2^C}} \leq \frac{1}{2^C-1}$.
 % = \frac{1-\frac{1}{2}^{nC}}{2^C-1}  $.
 % Thus, with a probability $1 - \delta > 1- (1/2)^{C}$ the growth stops after $v_{n,d}$ growths, in round $t$. 
 
 Therefore, for any $\varepsilon_1 > 0$, any $\tau_{n,d} > 1$, any constant $C > 1$, and any function $f \in \mathcal{F}$, there exists $\delta$, such as with a probability $1 - \delta > 0$, the computational complexity of running ECP is bounded by $\mathcal{O}\left(n^2 \cdot \frac{1}{d} \log_{\tau_{n,d}}\left(\frac{2 n (\sqrt{\pi}\Delta)^{d} }{\varepsilon_1^d \Gamma(d / 2+1) \lambda(\mathcal{X})} \right) \cdot C\right)$.
\hfill \(\square\)

\subsection{Proof of \cref{hitting_time_upperbound}}
\label{proof:hitting_time_upperbound}
Notice that for all $t \geq 1$, for any choice of $C > 1$, $\varepsilon_1 > 0$, and $\tau_{n,d} > 1$, $\varepsilon_t$ grows throughout the iterations when the stochastic growth condition is satisfied and when a point is evaluated. While the growth condition may or may not occur, the latter happens deterministically after every evaluation. Therefore, $\varepsilon_t \geq \varepsilon_1 \tau_{n,d}^{t-1}$. Hence, the result follows from the non-decreasing and diverging geometric growth of $\varepsilon_1 \tau_{n,d}^{t-1}$.
\hfill \(\square\)

\subsection{Proof of \cref{thm:proba_convergence}}
\label{app:prop:proba_convergence}

For any given function $f$, with some fixed unknown Lipschitz constant $k$, and for any chosen constants $\varepsilon_1 > 0$, $\tau_{n,d} > 1$, and $C > 1$, as shown in \cref{hitting_time_upperbound}, there exists a constant $L = \left\lceil\log _{\tau_{n,d}}\left(\frac{k}{\varepsilon_1}\right)\right\rceil$, not depending on $n$, such that for $t \geq L$, we have $t \geq i^\star$. Hence, by \cref{i_star_definition}, for $t \geq L$, $\varepsilon_t$ reaches and exceeds $k$. Therefore, it is guaranteed that as $t$ tends to infinity, $\varepsilon_t$ surpasses $k$. Furthermore, by \cref{prop:potential}, we know that for all $\varepsilon_t > k$, $\mathcal{P}_{k,t} \subseteq \mathcal{A}_{\varepsilon_t,t}$. Thus, as $t$ tends to infinity, the search space uniformly recovers all the potential maximizers and beyond.

\subsection{Proof of \cref{prop:fasterprs}}
\label{app:faster}

We proceed by induction. Let \( x_1, \cdots, x_{i^\star} \) be a sequence of evaluation points generated by ECP after \( i^\star \) iterations, and let \( x_{i^\star}' \) be an independent point randomly sampled over \( \mathcal{X} \). Consider any \( y \geq \max_{i=1, \cdots, i^\star - 1} f(x_i) \), and define the corresponding level set \( \mathcal{X}_y = \{ x \in \mathcal{X} : f(x) \geq y \} \).

Assume, without loss of generality, that \( \mu(\mathcal{X}_y) > 0 \) (otherwise, \( \mathbb{P}(f(x_{i^\star}) \geq y) = 0 \), and the result trivially holds).

Now, recall that for all \( t \geq 1 \), by \cref{def:potential_maximizers},
$$
\begin{aligned}
\mathcal{P}_{k,t} &= \left\{x \in \mathcal{X} : \exists g \in \mathcal{F}_{k,t} \text{~such that~} x \in \underset{x \in \mathcal{X}}{\arg \max} g(x) \right\} \\
&= \left\{ x \in \mathcal{X} : \min_{i=1, \cdots, t} (f(x_i) + k \cdot \|x - x_i\|_2) \geq \max_{i=1, \cdots, t} f(x_i) \right\},
\end{aligned}
$$
where the second equality follows from \cref{prop:potential_k}.

Additionally, by \cref{def:potentially_accepted}, we have
\[
\mathcal{X}_{\varepsilon_t,t} := \left\{ x \in \mathcal{X} : \min_{i=1, \cdots, t} (f(x_i) + \varepsilon_t \cdot \|x - x_i\|_2) \geq \max_{i=1, \cdots, t} f(x_i) \right\}.
\]
For \( t = i^\star \), we have \( \varepsilon_{i^\star} \geq k \), implying that \( \mathcal{X}_{k,i^\star} \subseteq \mathcal{X}_{\varepsilon_{i^\star},i^\star} \subseteq \mathcal{X} \).

Moreover, since \( y \geq \max_{i=1, \cdots, i^\star - 1} f(x_i) \), if \( \mathcal{X}_y \) is non-empty, its elements are potential maximizers. Hence, \( \mathcal{X}_y \subseteq \mathcal{X}_{k,i^\star} \). If \( \mathcal{X}_y \) is empty, the result holds trivially.

Next, we compute the following probabilities:
\[
\begin{aligned}
\mathbb{P}(f(x_{i^\star}) \geq y) = \mathbb{E}\left[\mathbb{I}\{ x_{i^\star} \in \mathcal{X}_y \} \right] = \mathbb{E}\left[\frac{\mu(\mathcal{X}_{\varepsilon_{i^\star},i^\star} \cap \mathcal{X}_y)}{\mu(\mathcal{X}_{\varepsilon_{i^\star},i^\star})} \right] \geq \mathbb{E}\left[\frac{\mu(\mathcal{X}_{k,i^\star} \cap \mathcal{X}_y)}{\mu(\mathcal{X})} \right] = \mathbb{E}\left[\frac{\mu(\mathcal{X}_y)}{\mu(\mathcal{X})} \right] = \mathbb{P}(f(x_{i^\star}') \geq y).
\end{aligned}
\]

Now, suppose the statement holds for some \( n \geq i^\star \). Let \( x_1, \cdots, x_{n+1} \) be a sequence of evaluation points generated by ECP after \( n+1 \) iterations, and let \( x_1', \cdots, x_{n+1}' \) be a sequence of \( n+1 \) independent points sampled over \( \mathcal{X} \).

As before, assume \( \mu(\mathcal{X}_y) > 0 \), and let \( \mathcal{A}_{\varepsilon_n,n} \) denote the sampling region of \( x_{n+1} \mid x_1, \cdots, x_n \). Then, on the event \( \{ \max_{i=i^\star, \cdots, n} f(x_i) < y \} \), we have \( \mathcal{X}_y \subseteq \mathcal{X}_{k,n} \subseteq \mathcal{A}_{\varepsilon_n,n} \subseteq \mathcal{X} \).

We now compute:
\[
\begin{aligned}
\mathbb{P}\left(\max_{i=i^{\star}, \cdots, n+1} f(x_i) \geq y\right) & = \mathbb{E}\left[\mathbb{I}\left\{\max_{i=i^{\star}, \cdots, n} f(x_i) \geq y\right\} + \mathbb{I}\left\{\max_{i=i^{\star}, \cdots, n} f(x_i) < y, x_{n+1} \in \mathcal{X}_y\right\}\right] \\
& = \mathbb{E}\left[\mathbb{I}\left\{\max_{i=i^{\star}, \cdots, n} f(x_i) \geq y\right\} + \mathbb{I}\left\{\max_{i=i^{\star}, \cdots, n} f(x_i) < y\right\} \frac{\mu\left(\mathcal{A}_{\varepsilon_n,n} \cap \mathcal{X}_y\right)}{\mu\left(\mathcal{A}_{\varepsilon_n,n}\right)}\right] \\
& \geq \mathbb{E}\left[\mathbb{I}\left\{\max_{i=i^{\star}, \cdots, n} f(x_i) \geq y\right\} + \mathbb{I}\left\{\max_{i=i^{\star}, \cdots, n} f(x_i) < y\right\} \frac{\mu\left(\mathcal{X}_{k,n} \cap \mathcal{X}_y\right)}{\mu\left(\mathcal{A}_{\varepsilon_n,n}\right)}\right] \\
& \geq \mathbb{E}\left[\mathbb{I}\left\{\max_{i=i^{\star}, \cdots, n} f(x_i) \geq y\right\} + \mathbb{I}\left\{\max_{i=i^{\star}, \cdots, n} f(x_i) < y\right\} \frac{\mu\left(\mathcal{X}_y\right)}{\mu(\mathcal{X})}\right] \\
& \geq \mathbb{E}\left[\mathbb{I}\left\{\max_{i=i^{\star}, \cdots, n} f(x_i') \geq y\right\} + \mathbb{I}\left\{\max_{i=i^{\star}, \cdots, n} f(x_i') < y\right\} \frac{\mu\left(\mathcal{X}_y\right)}{\mu(\mathcal{X})}\right] \\
&= \mathbb{E}\left[\mathbb{I}\left\{\max_{i=i^{\star}, \cdots, n} f(x_i') \geq y\right\} + \mathbb{I}\left\{\max_{i=i^{\star}, \cdots, n} f(x_i') < y, x_{n+1}' \in \mathcal{X}_y\right\}\right] \\
& = \mathbb{P}\left(\max_{i=i^{\star}, \cdots, n+1} f(x_i') \geq y\right)
\end{aligned}
\]
where the third inequality follows from the fact that \( x \mapsto \mathbb{I}\{x \geq y\} + \mathbb{I}\{x < y\} \frac{\mu(\mathcal{X}_y)}{\mu(\mathcal{X})} \) is non-decreasing, and the induction hypothesis implies that \( \max_{i=i^\star, \cdots, n} f(x_i) \) stochastically dominates \( \max_{i=i^\star, \cdots, n} f(x_i') \).

Thus, by induction, the statement holds for all \( n \geq i^\star \), completing the proof.

\hfill \(\square\)

\subsection{Proof of \cref{thm:ecpupperbound}}
\label{proof:thm:ecpupperbound}

Fix any $\delta \in (0,1)$. Set $i^\star$ as defined in \cref{i_star_definition}.
Considering any $n> i^\star$. 
As the function satisfies $f\in\Lip(k)$ and for all $t\geq i^\star$, $\varepsilon_t \geq k$, as shown in Proposition \ref{prop:fasterprs}, for $t \geq i^\star$ the algorithm is always faster or equal to a Pure Random Search
with $n-i^\star + 1$ i.i.d.~copies of $x' \sim\mathcal{U}(\X)$, in achieving higher values than what is already achieved, i.e, for $y \geq \max_{i=1, \cdots, i^\star-1}f(x_i)$. Therefore, using the bound of \cref{prop:cvg_prs}, we obtain that with probability at least 
$1-\delta$,
\begin{align*}
\mathcal{R}_{\text{ECP}, f}(n)
& \leq k \cdot
\diam{\X} \cdot \left(  \frac{\ln(1/\delta)}{ n-i^\star + 1} \right)^{\frac{1}{d}} \\
& = k \cdot
\diam{\X} \cdot \left(  \frac{n}{n-i^\star + 1} \right)^{\frac{1}{d}} \cdot
\left(  \frac{\ln(1/\delta)}{ n } \right)^{\frac{1}{d}}  \\
& \leq k \cdot
\diam{\X} \cdot \left(i^\star \right)^{\frac{1}{d}}
\left(  \frac{\ln(1/\delta)}{ n } \right)^{\frac{1}{d}}%\\
\end{align*}
where the last inequality follows from \cref{hitting_time_upperbound}.

The result is extended to the case where $n\leq i^\star$ by noticing that 
the bound is superior to $k\cdot \diam{\X}$ 
in that case, and thus trivial.
\hfill \(\square\)

\begin{figure*}[t]
    \centering
    \begin{subfigure}[b]{0.16\textwidth}
        \includegraphics[width=\textwidth]{figures/ackley_plot.png}
        \caption{\small  Ackley}
    \end{subfigure}
    \begin{subfigure}[b]{0.16\textwidth}
        \includegraphics[width=\textwidth]{figures/bukin_plot.png}
        \caption{\small  Bukin}
    \end{subfigure}
    \begin{subfigure}[b]{0.16\textwidth}
        \includegraphics[width=\textwidth]{figures/camel_plot.png}
        \caption{\small  Camel}
    \end{subfigure}
    \begin{subfigure}[b]{0.16\textwidth}
        \includegraphics[width=\textwidth]{figures/crossintray_plot.png}
        \caption{\small  Crossintray}
    \end{subfigure}
    \begin{subfigure}[b]{0.16\textwidth}
        \includegraphics[width=\textwidth]{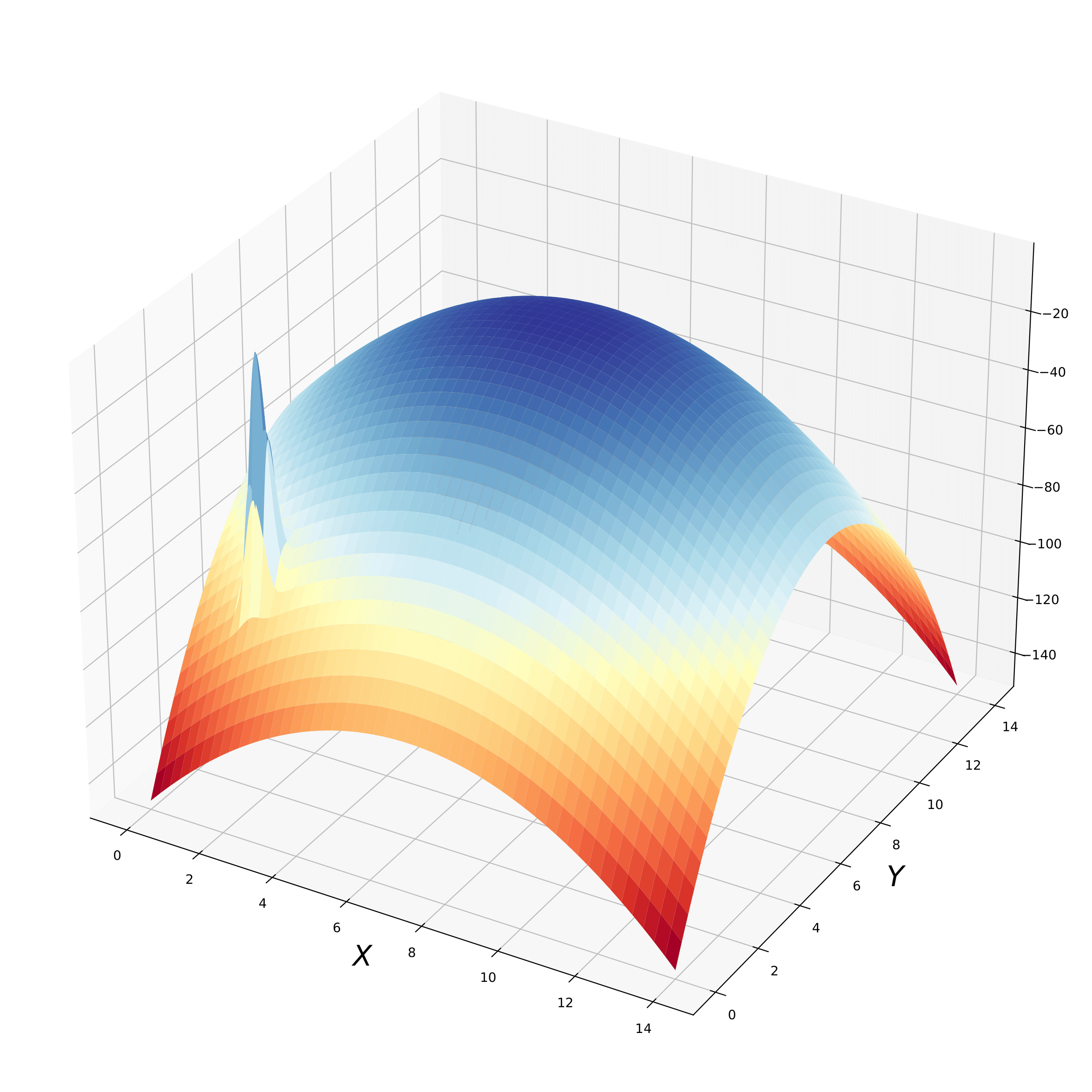}
        \caption{\small  Damavandi}
        \end{subfigure}
        \hfill
    \begin{subfigure}[b]{0.16\textwidth}
        \includegraphics[width=\textwidth]{figures/dropwave_plot.png}
        \caption{\small  Dropwave}
    \end{subfigure}
        \begin{subfigure}[b]{0.16\textwidth}
        \includegraphics[width=\textwidth]{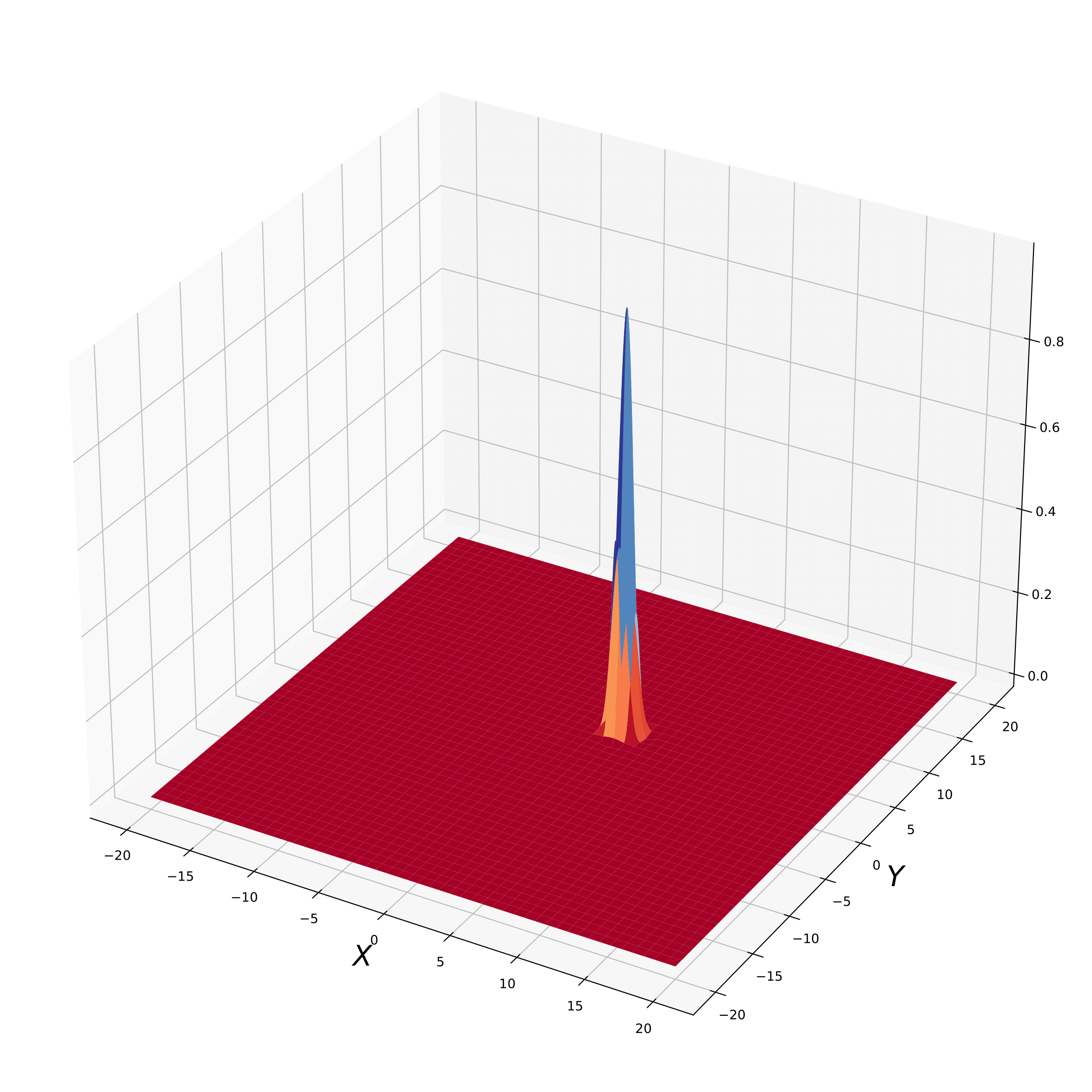}
        \caption{\small  Easom}
    \end{subfigure}
    \begin{subfigure}[b]{0.16\textwidth}
        \includegraphics[width=\textwidth]{figures/eggholder_plot.png}
        \caption{\small  Eggholder}
    \end{subfigure}
    \begin{subfigure}[b]{0.16\textwidth}
        \includegraphics[width=\textwidth]{figures/griewank_plot.png}
        \caption{\small  Griewank}
    \end{subfigure}
    \begin{subfigure}[b]{0.16\textwidth}
        \includegraphics[width=\textwidth]{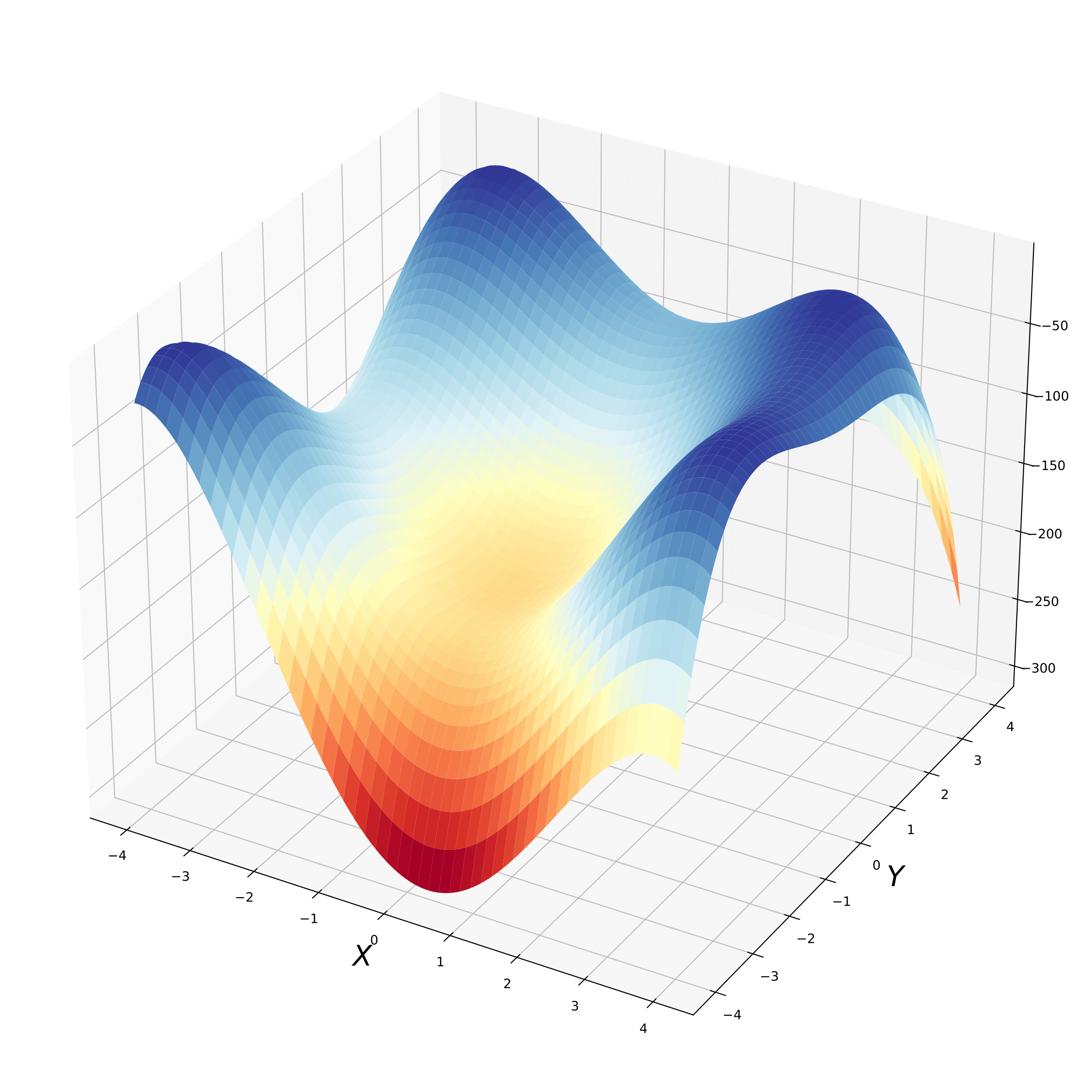}
        \caption{\small  Himmelblau}
    \end{subfigure}
    \begin{subfigure}[b]{0.16\textwidth}
        \includegraphics[width=\textwidth]{figures/holder_plot.png}
        \caption{\small  Holder}
    \end{subfigure}
    \begin{subfigure}[b]{0.16\textwidth}
        \includegraphics[width=\textwidth]{figures/langermann_plot.png}
        \caption{\small  Langermann}
        \end{subfigure}
    \begin{subfigure}[b]{0.16\textwidth}
        \includegraphics[width=\textwidth]{figures/levy_plot.png}
        \caption{\small  Levy}
    \end{subfigure}
     \begin{subfigure}[b]{0.16\textwidth}
        \includegraphics[width=\textwidth]{figures/michalewicz_plot.png}
        \caption{\small  Michalewicz}
    \end{subfigure}
    \begin{subfigure}[b]{0.16\textwidth}
        \includegraphics[width=\textwidth]{figures/rastrigin_plot.png}
        \caption{\small  Rastrigin}
    \end{subfigure}
    \begin{subfigure}[b]{0.16\textwidth}
        \includegraphics[width=\textwidth]{figures/schaffer_plot.png}
        \caption{\small  Schaffer}
    \end{subfigure}
        \begin{subfigure}[b]{0.16\textwidth}
        \includegraphics[width=\textwidth]{figures/schubert_plot.png}
        \caption{\small  Schubert}
    \end{subfigure}  
    \caption{\small  Figures of the various considered non-convex 2-dimensional objective functions.}
    \label{fig:all_figures_app}
\end{figure*}

\begin{figure*}[ht]
    \centering
    \begin{subfigure}[b]{0.45\textwidth}
        \includegraphics[width=\textwidth]{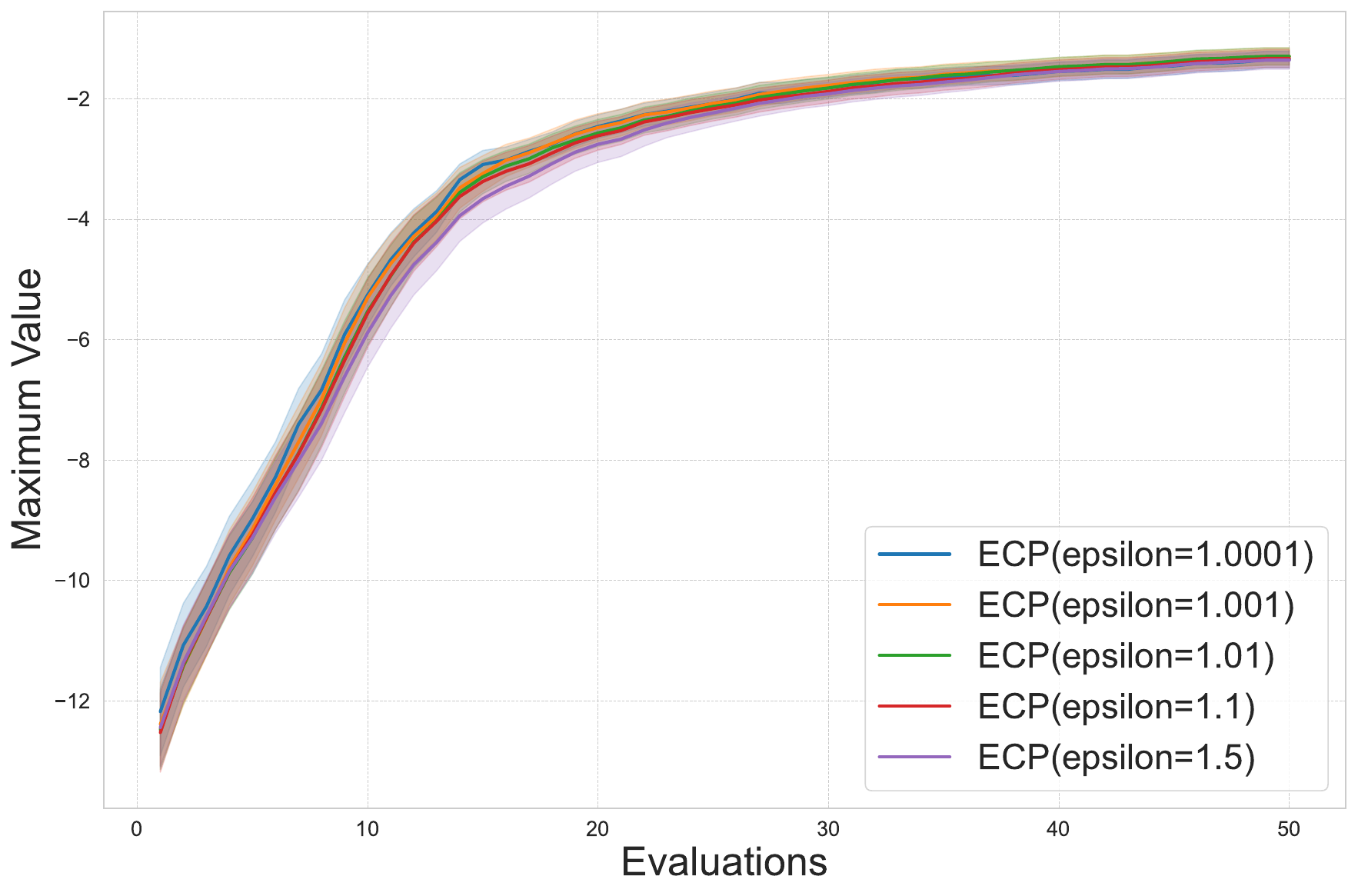}
        \caption{\small  Ackley}
        \label{fig:ackley_plot_epsilon}
    \end{subfigure}
    \hspace{0.6cm} % Add horizontal space
    \begin{subfigure}[b]{0.45\textwidth}
        \includegraphics[width=\textwidth]{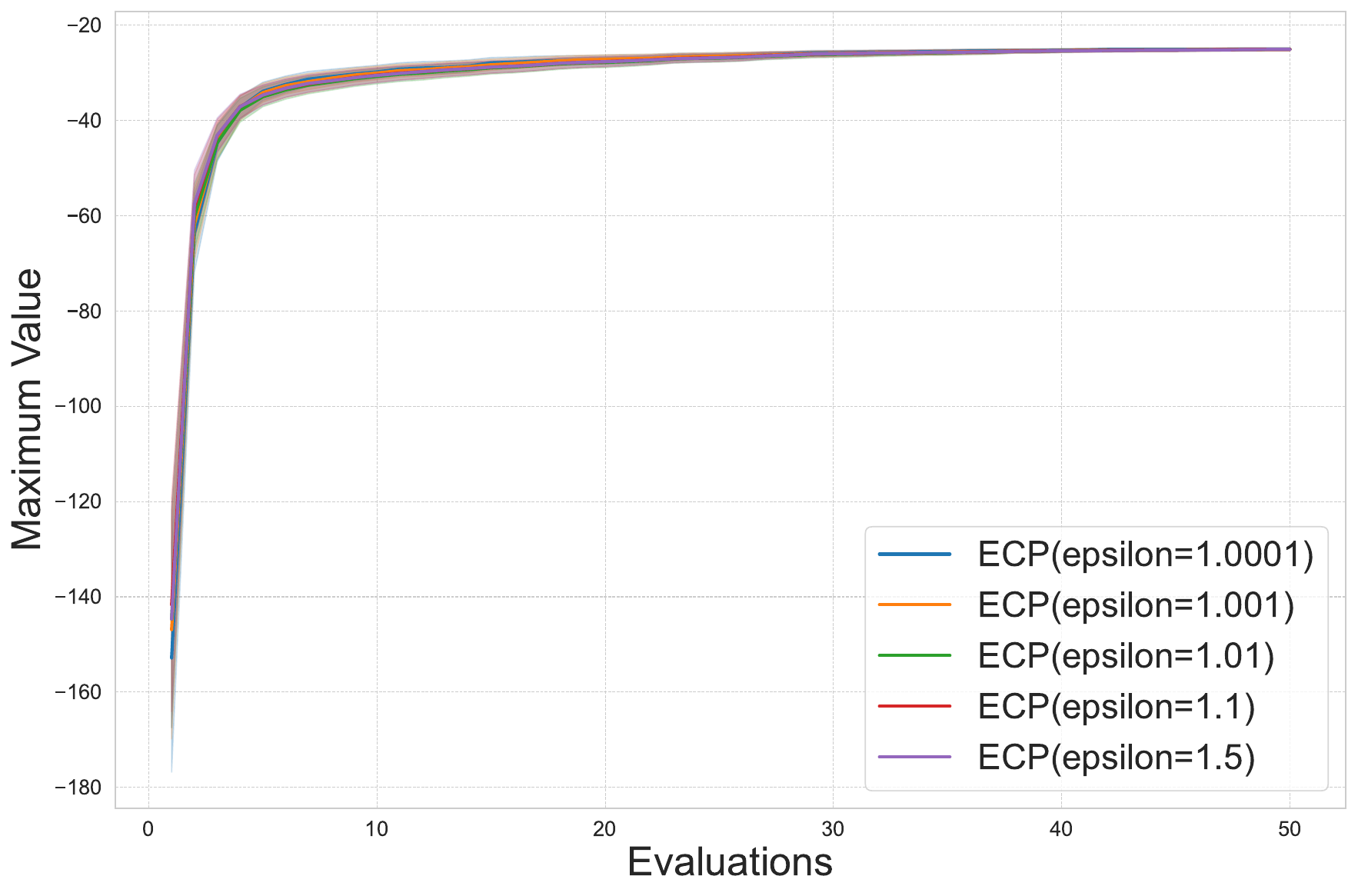}
        \caption{\small  AutoMPG (real-world)}
        \label{fig:autompg_epsilon}
    \end{subfigure}

    \begin{subfigure}[b]{0.45\textwidth}
        \includegraphics[width=\textwidth]{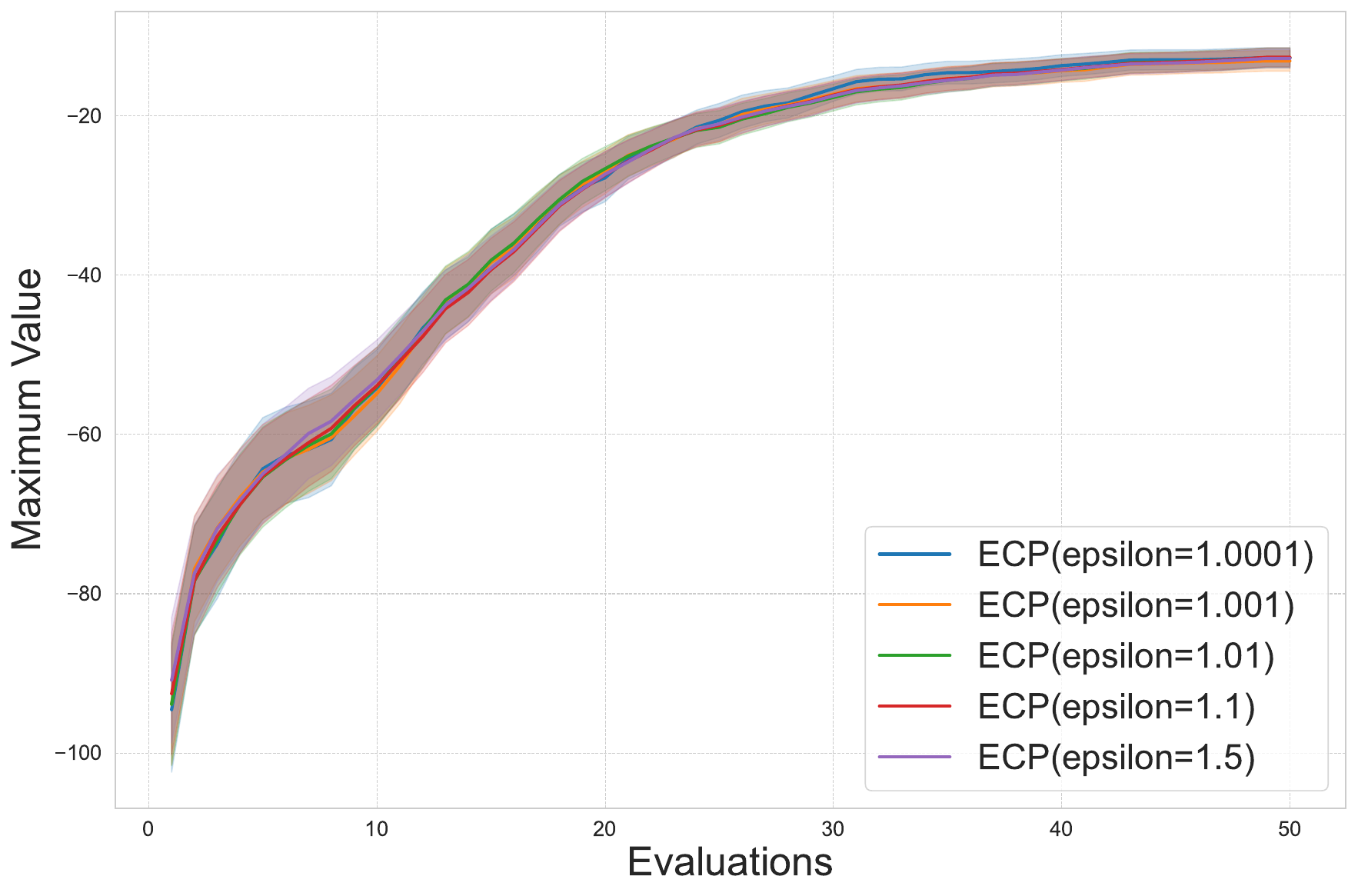}
        \caption{\small  Bukin 2D}
        \label{fig:bukin_epsilon}
    \end{subfigure}
    \hspace{0.6cm} % Add horizontal space
    \begin{subfigure}[b]{0.45\textwidth}
        \includegraphics[width=\textwidth]{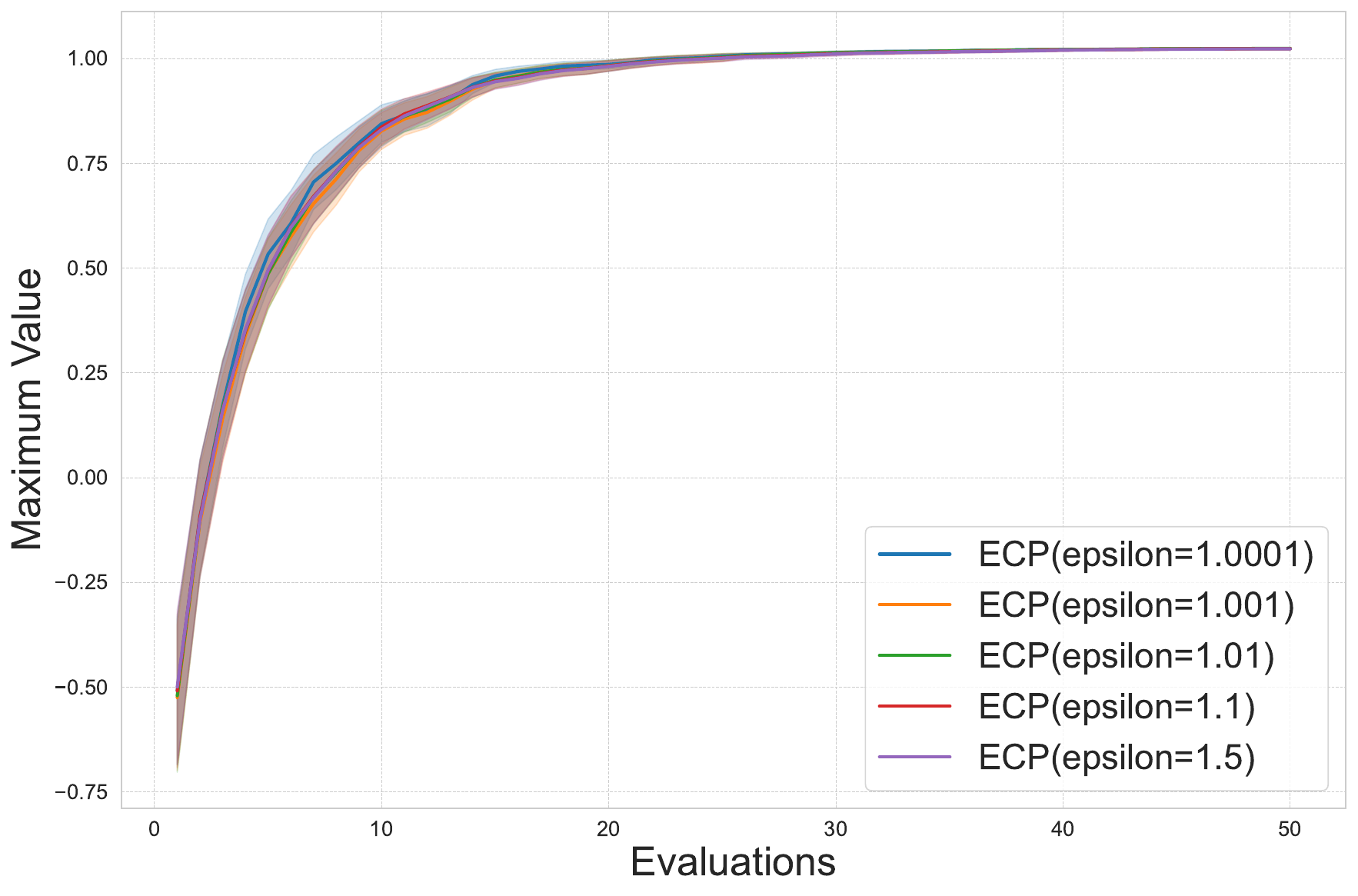}
        \caption{\small  Camel 2D}
        \label{fig:camel_epsilon}
    \end{subfigure}

    \begin{subfigure}[b]{0.45\textwidth}
        \includegraphics[width=\textwidth]{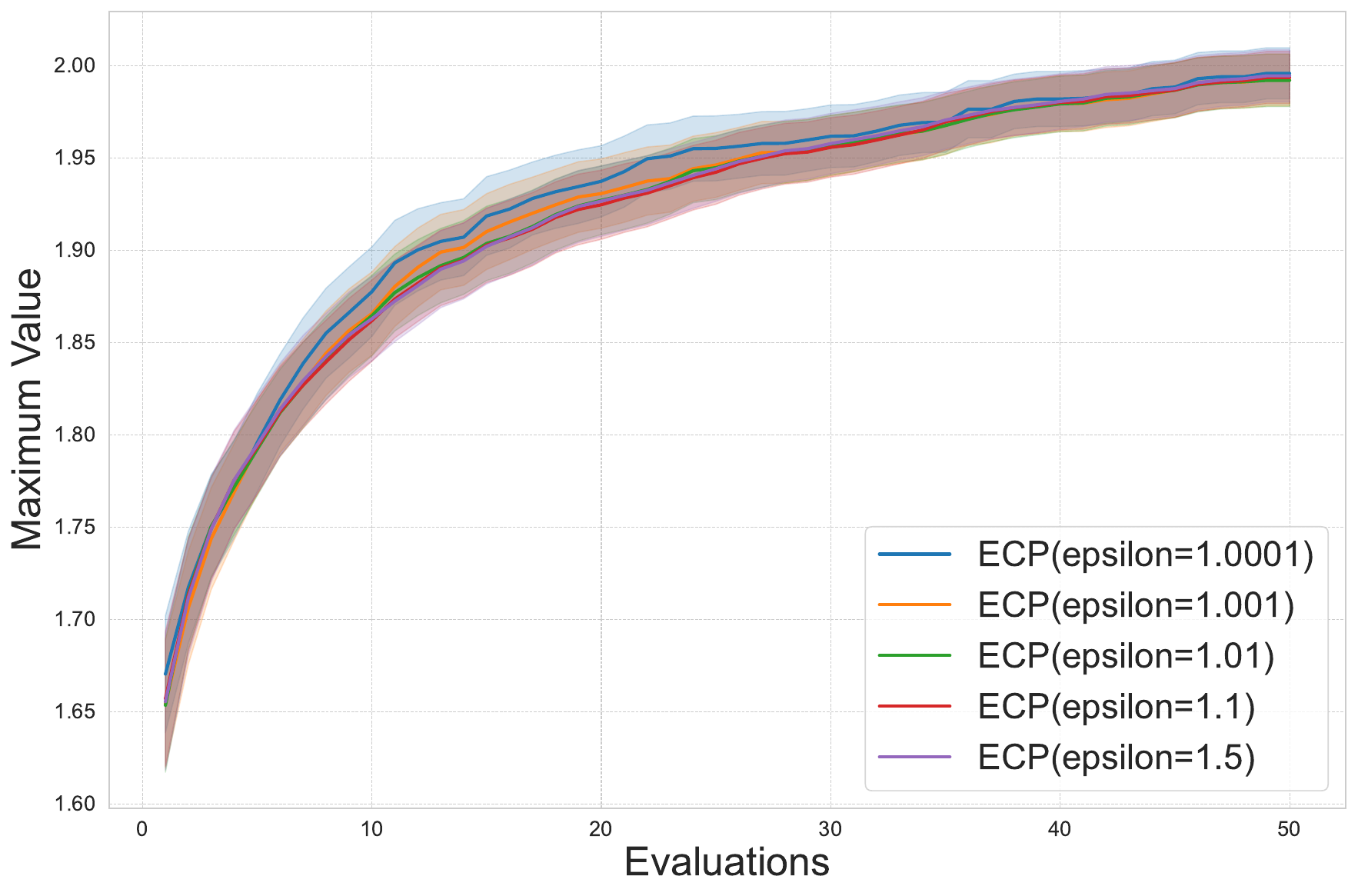}
        \caption{\small  Crossintray 2D}
        \label{fig:Crossintray_epsilon}
    \end{subfigure}
    \hspace{0.6cm} % Add horizontal space
    \begin{subfigure}[b]{0.45\textwidth}
        \includegraphics[width=\textwidth]{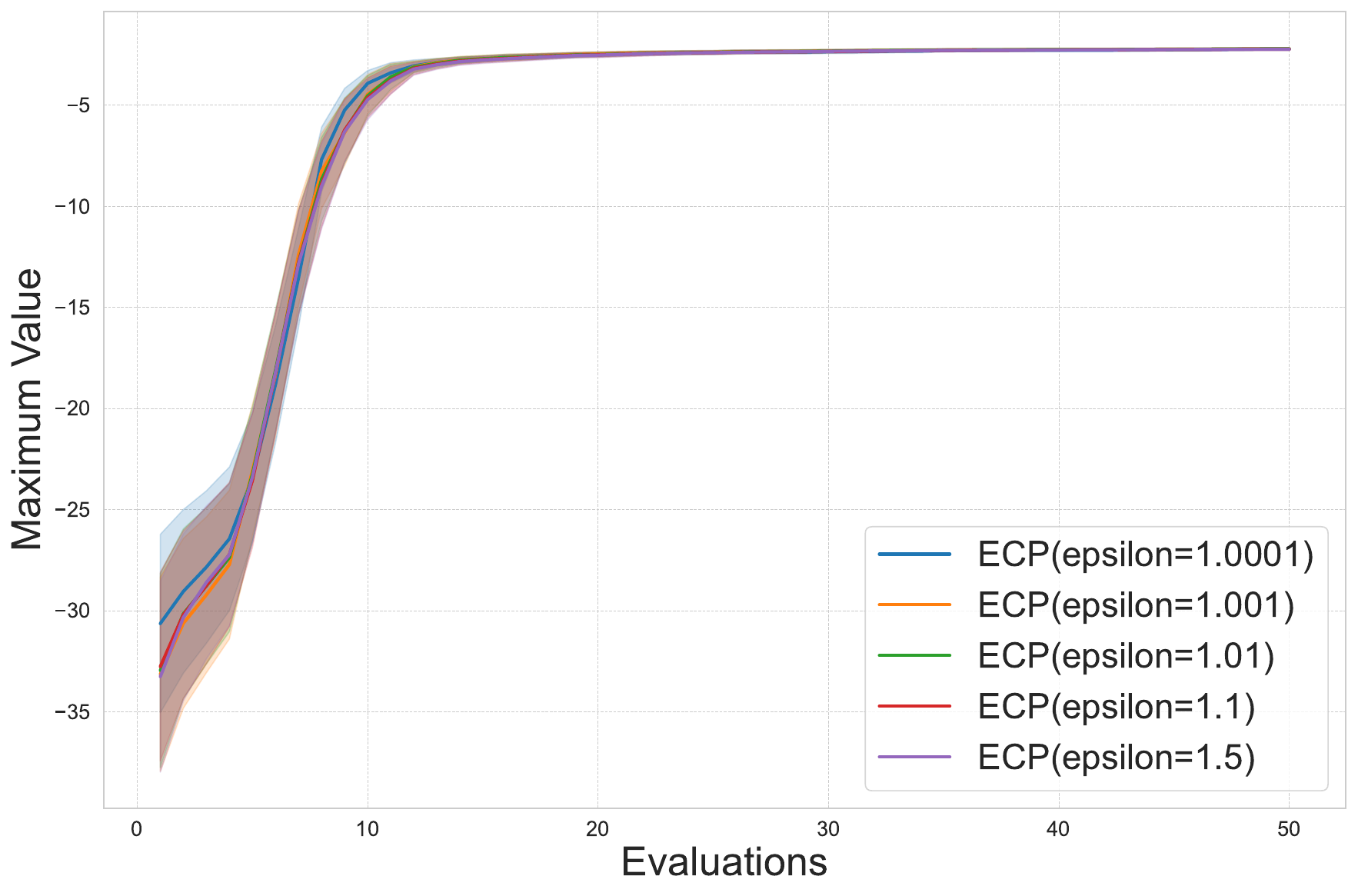}
        \caption{\small  Damavandi 2D}
        \label{fig:damavandi_epsilon}
    \end{subfigure}

    \begin{subfigure}[b]{0.45\textwidth}
        \includegraphics[width=\textwidth]{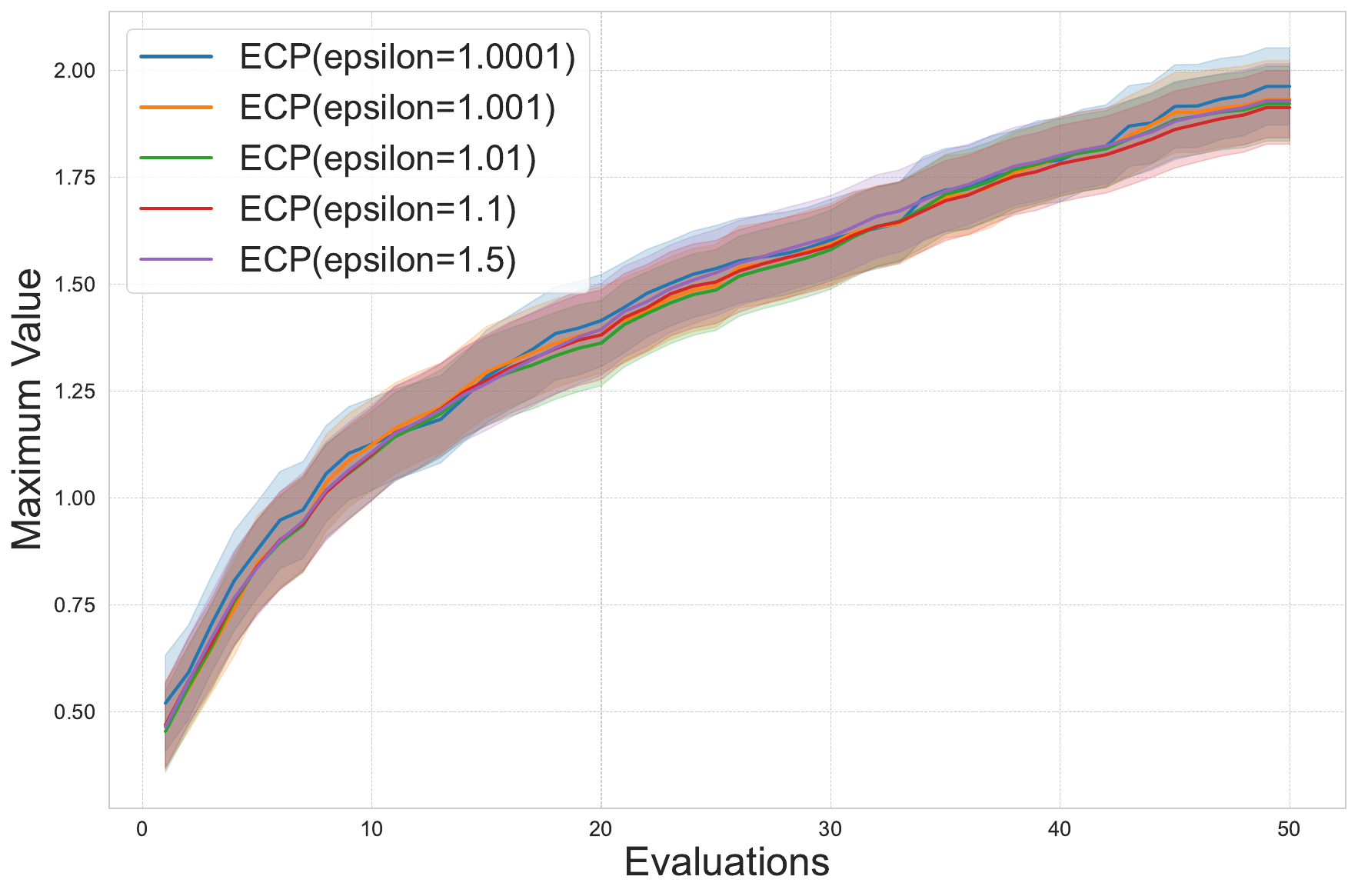}
        \caption{\small  Hartmann 6D}
        \label{fig:hartmann6_epsilon}
    \end{subfigure}
    \hspace{0.6cm} % Add horizontal space
    \begin{subfigure}[b]{0.45\textwidth}
        \includegraphics[width=\textwidth]{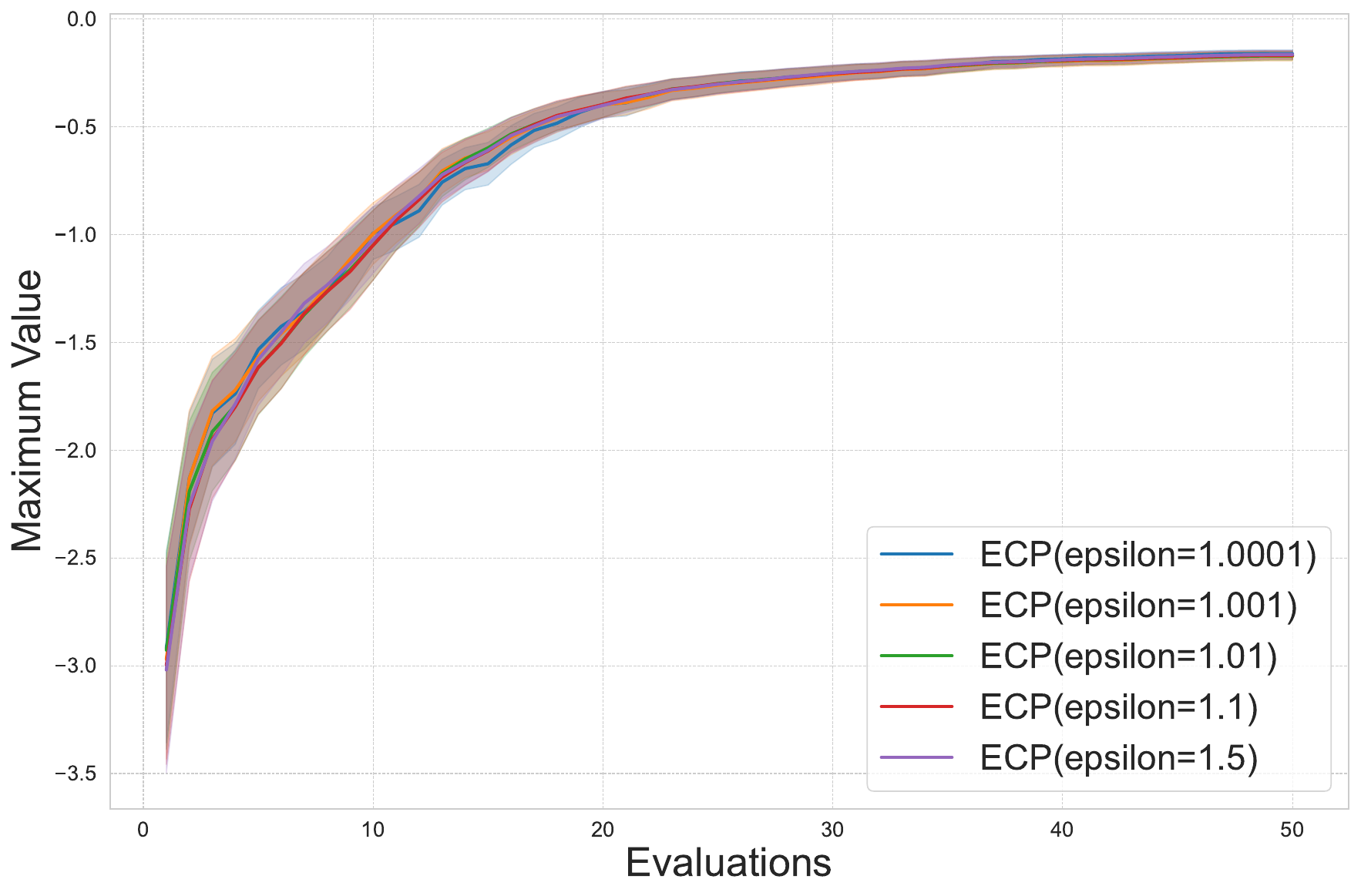}
        \caption{\small  Rosenbrock 3D}
        \label{fig:rosenbrock_epsilon}
    \end{subfigure}

    \caption{\small  Ablation Study on the Constant $\varepsilon_1 > 0$ of ECP with fixed $C=10^{3}$ and $\tau=10^{-3}$ on various real-world and synthetic non-convex multi-dimensional optimization problems.}
    \label{fig:all_figures_ablation_epsilon}
\end{figure*}

\begin{figure*}[ht]
    \centering
    \begin{subfigure}[b]{0.45\textwidth}
        \includegraphics[width=\textwidth]{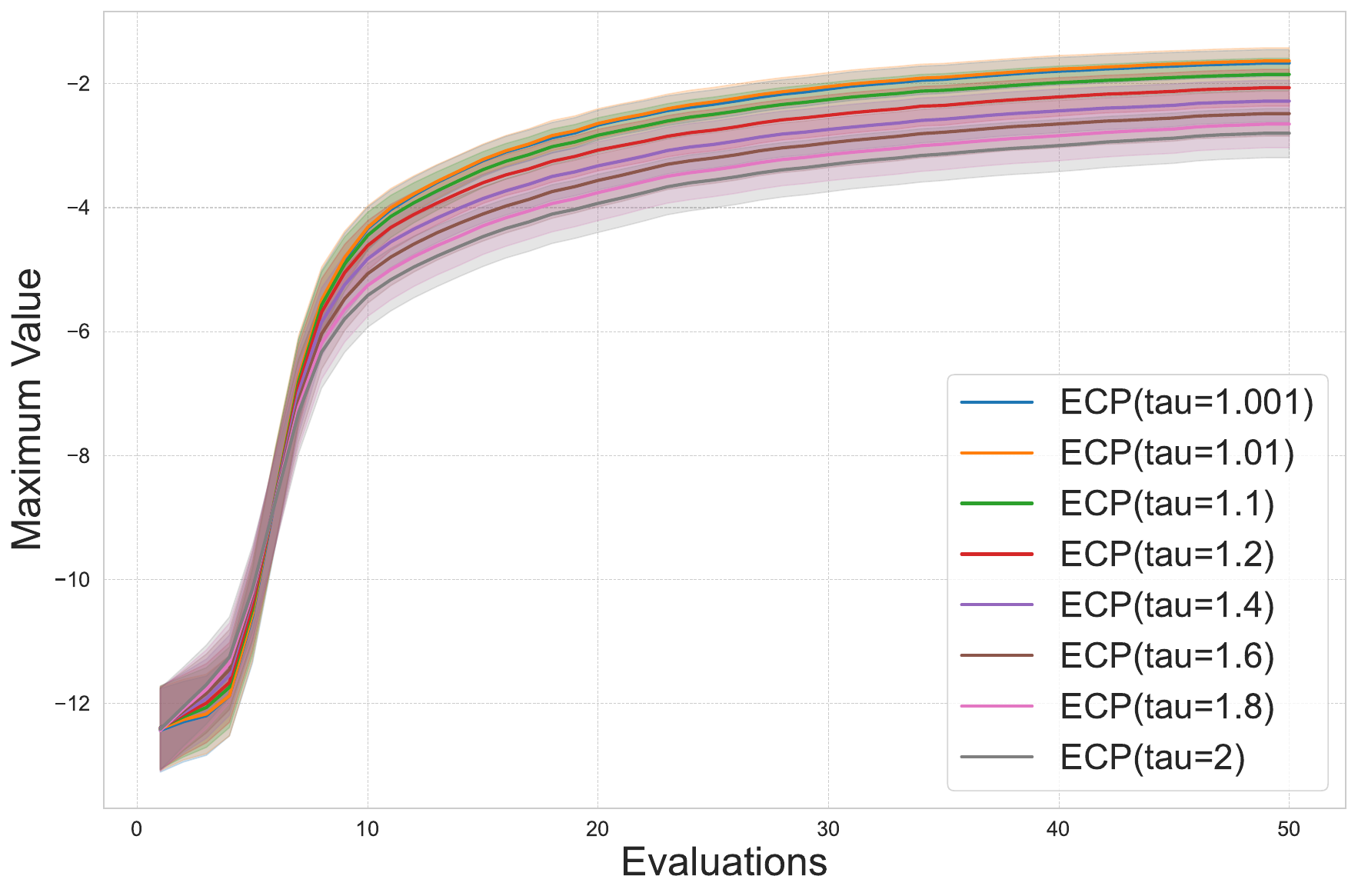}
        \caption{\small  Ackley}
        \label{fig:ackley_plot_tau}
    \end{subfigure}
    \hspace{0.6cm} % Add horizontal space
    \begin{subfigure}[b]{0.45\textwidth}
        \includegraphics[width=\textwidth]{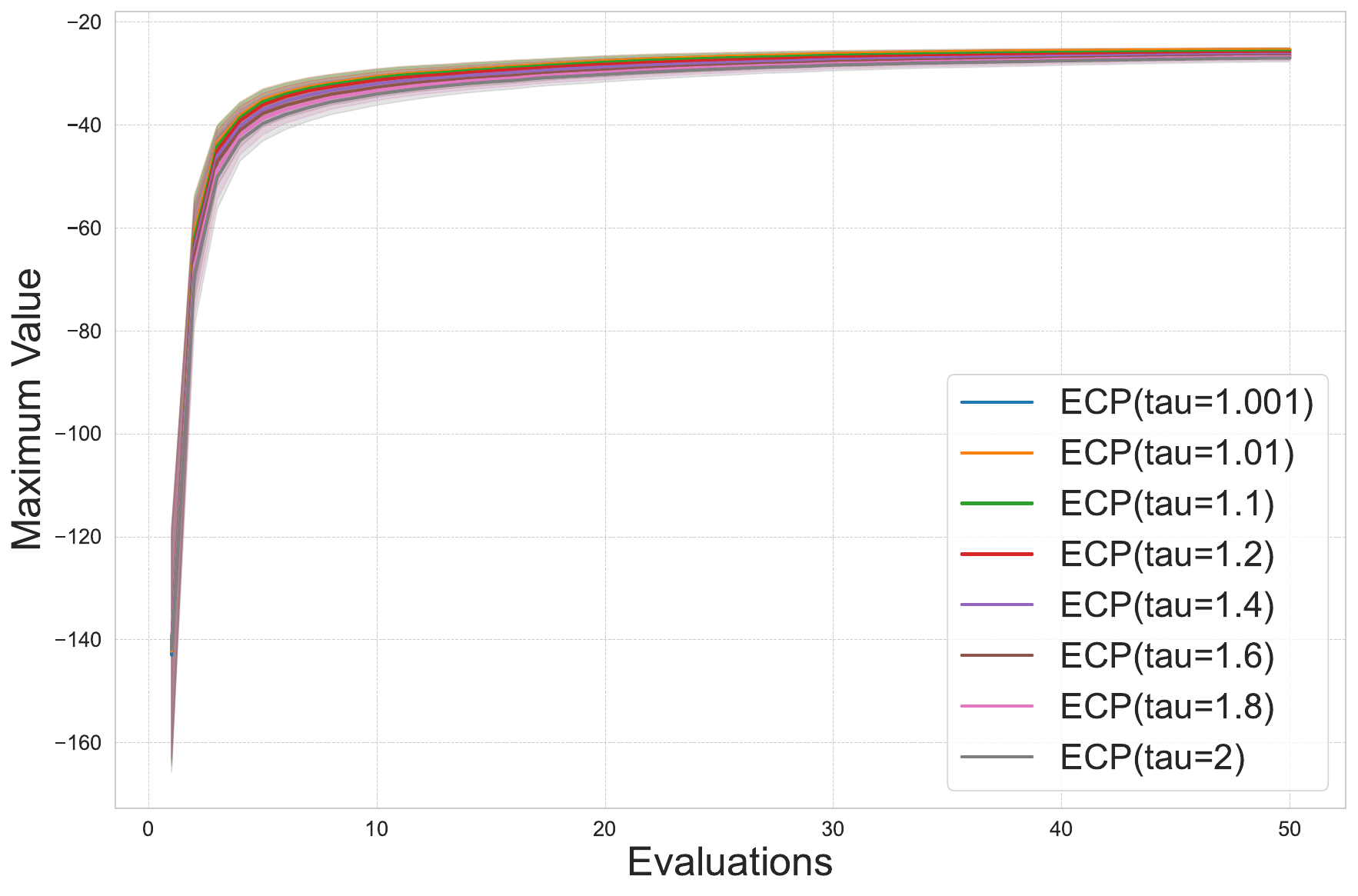}
        \caption{\small  AutoMPG (real-world)}
        \label{fig:autompg_tau}
    \end{subfigure}

    \begin{subfigure}[b]{0.45\textwidth}
        \includegraphics[width=\textwidth]{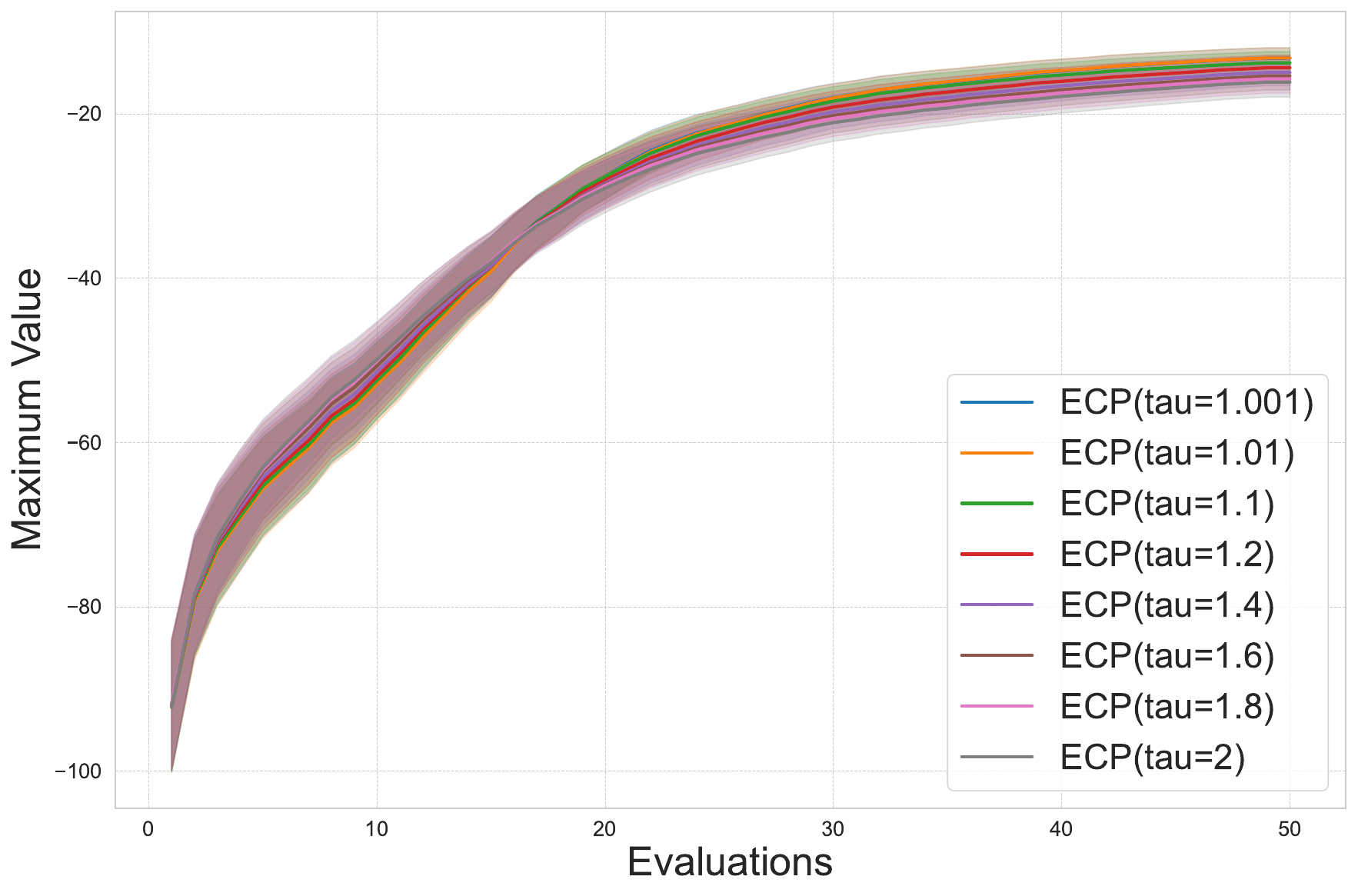}
        \caption{\small  Bukin 2D}
        \label{fig:tau_bukin}
    \end{subfigure}
    \hspace{0.6cm} % Add horizontal space
    \begin{subfigure}[b]{0.45\textwidth}
        \includegraphics[width=\textwidth]{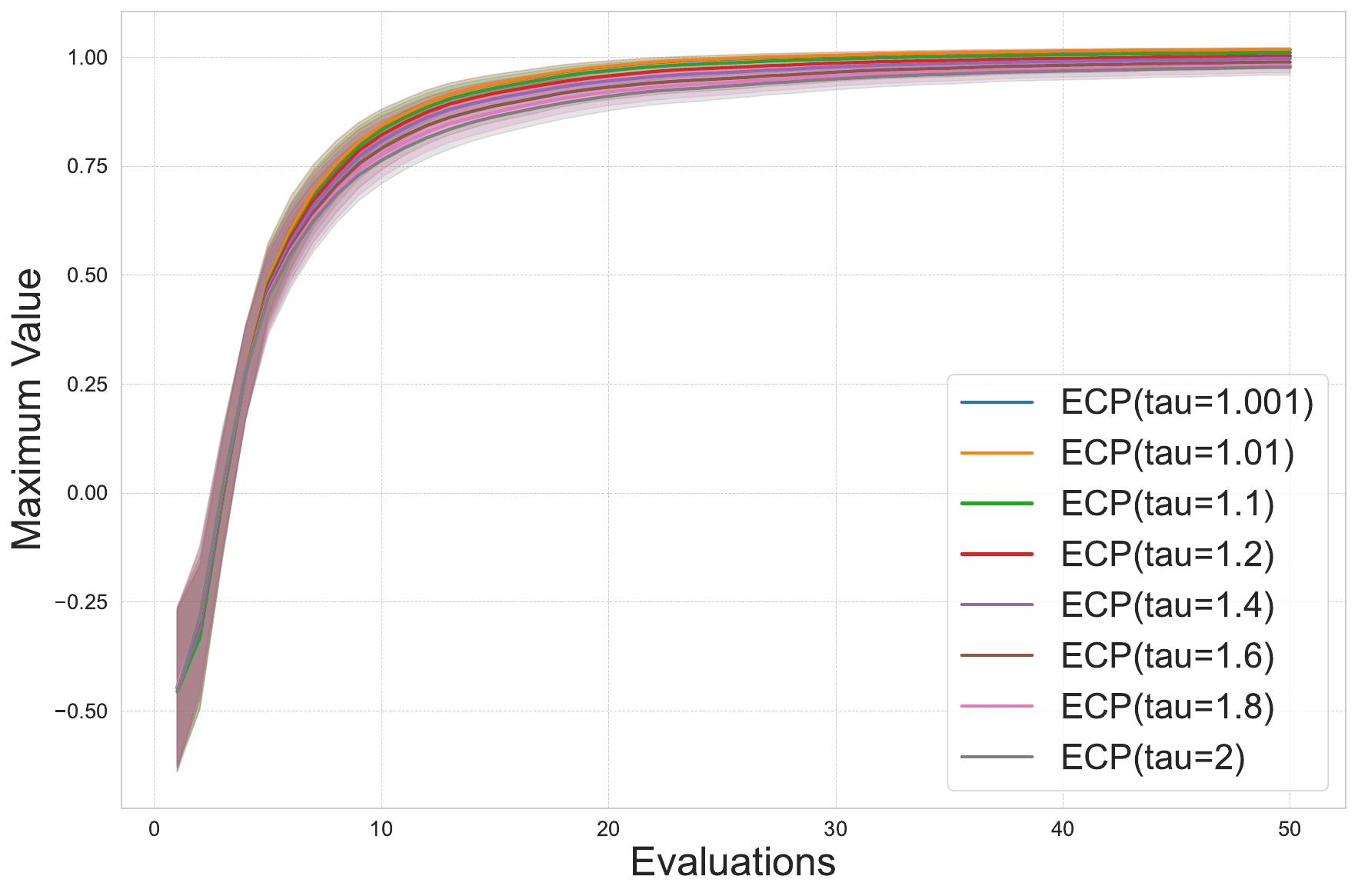}
        \caption{\small  Camel 2D}
        \label{fig:camel_tau}
    \end{subfigure}

    \begin{subfigure}[b]{0.45\textwidth}
        \includegraphics[width=\textwidth]{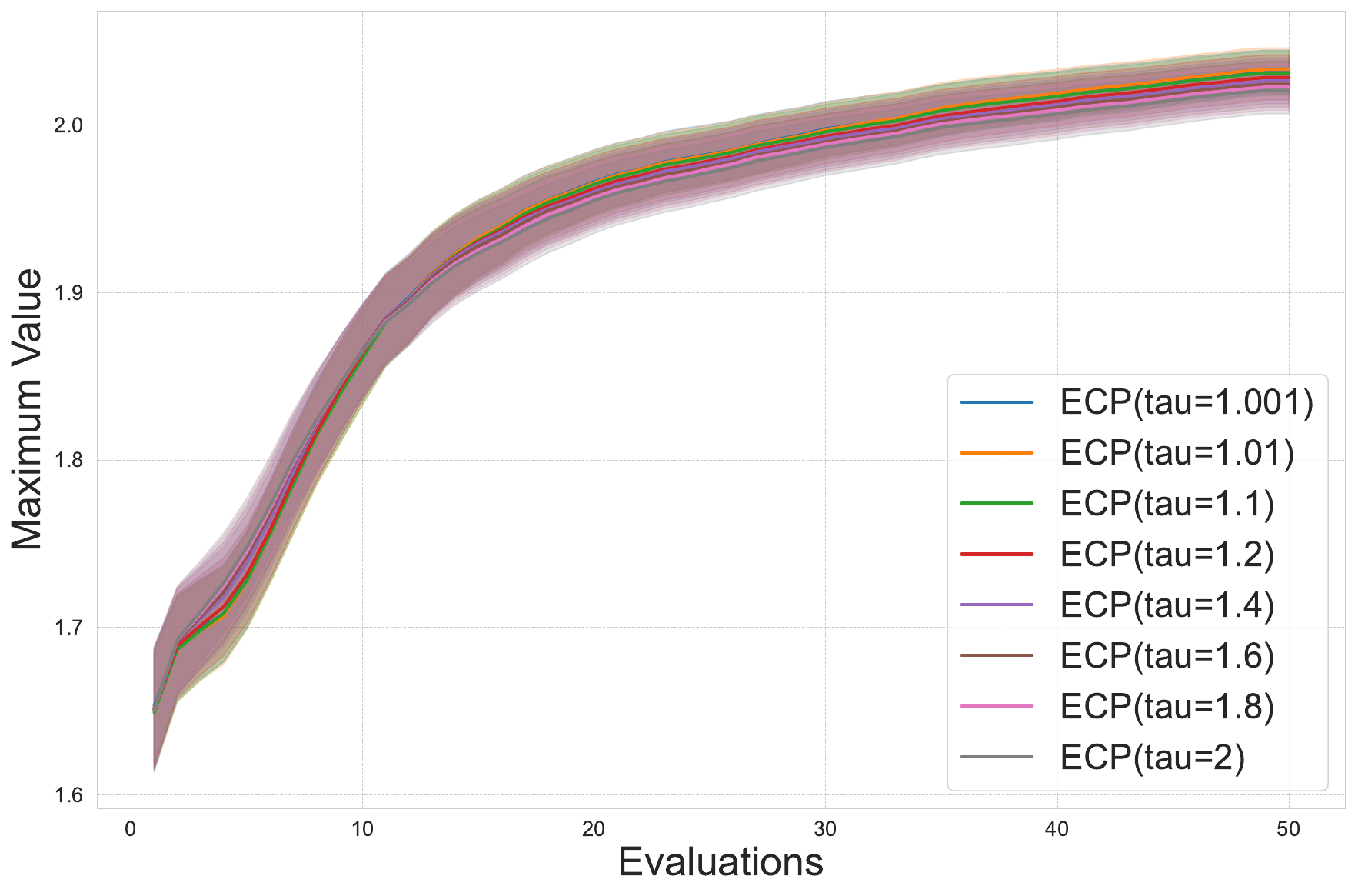}
        \caption{\small  Crossintray 2D}
        \label{fig:Crossintray_tau}
    \end{subfigure}
    \hspace{0.6cm} % Add horizontal space
    \begin{subfigure}[b]{0.45\textwidth}
        \includegraphics[width=\textwidth]{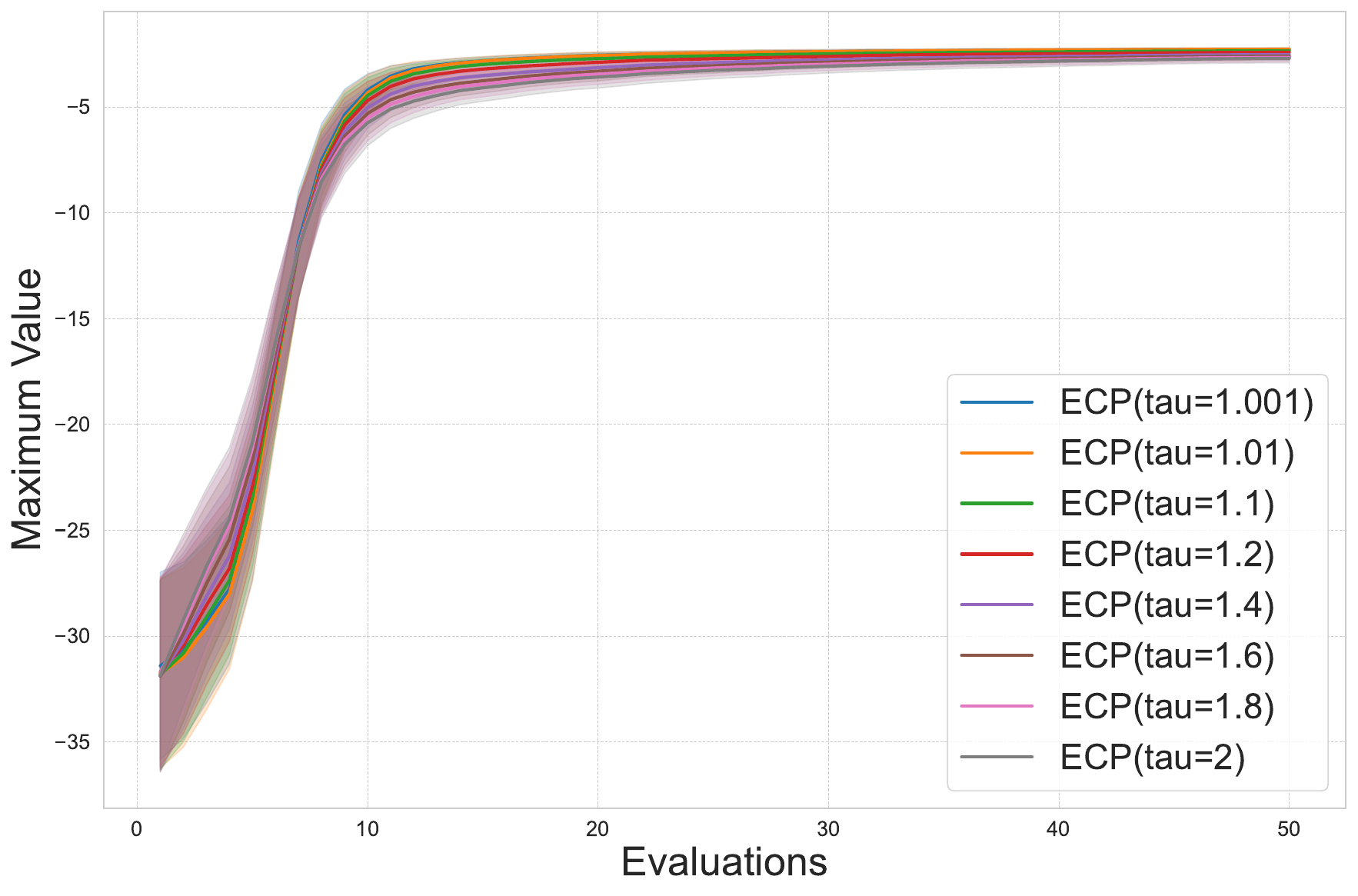}
        \caption{\small  Damavandi 2D}
        \label{fig:damavandi_tau}
    \end{subfigure}

    \begin{subfigure}[b]{0.45\textwidth}
        \includegraphics[width=\textwidth]{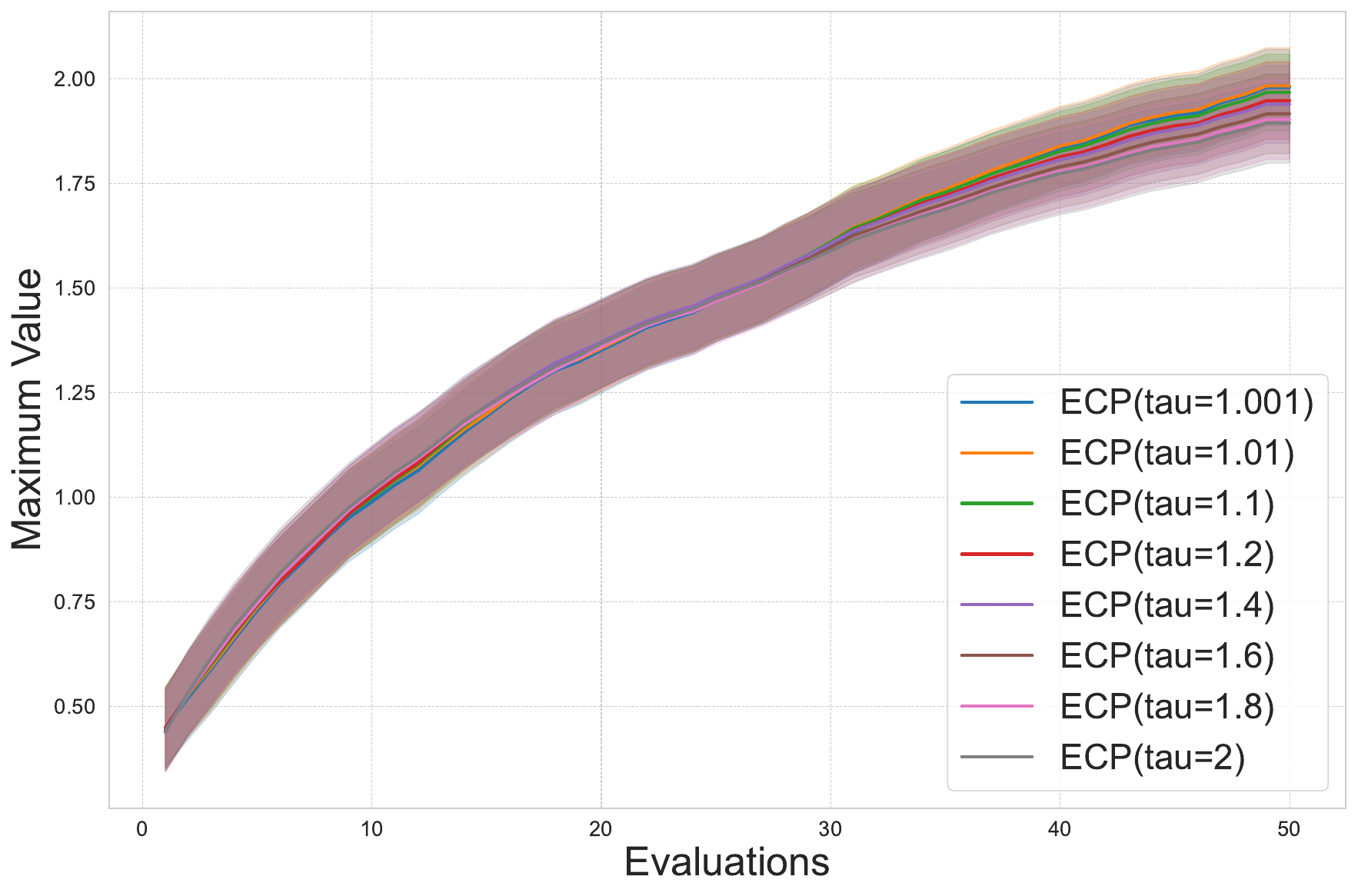}
        \caption{\small  Hartmann 6D}
        \label{fig:hartmann6_tau}
    \end{subfigure}
    \hspace{0.6cm} % Add horizontal space
    \begin{subfigure}[b]{0.45\textwidth}
        \includegraphics[width=\textwidth]{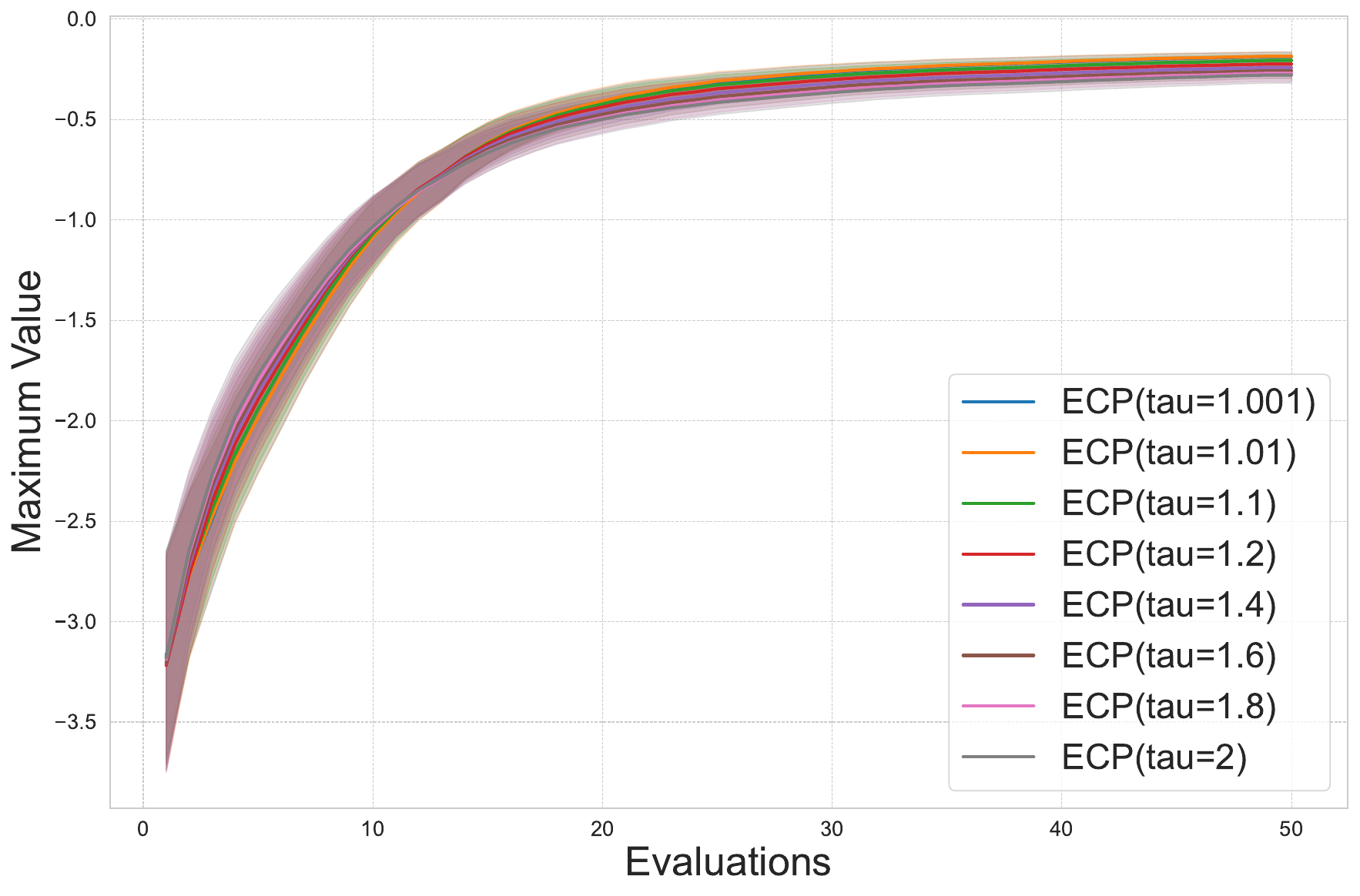}
        \caption{\small  Rosenbrock 3D}
        \label{fig:rosenbrock_tau}
    \end{subfigure}

    \caption{\small  Ablation Study on the Constant $\tau > 1$ of ECP with fixed $C=10^{3}$ and $\varepsilon_1=10^{-2}$ on various real-world and synthetic non-convex multi-dimensional optimization problems.}
    \label{fig:all_figures_ablation_tau}
\end{figure*}

\begin{figure*}[ht]
    \centering
    \begin{subfigure}[b]{0.45\textwidth}
        \includegraphics[width=\textwidth]{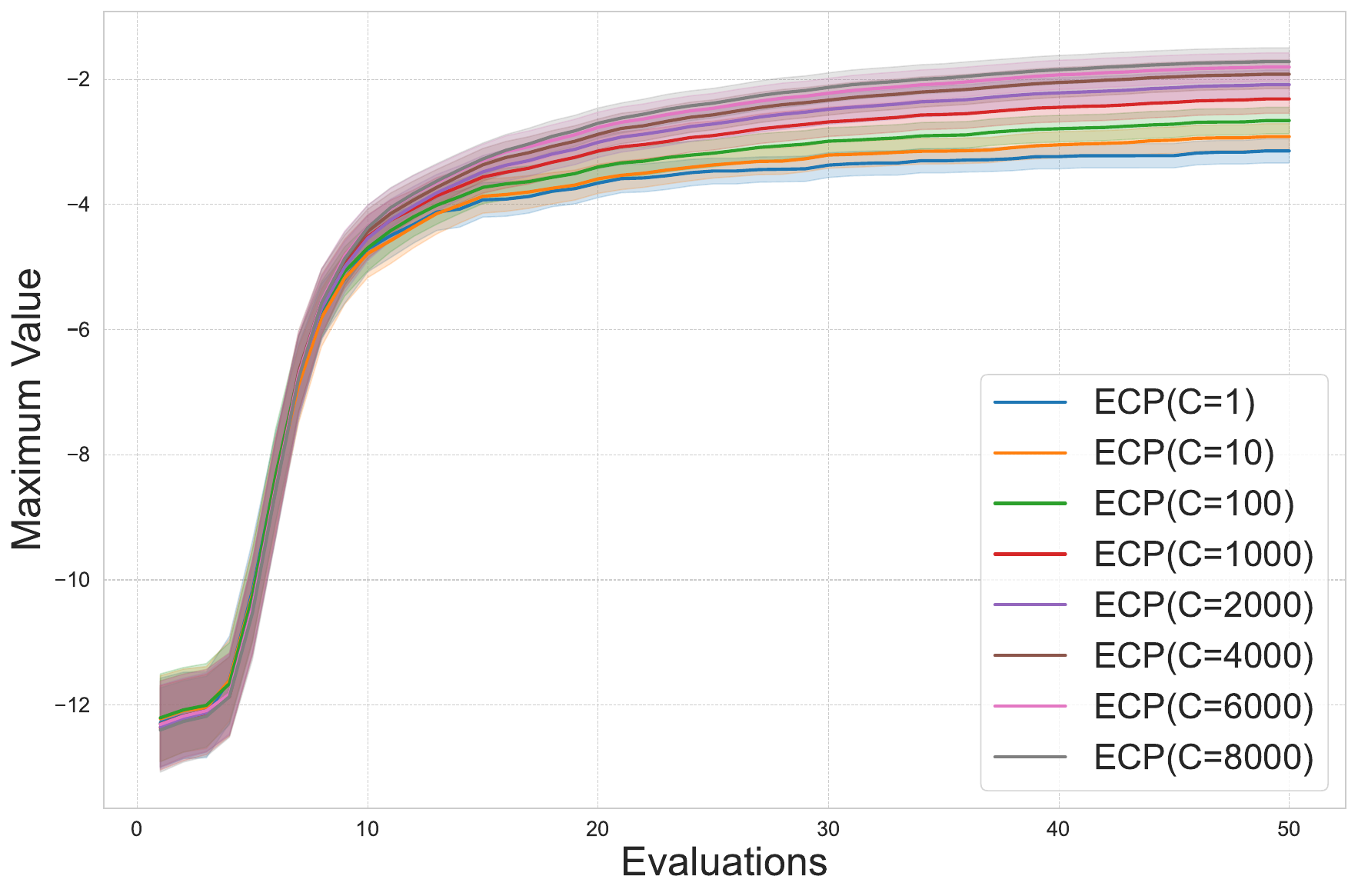}
        \caption{\small  Ackley}
        \label{fig:ackley_plot_c}
    \end{subfigure}
    \hspace{0.6cm} % Add horizontal space
    \begin{subfigure}[b]{0.45\textwidth}
        \includegraphics[width=\textwidth]{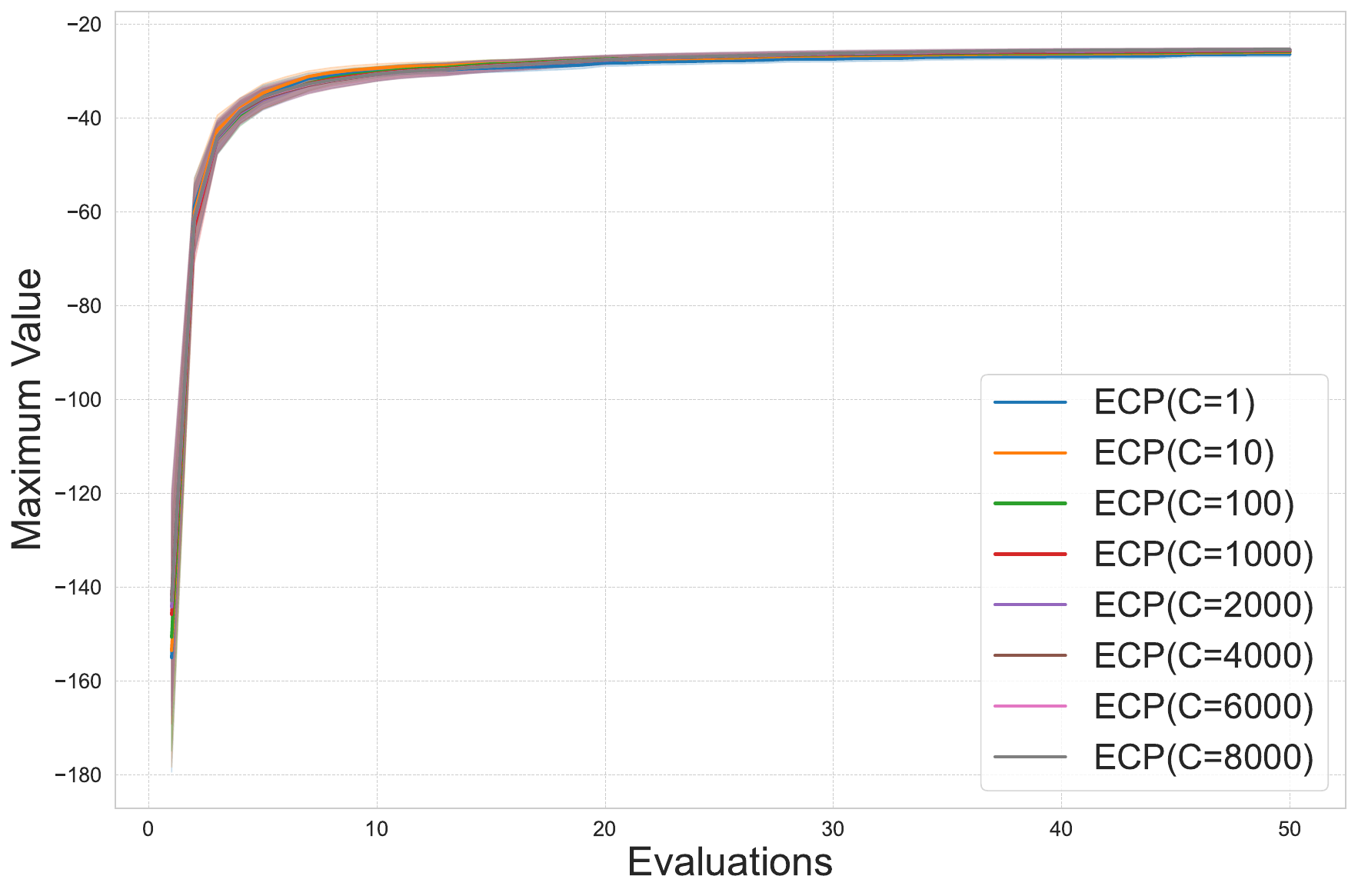}
        \caption{\small  AutoMPG (real-world)}
        \label{fig:autompg_c}
    \end{subfigure}

    % \vspace{0.1cm} % Optional vertical space if needed

    \begin{subfigure}[b]{0.45\textwidth}
        \includegraphics[width=\textwidth]{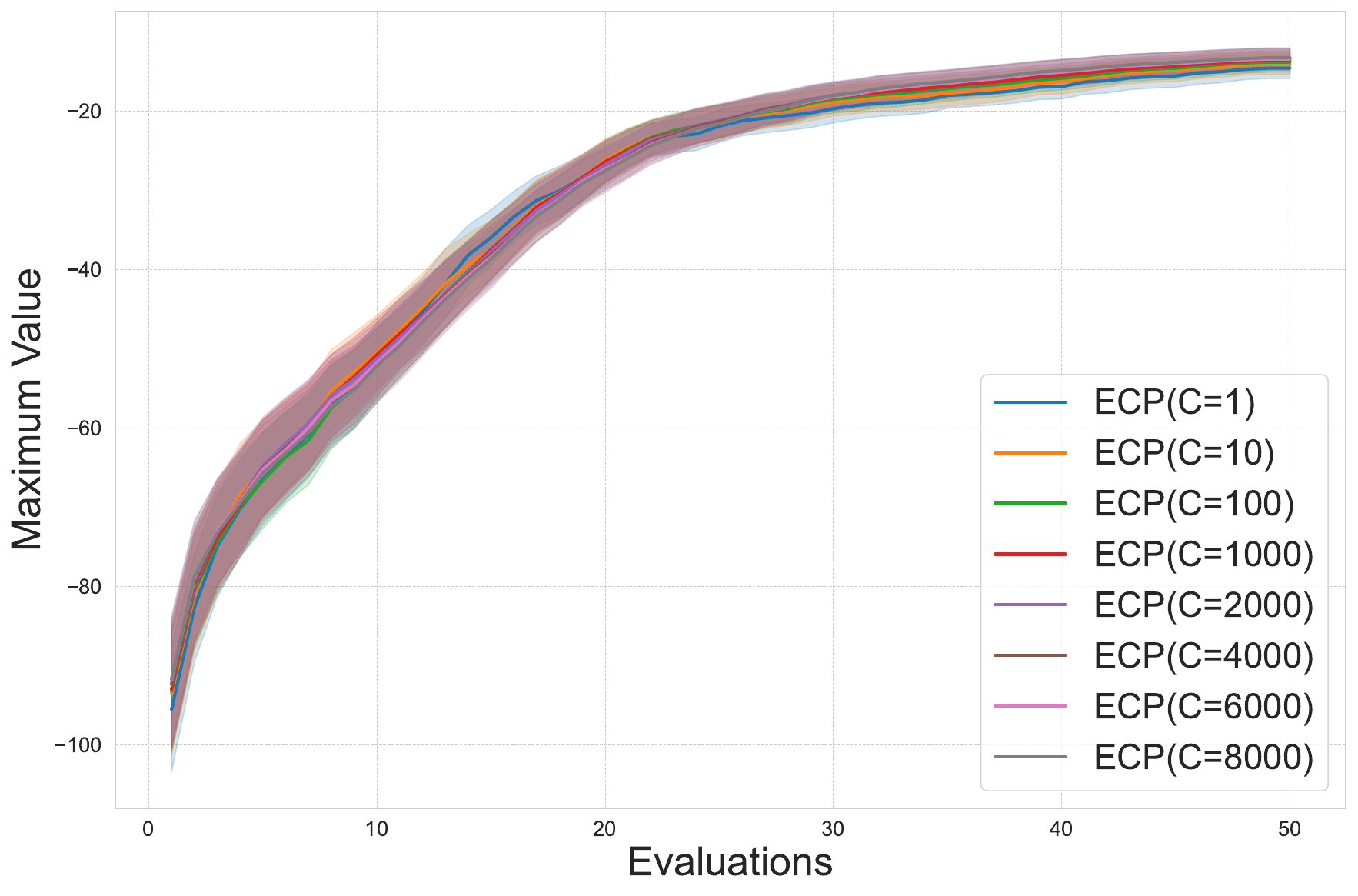}
        \caption{\small  Bukin 2D}
        \label{fig:c_bukin}
    \end{subfigure}
    \hspace{0.6cm} % Add horizontal space
    \begin{subfigure}[b]{0.45\textwidth}
        \includegraphics[width=\textwidth]{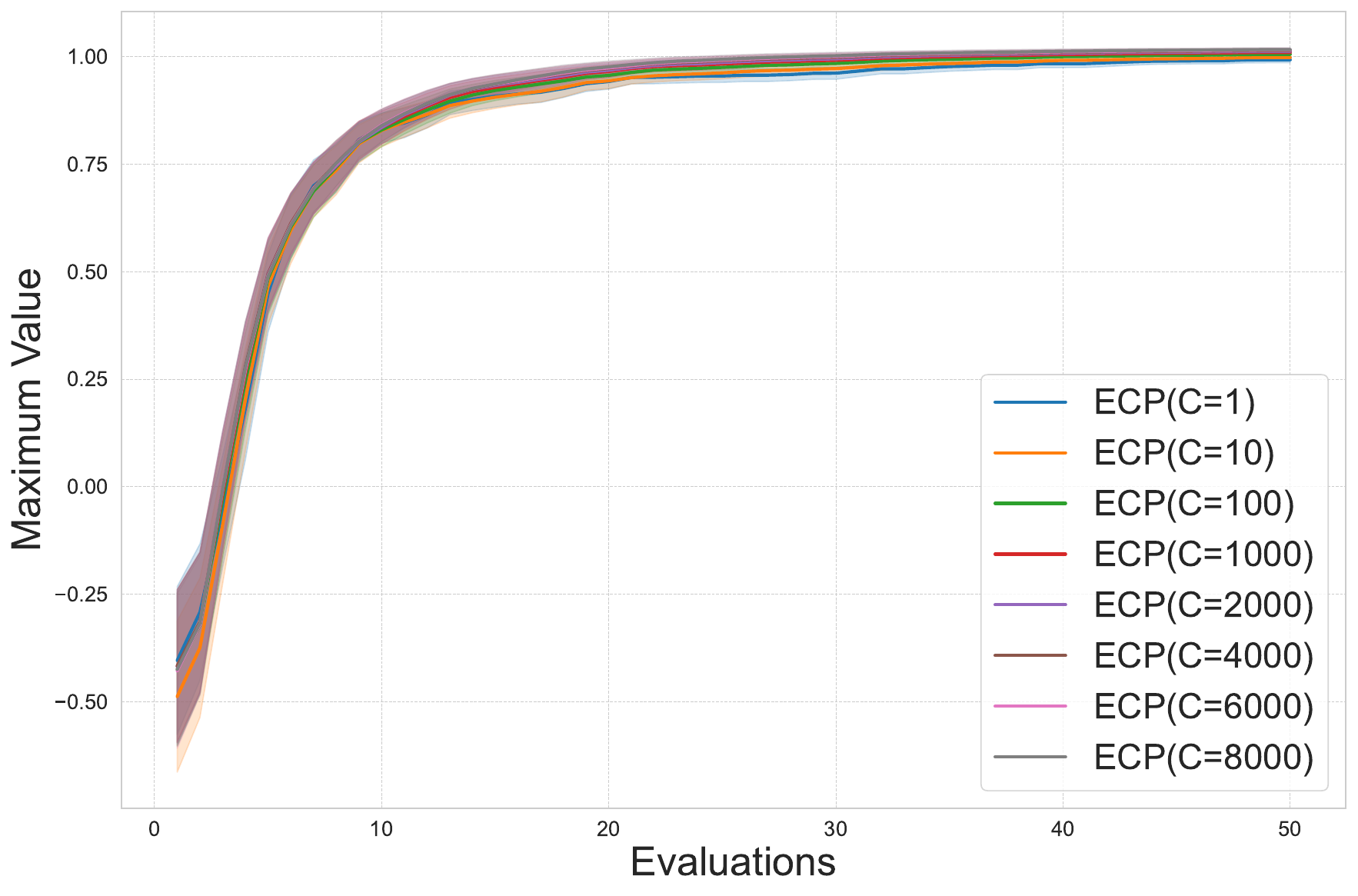}
        \caption{\small  Camel 2D}
        \label{fig:camel_c}
    \end{subfigure}
    
    % \vspace{0.1cm} % Optional vertical space if needed

    \begin{subfigure}[b]{0.45\textwidth}
        \includegraphics[width=\textwidth]{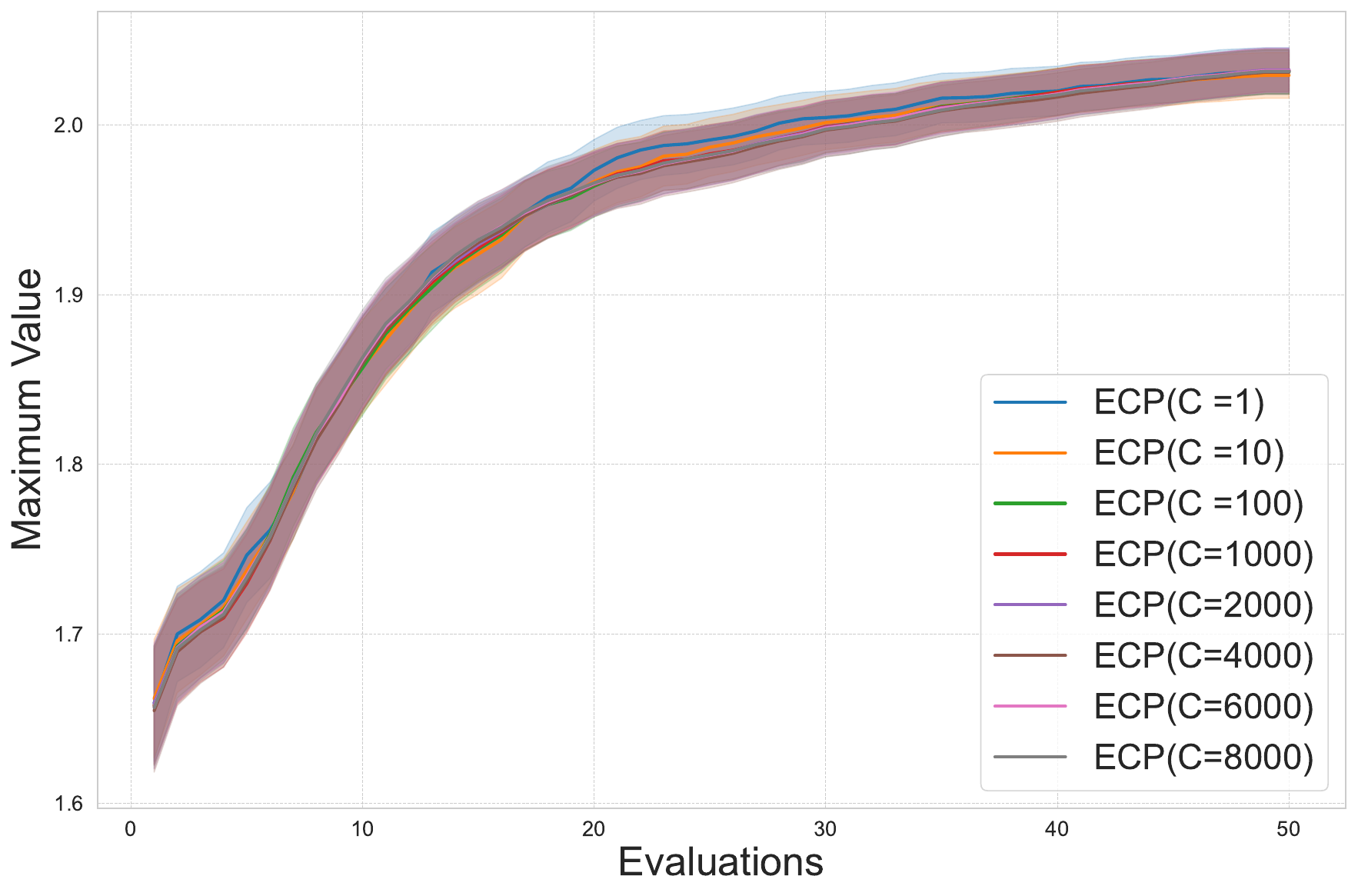}
        \caption{\small  Crossintray 2D}
        \label{fig:Crossintray_c}
    \end{subfigure}
    \hspace{0.6cm} % Add horizontal space
    \begin{subfigure}[b]{0.45\textwidth}
        \includegraphics[width=\textwidth]{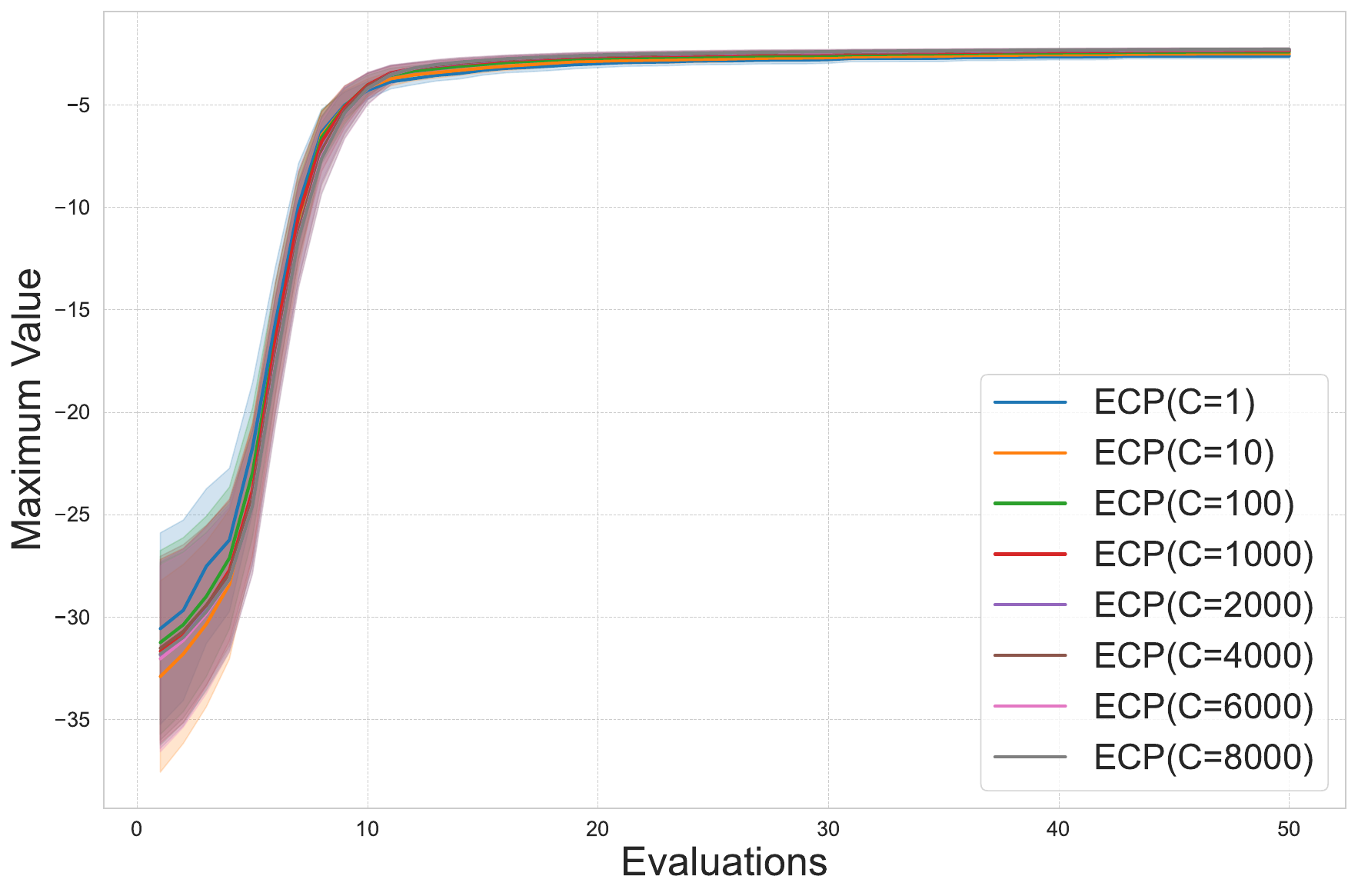}
        \caption{\small  Damavandi 2D}
        \label{fig:damavandi_c}
    \end{subfigure}
    
    % \vspace{0.1cm} % Optional vertical space if needed

    \begin{subfigure}[b]{0.45\textwidth}
        \includegraphics[width=\textwidth]{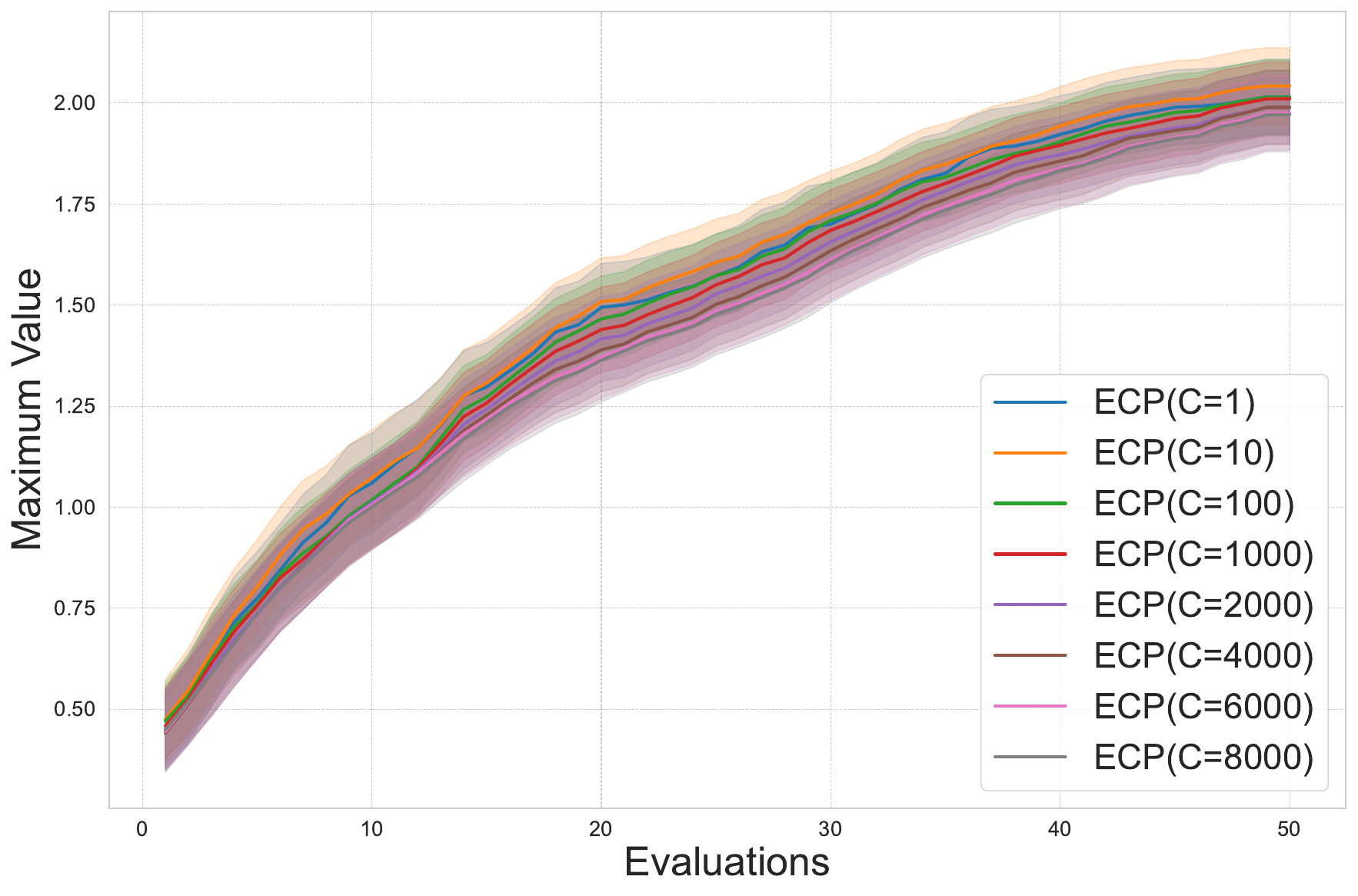}
        \caption{\small  Hartmann 6D}
        \label{fig:hartmann6_c}
    \end{subfigure}
    \hspace{0.6cm} % Add horizontal space
    \begin{subfigure}[b]{0.45\textwidth}
        \includegraphics[width=\textwidth]{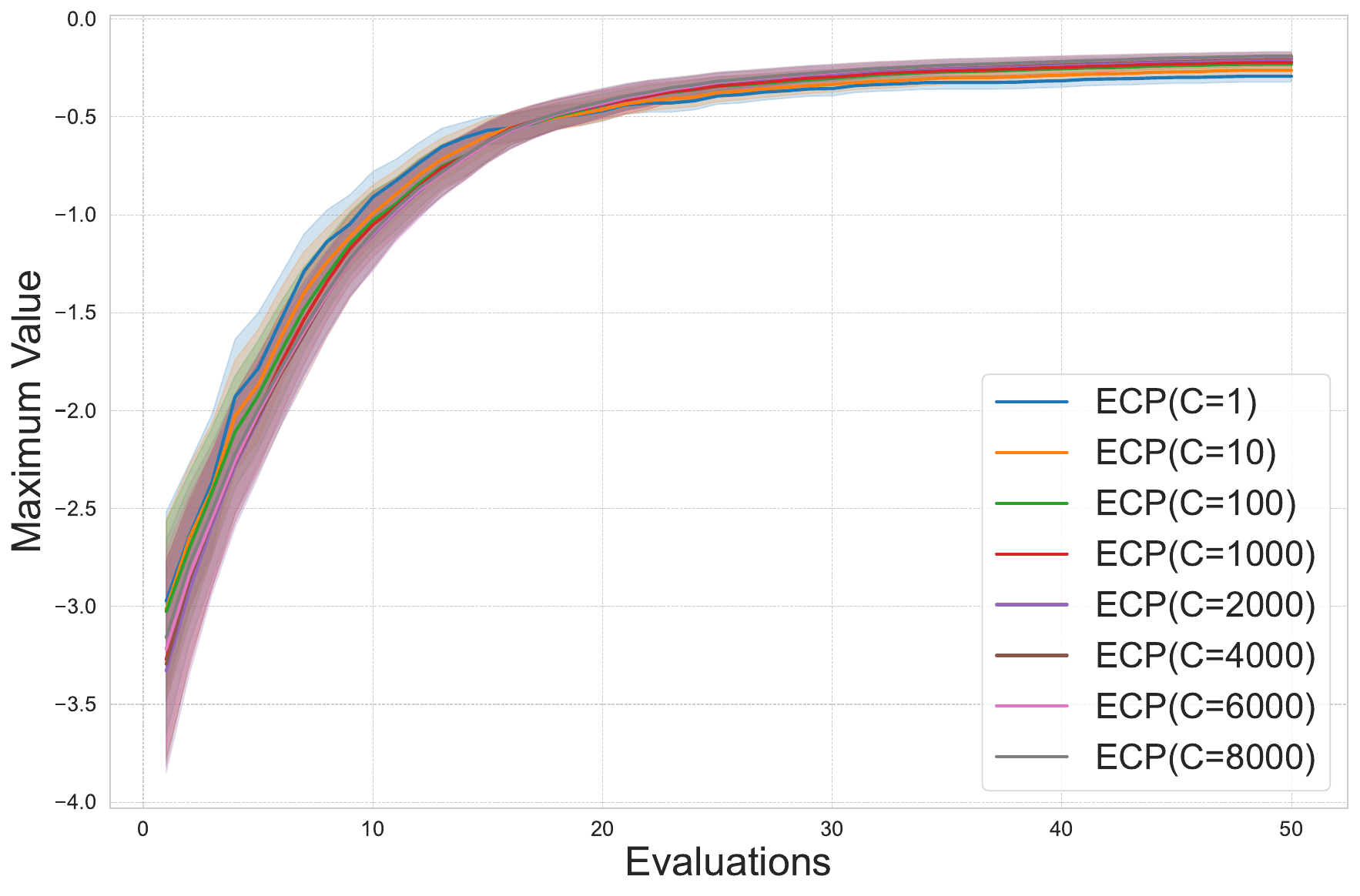}
        \caption{\small  Rosenbrock 3D}
        \label{fig:rosenbrock_c}
    \end{subfigure}

    \caption{\small  Ablation Study on the constant $C > 1$ of ECP with fixed $\tau=10^{-3}$ and $\varepsilon_1=10^{-2}$ on various real-world and synthetic non-convex multi-dimensional optimization problems.}
    \label{fig:all_figures_ablation_C}
\end{figure*}

\newpage
\section{Details of the Experiments}
\label{app:experiments_details}

The implementations of ECP and the considered objectives are publicly available at \href{https://github.com/fouratifares/ECP}{\texttt{https://github.com/fouratifares/ECP}}.

\subsection{Optimization Algorithms}

For DIRECT and Dual Annealing, we use the implementations from SciPy \citep{virtanen2020scipy}, with standard hyperparameters and necessary modifications to adhere to the specified budgets.

For CMA-ES, we use the implementation described in \citep{nomura2024cmaes}, with standard hyperparameters.

For NeuralUCB, we refer to the authors' implementation in \citep{zhou2020neural}. We adapt it to the global optimization setting, where, at each step, we randomly sample four arms and use a neural network with a hidden size of 20 to estimate the upper-confidence bound of these arms, evaluating only the highest one.

We adopt the implementation provided in Botorch \citep{balandat2020botorch}, setting the number of initial points to 20, the acquisition function to log expected improvement, the number of restarts to 5, and the number of raw samples to 20.

For SMAC3, we utilize the implementation provided by \cite{JMLR:v23:21-0888}, with the default hyperparameters.

For the A-GP-UCB \citep{JMLR:v20:18-213}, the kernel used is ``Matern". The tolerance is set to \( 1 \times 10^{-2} \), and the exploration-exploitation tradeoff is controlled by a gamma value of 10. The noise variance is set to \( 1 \times 10^{-4} \). The initial number of samples is 10, with the starting number of hyperparameters set to 5.

For AdaLIPO and AdaLIPO+, we use the implementation provided by \citep{serre2024lipo+}. To ensure fairness, we run AdaLIPO+ without stopping, thereby maintaining the same budget across all methods. Furthermore, since AdaLIPO requires an exploration probability \( p \), we fix it at \( 0.1 \), as done by the authors \citep{malherbe2017global}.

\subsection{Optimization Objectives}

We evaluate the proposed method on various global optimization problems using both synthetic and real-world datasets. The synthetic functions were designed to challenge global optimization methods due to their highly non-convex curvatures \citep{molga2005test, simulationlib}, including Ackley, Bukin, Camel, Colville, Cross-in-Tray, Damavandi, Drop-Wave, Easom, Eggholder, Griewank, Hartmann3, Hartmann6, Himmelblau, Holder, Langermann, Levy, Michalewicz, Perm10, Perm20, Rastrigin, Rosenbrock, Schaffer, and Schubert. Some of the 2D functions are shown in Fig. \ref{fig:all_figures} for reference.

For the real-world datasets, we follow the same set of global optimization problems considered in \citep{malherbe2016ranking, malherbe2017global}, drawn from \citep{frank2010uci}. These include Auto-MPG, Breast Cancer Wisconsin, Concrete Slump Test, Housing, and Yacht Hydrodynamics. The task involves optimizing the logarithm of the regularization parameter $\ln(\lambda) \in [-1,1]$ and the logarithm of the bandwidth $\ln(\sigma) \in [-1,1]$ of a Gaussian kernel ridge regression by minimizing the empirical mean squared error of the predictions over a 3-fold cross-validation.

\subsection{Comparison Protocol}

We use the same hyperparameters across all optimization tasks without fine-tuning them for each task, as this may not be practical when dealing with expensive functions and a limited budget.

We allocate a fixed budget of function evaluations, denoted by $n$, for all methods. The maximum value over the $n$ iterations is recorded for each algorithm. This maximum is then averaged over 100 repetitions, and both the mean and standard deviation are reported.

More evaluations increase the likelihood of finding better points. To ensure that all methods fully utilize the budget, we eliminate any unnecessary stopping conditions, such as waiting times. For example, in AdaLIPO+, we use the variant AdaLIPO+(ns), which continues running even when large rejections occur.

\subsection{Compute and Implementations}
We implement our method using open source libraries, Python 3.9 and Numpy 1.23.4. We use a CPU 11th Gen Intel(R) Core(TM) i7-1165G7 @ 2.80GHz 1.69 GHz. with 16.0 GB RAM

\newpage
\section{ECP Hyper-parameter Discussion and Ablation Study}
\label{app:hyperparams_ecp}

\subsection{Discussion of ECP Hyperparameters}

ECP requires three hyperparameters: $\varepsilon_1 > 0$ (arbitraly small), $\tau_{n,d} > 1$, and $C > 1$. These parameters have been theoretically studied and empirically verified. In our experiments, across all optimization problems represented in \cref{tab:example_50}, \cref{tab:example_25}, and \cref{tab:example_100}, we fixed $\varepsilon_1 = 10^{-2}$ and used $\tau_{n,d} = \max\left(1 + \frac{1}{nd}, \tau\right)$ with $\tau = 1.001$ and $C = 10^3$. Ablation studies have conducted in \cref{ablation_c} and \cref{ablation_tau}.

Increasing the values of $\varepsilon_1$ and $\tau_{n,d}$ causes the rejection probability to approach zero, as shown in \cref{prop:rejection_proba}, reducing the algorithm to a pure random search and undermining the efficiency of function evaluations. Therefore, smaller values for both $\varepsilon_1$ and $\tau_{n,d}$ are required. By multiplying $\varepsilon_t$ by $\tau_{n,d} > 1$ during rejection growth, we guarantee the eventual acceptance of a point, even with small values of $\tau_{n,d}$ and $\varepsilon_1$, as shown in \cref{cor:non_zero_acceptance}.

Note that a larger constant $C$ implies less constraint on rejection growth, which leads to greater patience before increasing $\varepsilon_t$. Consequently, increasing $C$ results in higher rejection rates, further drifting the algorithm away from pure random search at the cost of potentially longer waiting times to accept a sampled point. 

Therefore, either increasing $C$, decreasing $\tau_{n,d}$, or decreasing $\varepsilon_1$ result in higher rejection rates, which result in more careful acceptance at the cost of an increasing computational complexity of the algorithm, see \cref{thm:complexity}.

We could have achieved better results with a larger \(C\) and smaller \(\varepsilon_1\) or \(\tau\). However, we fixed these parameters because they demonstrate outstanding performance while still being a fast algorithm. Users of this algorithm can indeed try other values depending on their problem constraints.

\subsection{Ablation Study on the Constant $\varepsilon_1$}

In the following, we test the performance of ECP with various values of \( \varepsilon_1 \), while keeping \( \tau = 10^{-3} \) and \( C = 10^{3} \) fixed. As shown in \cref{fig:all_figures_ablation_epsilon}, for different values of \( \varepsilon_1 \in [1.0001, 1.001, 1.01, 1.1, 1.5] \), the performance of ECP remains consistent. While for the Hartman 6D function, smaller value of \( \varepsilon_1 \) lead to noticeably better results, in general, for all functions, smaller \( \varepsilon_1 \) values lead to better results as predicted by theory. However, in most of the examples, the differences are less significant. Therefore, the performance of ECP is both consistent and robust across different values of \( \varepsilon_1 \). In our work, we chose a middle value of \( \varepsilon_1 = 10^{-2} \), as it achieves good performance while being computationally less expensive than much smaller values of \( \varepsilon_1 \), as predicted by \cref{thm:complexity}.

\subsection{Ablation Study on the Coefficient $\tau$}
\label{ablation_tau}

Recall that $\tau_{n,d} = \max\left(1 + \frac{1}{nd}, \tau\right)$, therefore the choice of $\tau$ only impacts the algorithm, when the value of $\tau$ is larger than $1 + \frac{1}{nd}$. In the following, we test the performance of ECP, with various values of $\tau$, with fixed $C=10^3$ and $\varepsilon_1=10^{-2}$. As shown in \cref{fig:all_figures_ablation_tau}, for different values of \( \tau \in [1.001, 1.01, 1.1, 1.2, 1.4, 1.6, 1.8, 2] \), the performance of ECP remains consistent. While smaller values of \( \tau \) (blue and orange) yield significantly better results for the Ackley and Hartmann 6D functions, in general, smaller \( \tau \) values tend to perform better across all functions, as predicted by theory. However, in some cases, such as AutoMPG and Damavandi, the differences are less pronounced. Therefore, the performance of ECP is both consistent and robust across various values of \( \tau \). In this work, we chose \( \tau = 1.001 \), as smaller values would not be considered since \( \tau_{n,d} = \max\left(1 + \frac{1}{nd}, \tau\right) \), and we are considering settings with limited budgets.

\subsection{Ablation Study on the Constant $C$}
\label{ablation_c}

In the following, we test the performance of ECP with various values of \( C \), while keeping \( \tau = 10^{-3} \) and \( \varepsilon_1 = 10^{-2} \) fixed. As shown in \cref{fig:all_figures_ablation_C}, for different values of \( C \in [1, 10, 100, 1000, 2000, 4000, 6000, 8000] \), the performance of ECP remains consistent. While for the Ackley function, larger values of \( C \) lead to remarkably better results, in general, for all functions, larger \( C \) values lead to better results as predicted by theory. However, in examples such as AutoMPG, Damavandi, and Bukin, the differences are less significant. Therefore, the performance of ECP is both consistent and robust across different values of \( C \). In our work, we chose a middle value of \( C = 1000 \), as it achieves good performance while being computationally less expensive than larger values of \( C \), as predicted by \cref{thm:complexity}.

\end{document}